\documentclass[twoside,11pt]{article}

\usepackage{jmlr2e}
\usepackage{AdditionalCommands}

\graphicspath{{figures/}}
\usepackage{setspace}
\firstpageno{1}
\usepackage{lastpage}
\jmlrheading{21}{2020}{1-\pageref{LastPage}}{12/19}{7/20}{19-1015}{David R. Burt, Carl Edward Rasmussen, Mark van der Wilk}
\ShortHeadings{Convergence of Sparse VI in GP Regression}{Burt, Rasmussen, van der Wilk}
\begin{document}

\title{Convergence of Sparse Variational Inference \\ in Gaussian Processes Regression}

\author{\name David R. Burt \email drb62@cam.ac.uk \\
       \name Carl Edward Rasmussen \email cer54@cam.ac.uk \\
       \addr Department of Engineering, University of Cambridge, UK%\\
       \AND
         \name Mark van der Wilk\email m.vdwilk@imperial.ac.uk \\
       \addr Department of Computing, Imperial College London, UK \\
       \addr Prowler.io,\thanks{Previous affiliation where significant portion of work was completed.} Cambridge, UK
}
\editor{Kilian Weinberger}

\maketitle
\begin{abstract}%
    Gaussian processes are distributions over functions that are versatile and mathematically convenient priors in Bayesian modelling. However, their use is often impeded for data with large numbers of observations, $N$, due to the cubic (in $N$) cost of matrix operations used in exact inference. Many solutions have been proposed that rely on $M \ll N$ \emph{inducing variables} to form an approximation at a cost of $\BigO(NM^2)$. While the computational cost appears linear in $N$, the true complexity depends on how $M$ must scale with $N$ to ensure a certain quality of the approximation. In this work, we investigate upper and lower bounds on how $M$ needs to grow with $N$ to ensure high quality approximations. We show that we can make the KL-divergence between the approximate model and the exact posterior arbitrarily small for a Gaussian-noise regression model with $M \ll N$. Specifically, for the popular squared exponential kernel and $D$-dimensional Gaussian distributed covariates, $M = \BigO((\log N)^D)$ suffice and a method with an overall computational cost of $\BigO\left(N(\log N)^{2D}(\log \log N)^2\right)$ can be used to perform inference.
\end{abstract}

%%%%%%%%%%%%%%%%%%%%%%%%%%%%%% Introduction %%%%%%%%%%%%%%%%%%%%%%%%%%%%%%%%%%%%%%%%
\section{Introduction}

Gaussian process (GP) priors are commonly used in Bayesian modelling due to their mathematical convenience and empirical success. The resulting models give flexible mean predictions, as well as useful estimates of uncertainty. GP priors are often used with a Gaussian likelihood for regression tasks, as the Bayesian posterior can be computed in closed form in this case. Additionally, in many instances, the kernel is differentiable with respect to hyperparameters, in which case hyperparameters can be efficiently learned using gradient-based optimization by maximizing the marginal likelihood, which can be computed analytically (also known as empirical Bayes, or type-II maximum likelihood). However, standard implementations of exact inference in Gaussian process regression models require storing and inverting a kernel matrix, imposing an $\mcO(N^2)$ memory cost and an $\mcO(N^3)$ computational cost, where $N$ is the number of training examples. These computational constraints have pushed researchers to adopt approximate methods in order to allow Gaussian process models to scale to large data sets.

Sparse methods \citep[e.g.~][]{seegerfast,snelson_sparse_2006,titsias_variational_2009} rely on a set of \emph{inducing variables} to represent the posterior distribution. While these methods have been widely adopted in research and application areas, there is a limited theoretical understanding of the effects of these approximations on the quality of posterior predictions, as well as what biases are introduced into hyperparameter selection when using approximations to the marginal likelihood. In this work, we aim to characterize the accuracy of sparse approximations. If all of the key properties of the exact model, i.e.~the predictive mean and uncertainties and the marginal likelihood, are maintained by very sparse models, then a great deal of computation can be saved through these approximations. 

We focus on the case of sparse inference in the variational framework of \citet{titsias_variational_2009}. We analyze the relationship between the level of sparsity used in performing inference, which dictates the computational cost, and the quality of the approximate posterior distribution. In particular, we analyze how many inducing variables should be used in order for the KL-divergence between the approximate posterior and the Bayesian posterior to be small. This offers theoretical insight into the trade-off between computation and quality of inference within the variational framework. From a practical perspective, our work suggests new methods for choosing which inducing variables to use to construct the approximation and provides theoretically grounded insight into the types of problems to which the sparse variational approach is particularly well-suited.

\subsection{Our Contributions}
\begin{itemize}
    \item We derive bounds on the quality of variational inference in Gaussian process models. When our bounds are applied in the case of the squared exponential (SE)  kernel and Gaussian or compactly supported inputs, we prove that the variational approximation can be made arbitrarily close to the true posterior with arbitrarily high probability using $\mcO((\log N)^D)$ inducing variables, where $D$ is the dimensionality of the training inputs, leading to an overall computational cost of $\BigO(N(\log N)^{2D}(\log\log N)^2)$. Note that we consider $D$ fixed throughout, implying a scaling in $N$ that is nearly linear, i.e.~$\BigO(N^{1+\epsilon})$, $\forall \epsilon > 0$.
    \item Our bounds measure the discrepancy to the true posterior using the KL-divergence between the approximate and exact posteriors. We also show that this implies convergence of the point-wise predictive means and variances.
    \item We show that theoretical guarantees on the quality of matrix approximation for existing methods for selecting regressors in sparse kernel ridge regression can be directly translated into guarantees on variational sparse GP regression. We demonstrate this for ridge leverage scores.
    \item We derive lower bounds on the number of inducing variables needed to ensure that the KL-divergence remains small. For the SE kernel and Gaussian covariate distribution, these lower bounds have the same dependence on the size of the data set as the upper bounds. 
    \item Based on the theoretical results, we provide recommendations on how to select inducing variables in practice, and demonstrate empirical improvements.
\end{itemize}
This paper is an extension of the work \citet{burt2019rates} presented at ICML 2019.
\subsection{Overview of this Paper}
In \cref{sec:background} we introduce notation and review the Gaussian process regression model, as well as sparse variational inference for Gaussian process models.  In \cref{sec:practical-considerations}, we discuss practical considerations regarding assessing the quality of sparse variational inference using upper bounds on the log marginal likelihood that can be computed after observing a data set. In \cref{sec:regression-rates}, we prove our main results, which bound the quality of the sparse approximate posterior, as measured by the KL-divergence. In order to do this, we consider methods for selecting inducing inputs inspired by methods used to obtain theoretical guarantees on sparse kernel ridge regression.  \Cref{sec:examples} considers specific, commonly studied kernels and covariate distributions and investigates the implications of our results in these instances. We provide concrete computational complexities for finding arbitrarily accurate approximations to GPs. In \cref{sec:lower-bounds}, we consider the inverse problem, and show that in certain instances the KL-divergence will be large unless the number of inducing variables increases sufficiently quickly as a function of the size of the data set. \Cref{sec:experiments} discusses practical insights and limitations of the theory as applied to real-world problems.

%%%%%%%%%%%%%%%%%%%%%%%%%% Background %%%%%%%%%%%%%%%%%%%%%%%%%%%%%%%%%%%%%
\section{Background and Notation}\label{sec:background}
In this section, we review exact inference in Gaussian process models, as well as sparse methods for approximate inference in these models. We particularly focus on the formulation of sparse methods based on variational inference \citep{titsias_variational_2009}. Throughout the paper, we use boldface letters to denote random variables, and the same letter in non-bold to denote a realization of this random variable. We follow the standard shorthand notation adopted in many Bayesian machine learning papers and denote probability densities by lower case letters $p$ and $q$, with the distribution to which they are associated inferred by the name of the argument; e.g. $p(X,y)$ is the density of a joint distribution over random variables $\bfX$ and $\bfy$ evaluated at $\bfX=X$ and $\bfy=y$.

\subsection{Gaussian Processes}
A Gaussian process is a collection of real-valued random variables indexed by a set $\mcX$, such that any finite collection of these random variables is jointly Gaussian distributed. While most commonly $\mcX$ is a subset of $\R^D$, Gaussian processes can be indexed by other sets. Such a process can be viewed as defining a distribution over functions $\bff: \mcX\to\Reals$, for which the distribution of function values for a finite set of inputs is Gaussian.

A procedure for specifying the first two moments of any finite marginal distribution in a consistent manner defines a GP. This can be done by selecting a \emph{mean function} $\mu: \mcX \to \R$ and a symmetric, positive semi-definite \emph{covariance function} $k: \mcX \times \mcX \to \R$. The finite dimensional marginal indexed by $\dataX = (x_1, \dotsc, x_N)\transpose \subset \mcX$, denoted $\fdata$, is distributed as
\begin{align}
    \fdata &\sim \mcN(\mu_X, \Kff) \,,
\end{align}
with $\mu_X = (\mu(x_1), \dotsc, \mu(x_N))\transpose$ and $\Kff$ an $N\times N$ matrix with $\left[\Kff\right]_{n,n'} = k(x_n, x_{n'})$. Properties such as smoothness, variance and characteristic lengthscale of functions that are sampled from the GP are determined by the covariance function. The covariance function is often parameterized in such a way that these properties can be adjusted based on properties of the observed data.

\subsection{Gaussian Process Regression}
In this work we perform Bayesian regression using a Gaussian process as the prior distribution over the function we want to learn. 
We observe a data set of $N$ training examples, $\mcD = \{(x_n,y_n)\}_{n=1}^N$ with $x_n \in \mcX$ and $y_n \in \R$ and want to infer a posterior distribution over functions $\bff: \mcX\to\Reals$ that relate the inputs to the outputs.
We define $\dataX=(x_1, \dotsc, x_N)\transpose$, $\datay=(y_1, \dotsc, y_N)\transpose$ and $\fdata = (\bff(x_1), \cdots \bff(x_N))\transpose$. More generally for any finite, set $X' \subset \mcX, |X'|=S<\infty$, which we assume has a fixed ordering, we will use $\mathbf{f}_{X'}$ to denote the (random) vector in $\R^S$, formed by considering the Gaussian process at indices $x \in X'$. 

We specify our Bayesian model through a prior and likelihood. We place a GP prior, which for notational convenience we assume has zero mean function i.e. $\mu \equiv 0$, on the function $\bff$ so that 
\begin{equation}\label{eqn:prior}
    \bff \sim \GP(0,k).
\end{equation}
To allow for deviations from $\bff$ in the observations, we model the data $\datay$ as a noisy observation of this process through the likelihood
\begin{equation}\label{eqn:likelihood-model}
    \bfy \given\fdata \sim \mcN(\fdata, \noisevariance \bfI) \,,
\end{equation}
where the noise variance $\noisevariance$, is a model hyperparameter and $\bfI$ is an $N\times N$ identity matrix.

Since the likelihood and the prior are conjugate in this model, Bayesian inference can be performed in closed form. The posterior density over the latent function values at any finite collection of $T$ new data points $X^\star = (x^\star_1, \cdots, x^\star_T)\transpose$ is given by
\begin{align}
    p(f_{X^\star} \given \mcD) &= \int_{f_X\in \Reals^N} p(f_{X^\star}, f_X \given \mcD) \calcd f_X \nonumber \\
    &= \int_{f_{X} \in \Reals^N} p(f_{X^\star} \given f_{X}) p(f_{X} \given \mcD) \calcd f_{X}.
\end{align}
Both $p(f_{X^\star} \given f_{X})$ and $p(f_{X} \given \mcD)$ are Gaussian densities and the marginal distribution of a Gaussian is Gaussian, so $p(f_{X^\star}\given \mcD)$ is also a Gaussian density. The posterior predictive distribution over the inputs $\fpred$ has mean vector and covariance matrix
\begin{equation}\label{eqn:GP-pred-posterior}
    \widehat{\mu}_\star = \Ksf(\Kff+\noisevariance \bfI)^{-1}\datay \text{ \quad and \quad} \widehat{\Sigma}_{\star \star} = \Kss - \Ksf(\Kff+\noisevariance \bfI)^{-1}\Ksf\transpose,
\end{equation}
where $\Ksf$ is $T \times N$ matrix with $[\Ksf]_{t,n} = k(x^{\star}_t,x_n)$ and $\Kss$ is a $T\times T$ matrix with $[\Kss]_{t,t'} = k(x^{\star}_t,x^{\star}_{t'})$.

The \emph{marginal likelihood} is of interest in Bayesian models for selecting the properties of the model, which are determined by hyperparameters. Point estimates of model hyperparameters are commonly obtained by maximizing the marginal likelihood with respect to the noise variance $\noisevariance$, and any parameters of the prior covariance function $k$. In the case of conjugate regression described above, the log marginal likelihood takes the form
\begin{equation}\label{eqn:full-marginal-likelihood}
 \log p(\datay) =   -\frac{1}{2}\log\det(\Kff+\noisevariance\bfI) -\frac{1}{2} \datay\transpose(\Kff+\noisevariance\bfI)^{-1}\datay-\frac{N}{2}\log 2 \pi.
\end{equation}

The quadratic term measures how well the data $\datay$ lines up with degrees of variation that are allowed under the prior. The log-determinant term measures how much variation there is in the prior, and penalizes priors which are widely spread. The combination of these terms in the log marginal likelihood balances the ability of the model to fit the data with model complexity, which allows a suitable model to be chosen; see \cite{gpml} for more discussion of the marginal likelihood as a tool for model selection as well as an introduction to Gaussian processes.

Despite closed-form expressions for both the predictive posterior (Eq.~\ref{eqn:GP-pred-posterior}) and marginal likelihood (Eq.~\ref{eqn:full-marginal-likelihood}), exact inference in Gaussian process regression models is impractical for large data sets due to the cost of storing and inverting the kernel matrix $\Kff$, leading to $\mcO(N^2)$ memory and $\mcO(N^3)$ time complexities. Sparse approximations have been widely adopted to address this issue.

\subsection{Approximate Inference for Gaussian Processes}
Approximate inference in Gaussian process regression is performed for a different reason than in most Bayesian models. Approximate inference is usually applied when the exact posterior is \emph{analytically} intractable. In our case, we can analytically write down the posterior, but the cost of computation is often prohibitive. The methods we discuss here all approximate the posterior with a different Gaussian process which has more favorable computational properties. As this approximate posterior has a similar form to the exact posterior, and we can control the trade-off between accuracy and computation, it is plausible that our approximation may be very accurate.

\subsubsection{Inducing Variable Methods}
The large cost of computing the posterior GP comes from needing to infer a Gaussian distribution for the function values at all $N$ input locations. Sparse approximations \citep{seegerfast,snelson_sparse_2006,titsias_variational_2009} avoid this cost by instead computing an approximate posterior that only depends on the data through the process at $M\ll N$ locations. 

The aim of these methods is to compress the combined effect of a large number of input and output pairs into a distribution over function values at a small set of inputs. In regions where data is dense, there is often redundant information about what the function is actually doing, so little is lost in performing this approximation. The selected input locations and their corresponding function values are named inducing inputs and outputs respectively, and together are named \emph{inducing points}. Later, it was suggested that more general linear transformations of the process could also be used to compress knowledge into \citep{lazaro-gredilla_inter-domain_2009}. We generally refer to these approaches as \emph{inducing variable} methods. In all of these methods, a low-rank matrix appears in place of $\Kff$ in the computation of the posterior predictive and log marginal likelihood. This matrix can be manipulated with a much lower computational cost than working with $\Kff$ directly.

The success of inducing variable methods depends heavily on \emph{which} $M$ random variables are chosen to represent the knowledge about the function $f$.  Because in this work we are concerned with characterizing how large $M$ should be, we need a good method for choosing the inducing variables, as well as a meaningful criterion for judging the quality of the resulting approximation. The variational formulation of \citet{titsias_variational_2009} is of particular interest, as it uses a well-defined divergence for characterizing the quality of the posterior, which can also be used as a guide for selecting the inducing variables.

\subsubsection{The Variational Formulation}\label{sec:variational-formulation}
Variational inference proceeds by defining a family of candidate distributions $\mcQ$, and then selecting the distribution $Q\in\mcQ$ that minimizes the KL-divergence between the approximation and the posterior. In practice, elements of $\mcQ$ are parameterized and the approximate posterior is selected by choosing an initial approximation which is then refined by finding a local minimum of the KL-divergence as a function of the variational parameters. In variational GP methods \citep{titsias_variational_2009,hensman_gaussian_2013} $\mcQ$ consists of GPs with finite dimensional marginal densities of the form
\begin{align}
    q(f_{X'}, U) = q(U)p(f_{X'}\given U) \label{eq:q-marinal} \,,
\end{align}
for any $X' \subset \mcX, |X'|<\infty$, where $q$ is the density of the approximate posterior at this collection of points, and $p(f_{X'}\given U)$ is the density of the prior distribution of $\bff_{X'}$ at $f_{X'}$ conditioned on the random variables $\bfU$ evaluated at $\bfU=U$. In inducing point approximations, we take the inducing variables to be point evaluations of $f$, i.e.~$\bfU = \bff_Z$, with inducing inputs $Z \subset \mcX$ and $|Z| = M$. 

As discussed in the previous section, we can also define inducing variables as linear transformations of the prior process of the form
\[
\bfu_m = \int_\mcX g_m(x) \bff(x) \calcd\rho(x),
\]
where we assume $\rho$ is a measure on $\mcX$ defined with respect to an appropriate $\sigma$-algebra and $g_m \in L^1(\mcX,\rho)$. If $\rho$ is taken to be a discrete measure, then these features correspond to (weighted) sums of inducing points; while other forms of these inducing variables of this form have been explored \citep{lazaro-gredilla_inter-domain_2009,hensman2018variational}. 

The density $q(U)$ is chosen to be an $M$-dimensional Gaussian density. This choice of variational family induces a Gaussian process approximate posterior with mean and covariance functions
\[
\mu_Q(x) = k_{f(x)\text{u}}\Kuu\inv\mu_{\text{U}} \,, \text{ \quad and \quad} k_Q(x,x') = k(x,x') + k_{f(x)\text{u}}\Kuu\inv\left(\Sigma_{\text{U}} - \Kuu\right)\Kuu\inv k_{\text{u}f(x')} \,,
\]
where $\mu_\text{U},\Sigma_\text{U}$ are the mean and covariance of $q(U)$, $\Kuu$ is the $M \times M$ matrix with entries $[\Kuu]_{m,m'} = \cov(\bfu_m,\bfu_{m'})$, $k_{f(x)\text{u}}$ is the row vector with entries $[k_{f(x')\text{u}}]_{m}=\cov(\bff(x),\bfu_{m})$ and $k_{\text{u}f(x)}$ is a column vector defined similarly. The variational parameters consist of $Z$, which determines the random variables that are included in $\bfU$, and $\mu_\text{U}$ and $\Sigma_\text{U}$, which determine the distribution over $\bfU$.

\newcommand{\lb}{\mathcal{L}}
As is usually the case in variational inference, minimizing the KL-divergence is done indirectly by maximizing a lower bound to the marginal likelihood, $\lb$ (also known as the evidence lower bound, or ELBO), which has $\KL{Q}{P}$ as its slack:
\begin{gather}
    \lb + \KL{Q}{P} = \log p(y) \label{eqn:lb-kl-relation} \, 
    \implies \argmin_{Q\in \mcQ} \KL{Q}{P} = \argmax_{Q\in \mcQ} \lb(Q) \,.
\end{gather}
where $P$ denotes the (exact) posterior process \citep{matthews_sparse_2016}. 

When the likelihood is isotropic Gaussian, the unique optimum for the parameters $\{\mu_\text{U},\Sigma_\text{U}\}$ can be computed in closed form. Using these optimal values, we obtain the ELBO as it was introduced by \citet{titsias_variational_2009},
\newcommand{\lbcol}{\lb}
\begin{equation}\label{eqn:elbo}
    \lbcol = -\frac{1}{2}\log\det(\Qff+\noisevariance\bfI) -\frac{1}{2} \datay\transpose(\Qff+\noisevariance\bfI)^{-1}\datay-\frac{N}{2}\log 2 \pi-\frac{1}{2\noisevariance}\Tr(\Kff-\Qff) \,,
\end{equation}
where $\Qff=\Kuf\transpose\Kuu^{-1}\Kuf$ and $\Kuf$ is the $M \times N$ matrix with entries $[\Kuf]_{i,j} = \cov(\bfu_i,\bff(x_j))$. 

\subsubsection{Measuring the Quality of a Variational Approximation}\label{sec:posterior-consistency}

In order to assess whether variational inference leads to an accurate approximation to the posterior, we need to choose a definition of what it means for an approximation to be accurate. We choose to measure the quality of an approximation in terms of the KL-divergence, $\KL{Q}{P}$. This KL-divergence is $0$ if and only if the approximate posterior is equal to the exact posterior. Under this measure, an approximation is considered good if this KL-divergence is small.

Variational approximations using this KL-divergence have been criticized for failing to provide guarantees on important quantities such as posterior estimates of the mean and variance. \citet{Huggins2018ScalableGP} observed that there exist Gaussian distributions such that the (normalized) difference between the means of the distributions is exponentially large as a function of the KL-divergence between the two distributions, as is the ratio of the variances. This has been used to motivate variational approaches based on other notions of divergence, as well as a more careful assessment of the quality of the approximations obtained via variational inference \citep{huggins2019practical}. 

However, in our case of sparse Gaussian process regression, a sufficiently small KL-divergence between the approximate and true posterior implies bounds on the approximation quality of the marginal posterior mean and variance function.  \Cref{prop:marginal-bound} states one such bound:
\begin{restatable}[]{prop}{marginalbound}
\label{prop:marginal-bound}
Suppose $2\KL{Q}{P} \leq \gamma \leq \frac{1}{5}$. For any $x^\star \in \mcX$, let $\mu_1$ denote the posterior mean of the variational approximation at $x^\star$ and $\mu_2$ denote the mean of the exact posterior at $x^\star$ . Similarly, let $\sigma^2_1,\sigma^2_2$ denote the variances of the approximate and exact posteriors at $x^\star$. Then,
\begin{align*}
\lvert\mu_1-\mu_2\rvert \leq \sigma_2\sqrt{\gamma} \leq \frac{\sigma_1\sqrt{\gamma}}{\sqrt{1-\sqrt{3\gamma}}}
\text{\quad and \quad}
\lvert 1- \sigma_1^2/\sigma_2^2 \rvert<\sqrt{3\gamma}.
\end{align*}
\end{restatable}
 
The proof (\cref{app:proof-marginals}) uses that the KL-divergence between any pair of joint distributions upper bounds the KL-divergence between marginals of these distributions. It then suffices to bound the difference between the mean and variance of univariate Gaussian distributions with a small KL-divergence between them.

\Cref{prop:marginal-bound} implies that in cases where we can prove the KL-divergence between the approximate posterior and the exact posterior is very small, we are guaranteed to obtain similar marginal predictions with the variational approximation to those we would obtain with the exact model.  We note that direct approaches to bounding marginal moments may lead to tighter bounds on these quantities \citep[e.g.~][]{calandriello2019gaussian}, but we prefer to consider the KL-divergence due to its connection to the variational objective function.

The consequences of a small KL-divergence for hyperparameter selection using the evidence lower bound are more subtle, as both the approximate posterior and exact posterior depend on model hyperparameters, and it is generally difficult to ensure that the KL-divergence is uniformly small. We will be discuss these issues in more detail in \cref{sec:experiments}.

\subsubsection{Computation and Accuracy Trade-Offs}
The ELBO (Eq.~\ref{eqn:elbo}) as well as the corresponding choices for $\mu_{\text{U}}$ and $\Sigma_{\text{U}}$ (needed for making predictions) can be computed in $\mcO(NM^2)$ time, and with $\mcO(NM)$ space. If a good approximation can be found with $M\ll N$, the savings in computational cost are large compared to exact inference. From \cref{eqn:elbo} we see that the approximation is perfect when choosing $Z = \dataX$, as this leads to $\Qff = \Kff$. However, no computation is saved in this setting. We seek a more complete understanding of the trade-off between accuracy and computational cost when $M < N$ by understanding how $M$ needs to grow with $N$ to ensure an approximation of a certain quality. We derive probabilistic upper and lower bounds on this rate that depend on kernel properties that can be analyzed \emph{before} observing any data.

\subsection{Spectrum of Kernels and Mercer's Theorem}

In the previous section, we noted that sparse methods imply a low-rank approximation $\Qff$ to the kernel matrix $\Kff$. In order to understand the impact of sparsity on the variational posterior, it is necessary to understand how well $\Kff$ can be approximated by a rank-$M$ matrix. This depends on the behavior of the eigenvalues of $\Kff$. 

For small data sets, an eigendecomposition of $\Kff$ allows direct empirical analysis. However, for problems where sparse approximations are actually of interest, eigendecompositions are not available within our computational constraints. However, even without access to a specific data set, we can reason that properties of the training inputs have a large impact on the properties of the eigendecomposition of the kernel matrix. For example, consider the case of a squared exponential kernel given by
\[
k_{\textup{SE}}(x,x') = v\exp\left(-\frac{\|x-x'\|^2_2}{2\ell^2}\right),
\]
where $v>0$ is the signal variance, which controls the variance of marginal distributions of the prior and $\ell>0$ is the lengthscale, which controls how quickly the correlation between function values decreases as a function of the distance between their inputs. If each covariate in our training data set is sufficiently far apart relative to the lengthscale, then $\Kff \approx v\bfI$  and any approximation by low-rank matrix will be of poor quality. Alternatively, if each covariate takes the same value, $\Kff$ is a rank-one matrix and $\Qff=\Kff$ if a single inducing point is placed at the location of the covariates. Therefore, in order to make statements about the eigenvalues of $\Kff$, we will need to make assumptions about the locations of training covariates $\dataX$.

For the remainder of the paper, we assume $\mcX=\R^D$ (the generalization of most of our results is straightforward). One method for understanding the eigenvalues of $\Kff$ is to suppose the $x_i$ are realizations of random variables $\bfx_i\iidsim \mu_x$, where $\mu_x$ is a probability measure on $\R^D$ with density $p(x)$. Under this assumption, the limiting properties of the kernel matrix are captured by the \emph{kernel operator}, $\mcK: L^2(\R^D,\mu_x)\to L^2(\R^D,\mu_x)$ defined by
\begin{equation}\label{eqn:operator_definition}
(\mcK g)(x) = \int g(x')k(x,x')p(x')\calcd x'.
\end{equation}
If the kernel is continuous and bounded, then $\mcK$ has countably many eigenvalues. We denote these eigenvalues in non-increasing order, so that $\lambda_1 \geq \lambda_2 \geq \dots \geq 0$. Corresponding to each non-zero eigenvalue $\lambda_m$ there is an eigenfunction $\phi_m$ which can be chosen to be continuous. 

Moreover, the collection $\{\phi_m\}_{m=1}^\infty$ can be chosen so that the eigenfunctions all have unit norm and are mutually orthogonal as elements of $L^2(\R^D,\mu_x)$. 
In this case, Mercer's theorem \citep{mercer1909,konig2013eigenvalue} states that for sufficiently nice $k$, for all $x \in \R^D$ such that $p(x)>0$,
\begin{align}
    k(x,x') = \sum_{m=1}^\infty \lambda_m \phi_m(x)\phi_m(x') \text{\quad and \quad} \sum_{m=1}^\infty \lambda_m < \infty,
\end{align}
where the sum on the left converges absolutely and uniformly.\footnote{See \citet{gpml}, section 4.3 for more discussion of Mercer's theorem.}

The bounds we derive in the remainder of this work will depend on how rapidly the eigenvalues $\{\lambda_m\}_{m=1}^\infty$ decay. As they are absolutely summable, they must decay faster than $1/m$. The decay of these eigenvalues is closely related to the complexity of the non-parametric model as well as the generalization properties of the posterior \citep{Micchelli1979DesignPF, plaskota1996noisy}. 
Generally, these eigenvalues decay faster for covariate distributions that are concentrated in a small volume, and for kernels that give smooth mean predictors \citep{widom_asymptotic_1963,widom1964asymptotic}. Therefore, the bounds we prove in \cref{sec:regression-rates} can be seen as verifying the intuition that sparse variational approximations can be successfully applied to models with smooth prior kernels, as well as data sets with densely clustered covariates. 

\subsection{Inducing Variable Selection and Related Bounds}
While the kernel eigenvalues determine how well a kernel matrix \emph{can} be approximated, the quality of an actual approximation depends on how the inducing variables are chosen. 
Inducing point selection has been widely studied for many methods that require constructing a Nystr\"{o}m approximation, like sparse Gaussian processes and kernel ridge regression (KRR). In the simplest case, a subset can be uniformly sampled from the training inputs. Bounds on the quality of the resulting matrix approximation, and downstream Kernel Ridge Regression predictor have been found for this case \citep{bach2013sharp,gittens16a} and depend heavily on assumptions about the covariate distribution and resulting kernel matrix. In the Gaussian process literature, some specific low-rank parametric approximations based on spectral information about the kernel operator or matrix have been proposed \citep{zhu_gaussian_1997, ferrari-trecate_finite-dimensional_1999, solin2020hilbert} together with analysis on the rate of decrease in error with additional features. However, these methods generally are either limited in the types of kernels they can be applied to or have higher computational complexity than inducing point methods. 

Heuristic inducing point selection methods have also been proposed in the hope of improving performance, for instance approximately minimizing $\Tr(\Kff-\Qff)$ \citep{smola2000sparse}, approximating the information gain of including a data point in the posterior \citep{seegerfast}, or using the k-means centres of the input distribution \citep{hensman_gaussian_2013,hensman_scalable_2015}.

Two methods from the KRR literature are of particular interest: sampling from a Determinantal Point Process (DPP) \citep{li2016fast}, and ridge leverage scores \citep{AlaouiFast2015,rudi2015less,calandriello2017distributed}. Theoretical guarantees exist in the literature for these methods applied to KRR, as well as empirical evidence of their efficacy compared to uniform sampling. The initial version of this work \citep{burt2019rates} analyzed convergence of the sparse variational GP posterior and marginal likelihood using the DPP initialization. Concurrently, \citet{calandriello2019gaussian} used ridge leverage scores to show  the DTC approximation \citep{seegerfast,quin2005unifying} can be made similar to the true posterior, in terms of pointwise predictive means and variances. Given the similarity between the DTC and variational posteriors, we include an analysis of ridge leverage sampling in this extended work to also provide results of convergence of the ELBO, and of the posterior in terms of the KL, which
also implies pointwise convergence of the predictive means and variances.

\section{Assessing Variational Inference: a Posteriori Bounds on the KL-divergence}\label{sec:practical-considerations}
We begin our investigation by considering how to choose the number of inducing variables for a specific  data set. The simplest approach to assessing whether sufficiently many inducing points are used is to gradually increase the number of inducing points, and assess how the evidence lower bound changes with each additional point. If the ELBO increases only slightly or not at all when an additional inducing point is added, it is tempting to conclude that the approximate posterior is very close to the exact posterior. However, this is not a sufficient condition for the approximation to have converged. It could be the case that the last inducing point placed was not placed effectively, or that increasing from $M$ to $M+1$ inducing points has little impact, but increasing to $M+c$, for some $c>1$, inducing points would lead to significantly better performance if these points are well-placed. 

A more refined mechanism for assessing the quality of the variational posterior would be to consider an upper bound on the KL-divergence that can be computed in similar computational time to the ELBO. Such a bound was proposed by \citet{titsias_variational_2014} and discussed as a method for assessing convergence in \citet{kim2018scaling}. In order to state this bound, we first need to introduce some notation. Let $t\coloneqq \Tr(\Kff-\Qff)$ denote the trace of $\Kff-\Qff$ and $\|\Kff-\Qff\|_{\textup{op}}$ denote the operator norm of $\Kff - \Qff$, which in this case is equal to the largest eigenvalue of this matrix as it is symmetric positive semidefinite. 
\begin{restatable}[\citealp{titsias_variational_2014}]{lem}{upperbound}
\label{lem:titsias-aposteriori}
For any $y \in \R^N, X \in \mcX^N$, and set of $M$ inducing variables, $U$, 
\[
 \log p(\datay) \leq \upperone \leq \uppertwo
\]
 where
\begin{equation}
    \upperone\coloneqq -\frac{1}{2}\log\det(\Qff+\noisevariance\bfI) -\frac{1}{2} \datay\transpose(\Qff+\|\Kff-\Qff\|_{\textup{op}}\bfI+\noisevariance\bfI)^{-1}\datay-\frac{N}{2}\log 2 \pi, \nonumber
\end{equation}
and
\begin{equation}
\uppertwo\coloneqq -\frac{1}{2}\log\det(\Qff+\noisevariance\bfI) -\frac{1}{2} y\transpose(\Qff+t\bfI+\noisevariance\bfI)^{-1}y-\frac{N}{2}\log 2 \pi.\label{eqn:upper2}
\end{equation}
\end{restatable}
For completeness we give a brief derivation of \cref{lem:titsias-aposteriori} in \cref{app:proofs-section3}, which essentially follows the derivation of \citet{titsias_variational_2014}.

In problems where sparse GP regression is applied, computing the largest eigenvalue of $\Kff-\Qff$ in order to compute $\mcU_1$ is computationally prohibitive. However, $\Tr(\Kff-\Qff)$ can be computed in $\mcO(NM^2)$, so that $\mcU_2$ can be computed efficiently.

As $\ELBO = \log p(\datay) - \KL{Q}{P}$, we have
\begin{align}\label{eqn:kl-bounded-upper}
\KL{Q}{P} = \log p(\datay) - \ELBO \leq \mcU_1  - \ELBO\leq \mcU_2  - \ELBO.
\end{align}
If the difference between the upper and lower bounds is small, we can therefore be sure that sufficiently many inducing points are being used for the KL-divergence to be small. This suggests a refinement of the method for selecting the number of inducing points discussed earlier: continue to place more inducing points until the difference between the upper and lower bounds is small. 

This raises the question: how many inducing variables do we need for the KL-divergence to be small in a typical problem? The upper bounds discussed above assess the approximation \emph{a posteriori}, i.e.~for a given data set and a given approximation. We would like to characterize the required number of inducing variables for a whole class of problems, \emph{before} observing any data. This allows us to understand \emph{a priori} how much computation is needed to solve a particular problem. For example, if the number of inducing variables $M$ needs to grow linearly with the number of observations $N$, then the $\mcO(NM^2)$ cost of the approximation effectively scales cubically in $N$, i.e.~in the same way as the exact implementation. In \cref{sec:regression-rates}, we show that under intuitive assumptions, the number of inducing points can be taken to be much smaller than the size of the data set, while still giving approximations with small KL-divergences.

%%%%%%%%%%%%%%%%UPPER BOUNDS%%%%%%%%%%%%%%%%%%%%%%%%%%%%%%%%%%%%%%%%%%%%%%%
\section{Convergence of Sparse Variational Inference in Gaussian Processes}\label{sec:regression-rates}

In this section, we prove upper bounds on the KL-divergence between the approximate posterior and the exact posterior that depend on the number of inducing points used in inference, properties of the prior and distributional assumptions on the training covariates. The proof proceeds in three parts:
\begin{enumerate}
    \item Derive an upper bound on the KL-divergence for a fixed data set and fixed set of inducing points that only depends on the quality of the approximation of $\Kff$ by $\Qff$. In order to do this we make assumptions about the data generating process for $\datay$. 
    \item Suggest a method for selecting inducing inputs that obtains a high quality low-rank approximation to $\Kff$. This yields an upper bound on the KL-divergence depending only on the eigenvalues of $\Kff$. We consider using a $k$-determinantal point process or ridge leverage scores as the initialization method. 
    \item Relate eigenvalues of the kernel matrix back to those of the corresponding kernel operator, \cref{eqn:operator_definition}, through assumptions on the distribution of the covariates. 
\end{enumerate}
The second step has precedent in the literature on sparse kernel ridge regression. For example, \citet{li2016fast} consider using a $k$-DPP to select the sparse regressors. Meanwhile ridge leverage scores have been studied in the setting of sparse kernel ridge regression and Gaussian process regression \citep{AlaouiFast2015,rudi2015less,calandriello2017distributed,calandriello2019gaussian}, and have been shown to lead to strong statistical guarantees.

The third step in our analysis is similar to the analysis carried out when studying generalization and approximation bounds for Gaussian processes and other kernel methods. We use a generalization of a lemma proven in \citet{shawe2005eigenspectrum} for this step.

In order to carry out our analysis, especially steps 2 and 3, we will treat $X,y$ and $Z$ as realizations of random variables $\bfX$, $\bfy$ and $\bfZ$ and make distributional assumptions about these random variables. This will allow us to make statements about bounds that hold in expectation or with fixed probability.

\subsection{A-Posteriori Upper Bounds on the KL-divergence Revisited}\label{sec:a-posteriori-bounds-revisited}
In \cref{sec:practical-considerations}, we considered bounds on the KL-divergence that can be computed for a specific data set. In this section, we first derive an upper bound on the KL-divergence that only depends on the squared norm of $\bfy$, with no additional assumptions on the distribution of the $\bfy$ (\cref{lem:agnostic-upper-bound}). We then derive a second bound, given in \cref{lem:kl-averaged}, that improves on \cref{lem:agnostic-upper-bound} in expectation, under the stronger assumption that $\bfy|\bfZ,\bfX\sim \mcN(0, \Kff+\noisevariance\bfI)$. This assumption is satisfied if $\bfy$ is distributed according to the prior model and the distributions of $\bfZ$ and $\bfy$ are independent, i.e. the inducing inputs are chosen without reference to $y$. While our results are stated in terms of inducing points, the proofs generalize without modification to other inducing variables of the form discussed in \cref{sec:variational-formulation}.

\subsubsection{Upper bounds on the KL-divergence}
We first consider the case where we make few assumptions on the distribution of $\bfy$. 

\begin{restatable}{lem}{uppergeneraly}
\label{lem:agnostic-upper-bound}
For any $y \in \R^N, X \in \mcX^N$, and any $Z \in \mcX^M$
\begin{equation*}
    \KL{Q}{P} \leq \upperone - \ELBO \leq \frac{1}{2\noisevariance}\left(t+\frac{\zeta\|y\|_2^2}{\zeta+\noisevariance}\right) \leq \frac{1}{2\noisevariance}\left(t+\frac{t\|y\|_2^2}{t+\noisevariance}\right),
\end{equation*}
with $t=\text{\emph{tr}}(\Kff-\Qff)$ and $\zeta = \|\Kff-\Qff\|_{\textup{op}}$.
\end{restatable}
The first inequality has already been established (Eq.~\ref{eqn:kl-bounded-upper}). The remainder of the proof, given in \cref{app:proofs-section4} relies on properties of symmetric positive semi-definite (SPSD) matrices.

 In most applications, it is reasonable to assume that the data generating process for $\bfy$ is such that $\|\bfy\|_2^2 \leq RN$ almost surely, or at least $\conditionalexp{\|\bfy\|^2_2}{\bfX, \bfZ} \leq RN$, for some constant $R>0$. For example, if $\bfy$ is formed by evaluating a bounded function corrupted by Gaussian or bounded noise and the location of the inducing inputs is independent of $\bfy$ then \cref{lem:agnostic-upper-bound} allows us to bound the conditional expectation $\conditionalexp{\KL{Q}{P}}{\bfX,\bfZ}$.

\subsubsection{Average Case Analysis for the Prior Model}
In \cref{lem:agnostic-upper-bound}, we did not make any assumption on the distribution of $\bfy$. From the Bayesian perspective, it is natural to make stronger distributional assumptions on $\bfy|\bfX$. We will see that in some instances stronger assumptions can lead to a much tighter upper bound than \cref{lem:agnostic-upper-bound} that holds in expectation.

The natural candidate distribution for $\bfy$ is the prior distribution, that is $\bfy|\bfX \sim \mcN(0,\Kff+\noisevariance\bfI)$; if we additionally assume that the distributions of $\bfy|\bfX$ and $\bfZ|\bfX$ are independent, then this implies $\bfy|\bfX,\bfZ \sim \mcN(0,\Kff+\noisevariance\bfI)$. In this case we can derive upper and lower bounds on the conditional expectation of the KL-divergence conditioned on $\bfX$ and $\bfZ$.

\begin{restatable}{lem}{upperavgy}\label{lem:kl-averaged}
Suppose $\bfy|\bfX,\bfZ \sim \mcN(0,\Kff+\noisevariance\bfI)$. For any $X \in \mcX^N$ and $Z \in \mcX^M$,
\begin{equation*}
t/(2\noisevariance) \leq \conditionalexp{\KL{Q}{P}}{\bfZ=Z,\bfX=X} 
\leq t/\noisevariance
\end{equation*}
where $t=\text{\emph{tr}}(\Kff-\Qff)$ and $\Kff$ and $\Qff$ are defined with respect to this $X,Z$ as in \cref{sec:background}.
\end{restatable}

\begin{rem}
Note that if $\bfy \sim \mcN(0,\Kff+\noisevariance \bfI)$,
\begin{align*}
    \conditionalexp{\|\bfy\|^2}{\bfX=X,\bfZ=Z} &= \conditionalexp{\text{\emph{tr}}(\bfy\transpose \bfy)}{\bfX=X,\bfZ=Z} \\
    &= \text{\emph{tr}}(\conditionalexp{\bfy \bfy\transpose}{\bfX=X,\bfZ=Z})
    = \text{\emph{tr}}(\Kff+\noisevariance\bfI).
\end{align*}
Therefore, under the strong assumption that $\bfy$ is sampled from the prior model, \cref{lem:kl-averaged} gives a significantly stronger bound on the expected KL-divergence as compared to \cref{lem:agnostic-upper-bound}.
\end{rem}

\begin{proof}[Proof Sketch of \cref{lem:kl-averaged}]
Recall that $\KL{Q}{P}= \log p(\bfy) - \ELBO$. Taking conditional expectations on both sides,
\begin{align}\label{eqn:expected-kl-1}
    \conditionalexp{\KL{Q}{P}}{\bfX=X,\bfZ=Z} & = \conditionalexp{\log p(\bfy) - \ELBO}{\bfX=X,\bfZ=Z}.
\end{align}
Let $n(\datay; m, S)$ denote the density of a (multivariate) Gaussian random variable with mean $m$ and covariance matrix $S$ evaluated at $\datay$. Then, 
\[
\log p(\bfy) = \log n(\bfy;0, \Kff+\noisevariance \bfI)
\]
and 
\[
\ELBO = \log n(\bfy;0, \Qff+\noisevariance \bfI) - \frac{1}{2 \noisevariance}\Tr(\Kff-\Qff).
\]
Using this in \cref{eqn:expected-kl-1},
\begin{align}
    \conditionalexp{\KL{Q}{P}}{\bfX=X,\bfZ=Z} & =  \conditionalexp{\log \frac{n(\bfy;0, \Kff+\noisevariance \bfI)}{n(\bfy;0, \Qff+\noisevariance \bfI)}}{\bfX=X,\bfZ=Z} + \frac{t}{2 \noisevariance}\nonumber \\
    & = \KL{\mcN(0, \Kff+\noisevariance\bfI)}{\mcN(0, \Qff+\noisevariance\bfI)} + \frac{t}{2 \noisevariance}\label{eqn:average-kl-exact}.
\end{align}
The lower bound follows from the non-negativity of KL-divergence. The proof of the upper bound (\cref{app:proofs-section4}) relies on the formula for the KL-divergence between multivariate Gaussian distributions as well as the identity $|\Tr(AB)| \leq \Tr(A)\|B\|_{\mathrm{op}}$ for SPSD matrices $A$ and $B$ \citep[Exercise 1.3.26]{tao2011topics}.
\end{proof}

\begin{rem}
The lower bound in \cref{lem:kl-averaged} holds in expectation \emph{conditioned} on $\bfX=X$ and $\bfZ=Z$, with $\bfy$ distributed according to our prior. Common practice is to optimize the inducing inputs with respect to the ELBO, which depends on $\datay$. We may therefore expect that the KL-divergence will be somewhat smaller than predicted by the average case lower bound in \cref{lem:kl-averaged} after this optimization. This is shown in \cref{fig:kl_trace_comparison}, where we generate a data set satisfying the conditions of the lemma, and look at the trace and KL-divergence before and after gradient based optimization of the ELBO with respect to the inducing inputs. In \cref{sec:lower-bounds}, we establish lower bounds that hold for any $y \in \R^N$ conditioned on $\bfX=X$ and $\bfZ=Z$ and are therefore applicable to the case when inducing points are selected via gradient-based methods.
\end{rem}

\begin{figure}
    \centering
    \includegraphics[width=\textwidth]{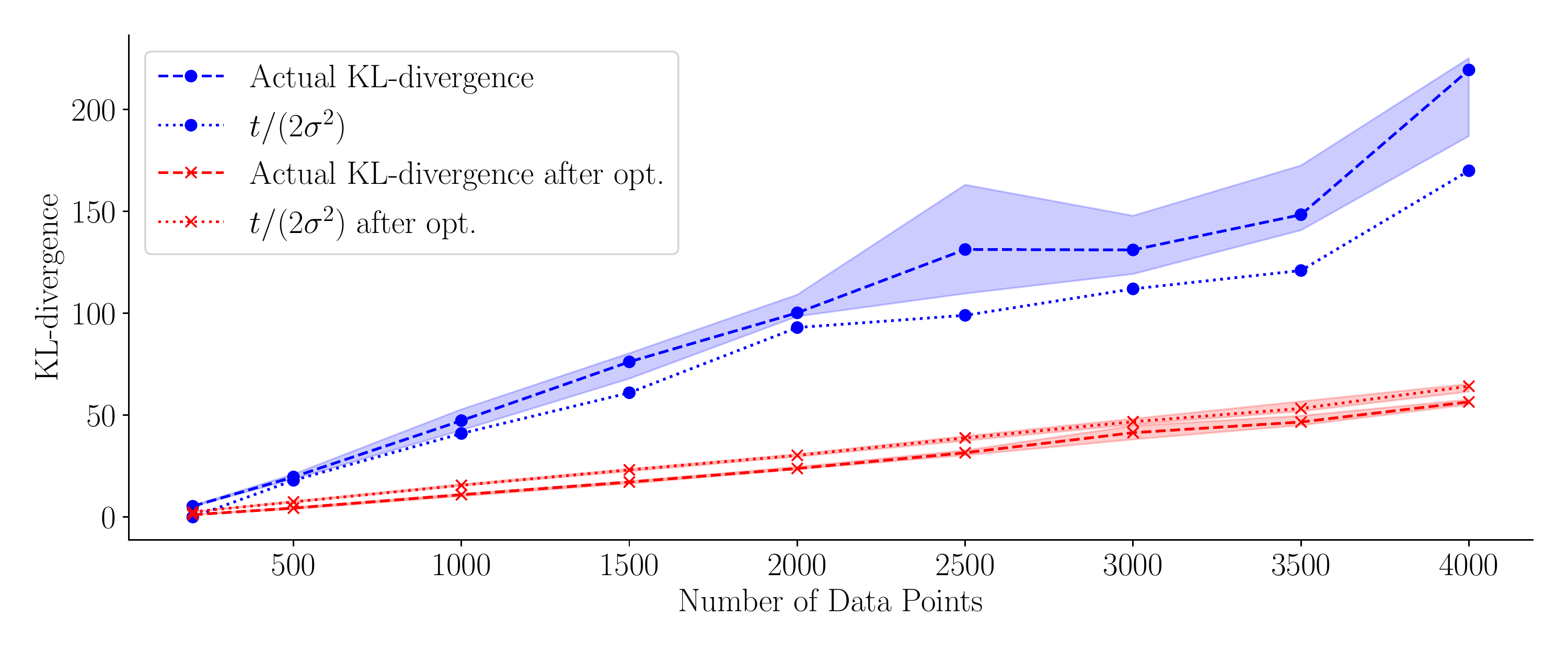}
    \caption{A comparison of the KL-divergence and the lower bound in \cref{lem:kl-averaged}. For each value of $N$ we first fix a set of covariates and $M=30$ inducing points and then compute both the trace and KL-divergence (shown in blue), with $y\sim \mcN(0,\Kff+\noisevariance \bfI)$. The dashed line shows the median of 20 random $y$'s, while the shaded region represents 20-80 percentile regions. We then optimize the locations of the inducing points via gradient descent on the ELBO. As $N$ increases, both the trace and KL-divergence increase. When the inducing points are selected via optimizing the ELBO, the KL-divergence is typically somewhat lower than the lower bound in \cref{lem:kl-averaged}, while if $Z$ is chosen without reference to $\datay$ the lower bound on the expected value of the KL-divergence holds.}
    \label{fig:kl_trace_comparison}
\end{figure}

\subsection{Initialization of Inducing Points}
In the previous section, we derived upper bounds on $\mathbb{E}[\mathrm{KL}[Q||P]|\bfX,\bfZ]$ that depend on assumptions about the distribution of $\bfy$. These bounds depend on either the trace or the largest eigenvalue of $\Kff-\Qff$, and will therefore be small if $\Kff \approx \Qff$. We begin this section with a brief overview of inducing variable selection methods, of which we will analyze two in the context of sparse variational inference. Using known results on the quality of the resulting matrix approximations, we can then obtain bounds on the KL-divergence for a fixed set of training inputs $\bfX=X$.

\subsubsection{Minimizing the upper bounds}\label{sec:eigenvec-features}
We take a brief detour from discussing initializations of inducing inputs to discuss the set of inducing variables that minimize the upper bounds in \cref{lem:kl-averaged,lem:agnostic-upper-bound}. 

Let $\Kff =\matV \Lambda \matV\transpose$, where $\matV = [ v_1, v_2, \cdots v_N]$ is an $N \times N$ orthogonal matrix and $\Lambda= \text{diag}(\tilde{\lambda}_1, \dotsc, \tilde{\lambda}_N)$ is a diagonal matrix of eigenvalues of $\Kff$ ordered such that $\tilde{\lambda}_1 \geq \tilde{\lambda}_2 \geq \dots \geq \tilde{\lambda}_N \geq 0$. As $\Kff$ is SPSD, such a decomposition exists. Define $\textup{K}_{M}=\matV_M\Lambda_M\matV_M\transpose$, where $\bfV_M$ is an $N\times M$ matrix containing the first $M$ columns of $\bfV$ and $\Lambda_M$ is an $M \times M$ diagonal matrix with entries, $\tilde{\lambda}_1, \dotsc, \tilde{\lambda}_M$, in other words $\textup{K}_M$ is the rank-$M$ truncated singular value decomposition of $\Kff$.

Both the trace and the operator norm are unitarily invariant, so $\textup{K}_M$ is the optimal rank-$M$ approximation to $\Kff$ according to either of these norms.\footnote{While the trace is not generally a matrix norm, it agrees with the norm $\|\cdot\|_1$ as $\Kff-\Qff$ is SPSD.} In particular, for any rank $M$ $N\times N$ SPSD matrix $\bfA$ satisfying $\bfA \prec \Kff$ (i.e.~$\Kff-A$ is SPSD), $\Tr(\Kff-\textup{K}_M) \leq \Tr(\Kff-\bfA)$ and $\|\Kff-\textup{K}_M\|_{\textup{op}} \leq \|\Kff-\bfA\|_{\textup{op}}$ \citep[see][Theorem 7.4.9.1]{horn_matrix_1990}.
% (see \citet[Theorem 7.4.9.1]{horn_matrix_1990}). 

As any subset of $M$ inducing variables will lead to a rank-$M$ matrix $\Qff \prec \Kff$ this implies
\begin{equation*}
    \|\Kff-\Qff\|_{\textup{op}} \geq \|\Kff-\textup{K}_M\|_{\textup{op}} = \tilde{\lambda}_{M+1} \,\, \text{and} \,\, \Tr(\Kff-\Qff) \geq \Tr(\Kff-\textup{K}_M) = \!\!\!\sum_{m=M+1}^N\!\!\! \tilde{\lambda}_{m}.
\end{equation*}

 Consider the inducing features defined as linear combinations of the random variables associated to evaluating the latent function at each observed input location, with weights coming from the eigenvectors of $\Kff$, i.e.
\[
\bfu_m= \frac{1}{\tilde{\lambda}_m} \sum_{i=1}^N v_{i,m} \bff(x_i).
\]
Then,
\[
\cov(\bfu_m,\bfu_{m'})\!=\! \frac{1}{\tilde{\lambda}_m\tilde{\lambda}_{m'}}\Exp{}{\sum_{i=1}^N \sum_{j=1}^N v_{m,j}v_{m',i} \bff(x_i) \bff(x_{j})} \!=\! \frac{1}{\tilde{\lambda}_m\tilde{\lambda}_{m'}}\sum_{i=1}^N \sum_{j=1}^N v_{m,i}v_{m',j} k(x_i,x_j).
\]
The final two sums are the quadratic form, $v_m\transpose\Kff v_{m'}$. As $v_m$ is an eigenvector of $\Kff$ and $v_m$ is orthogonal to $v_m'$ unless $m=m'$, this simplifies to $\cov(u_m,u_{m'})= \frac{\delta_{m,m'}}{\tilde{\lambda}_m}$. Similarly,
\[
\cov(\bfu_m,\bff(x_n)) = \frac{1}{\tilde{\lambda}_m}\Exp{}{\sum_{i=1}^N v_{m,n}\bff(x_i) \bff(x_n)} = \frac{1}{\tilde{\lambda}_m}\sum_{i=1}^N v_{m,n}k(x_i,x_n) = \frac{1}{\tilde{\lambda}_m}[\Kff v_m]_n = v_{m,n}.
\]
From these expressions, it follows that for these features $\Kuf = \bfV_M$ and $\Kuu^{-1}=\Lambda_M$. Therefore, $\Qff=\textup{K}_M$, and these features minimize our upper bounds among any set of $M$ inducing variables. Unfortunately, computing the matrices $\Kuf$ and $\Kuu$ in this case involves computing the first $M$ eigenvalues and eigenvectors of $\Kff$, which lies outside of our desired computational budget of $\BigO(N\mathrm{poly}(M)\mathrm{polylog}(N))$.

\subsubsection{M-Determinantal point processes}\label{sec:kdpp}

We now return to the more practical case of using inducing points for sparse variational inference. In order to derive non-trivial upper bounds on $\Tr(\Kff-\Qff)$ and $\|\Kff-\Qff\|_{\textup{op}}$, we need a sufficiently good method for placing inducing points. When using differentiable kernel functions, many practitioners select the locations of the inducing points with gradient-based methods by maximizing the ELBO. As this is a high-dimensional, non-convex optimization algorithm, directly analyzing the result of this procedure is beyond our analysis.

In this section, we assume $M$ inducing points are subsampled from data according to an approximate \emph{M-determinantal point process} ($M$-DPP) \citep{kulesza2011k} and use known bounds on the expected value of $\Tr(\Kff-\Qff)$.\footnote{The standard terminology is $k$-DPP. We use $M$ as this determines the number of inducing points and to avoid confusion with the kernel function.} We note that if this scheme is used as an initialization prior to a gradient-based optimization of the evidence lower bound with respect to the inducing inputs, the resulting KL-divergence will be at least as small, so our bounds still apply after optimization of variational parameters.

Given an SPSD matrix $\textup{L}$, an $M$-determinantal point process \citep{kulesza2011k} with kernel matrix $\textup{L}$ defines a discrete probability distribution over subsets of the $N$ columns of $\textup{L}$, with positive probability only assigned to subsets of cardinality $M$. The probability of any subset of cardinality $M$ is proportional to the determinant of the principal submatrix formed by selecting those columns and the corresponding rows, that is for any set $Z$ of $M$ columns of $L$
\[
\text{Pr}(\bfZ=Z) = \frac{\det(\textup{L}_{Z,Z})}{\sum_{ |Z'|=M}\det(\textup{L}_{Z',Z'})}.
\]
 where $\textup{L}_{Z,Z}$ is the principal submatrix of $\textup{L}$ with columns in $Z$. For a thorough introduction to determinantal point processes, as well as an implementation of many sampling methods, see \citet{GPBV19}.
 
As the determinant of $\textup{L}_{Z,Z}$ corresponds to the volume of the parallelepiped in $\R^M$ formed by the columns in $Z$, $M$-determinantal point processes introduce strong negative correlations between points sampled. This leads to samples that are more dispersed than subsets selected uniformly (\cref{fig:dpp_init}). This intuition, as well as the following result due to \citet{belabbas_spectral_2009} serves as motivation for using an $M$-DPP in order to select the location of inducing points.
\begin{lem}[\citealp{belabbas_spectral_2009}, Theorem 1]\label{lem:belabbas}
Let $\textup{L}$ be a SPSD $N \times N$ matrix with eigenvalues $\eta_1 \geq \dots \geq \eta_N \geq 0$. Suppose a set of points are sampled according to an $M$-determinantal point process with kernel matrix $\textup{L}$. Define the (random) matrix $\textup{L}_\bfZ=\textup{L}_{\bfZ,N}\transpose \textup{L}_{\bfZ,\bfZ}^{-1}\textup{L}_{\bfZ,N}$ where $\textup{L}_{\bfZ,\bfZ}$ is the $M \times M$ principal submatrix of $L$ with columns in $\bfZ$ and $\textup{L}_{N, \bfZ}$ is the $N \times M$ matrix with columns $\bfZ$. Then,
\begin{equation}
    \Exp{}{\text{\emph{tr}}(\textup{L}-\textup{L}_\bfZ)} \leq (M+1) \sum_{m=M+1}^N \eta_m.
\end{equation}
\end{lem}
If $\textup{L}=\Kff$ and the inducing points are selected as a subset of data points corresponding to the columns selected by the $M$-DPP, then $\textup{L}_\bfZ=\RQff$. This tells us that using an $M$-DPP to choose inducing inputs will make $\conditionalexp{\Tr(\Kff-\RQff)}{\bfX=X}$ relatively close to its optimal value of $\sum_{m=M+1}^N \kffeigenvalue$.

\begin{figure}
    \centering
    \includegraphics[width=0.7\textwidth]{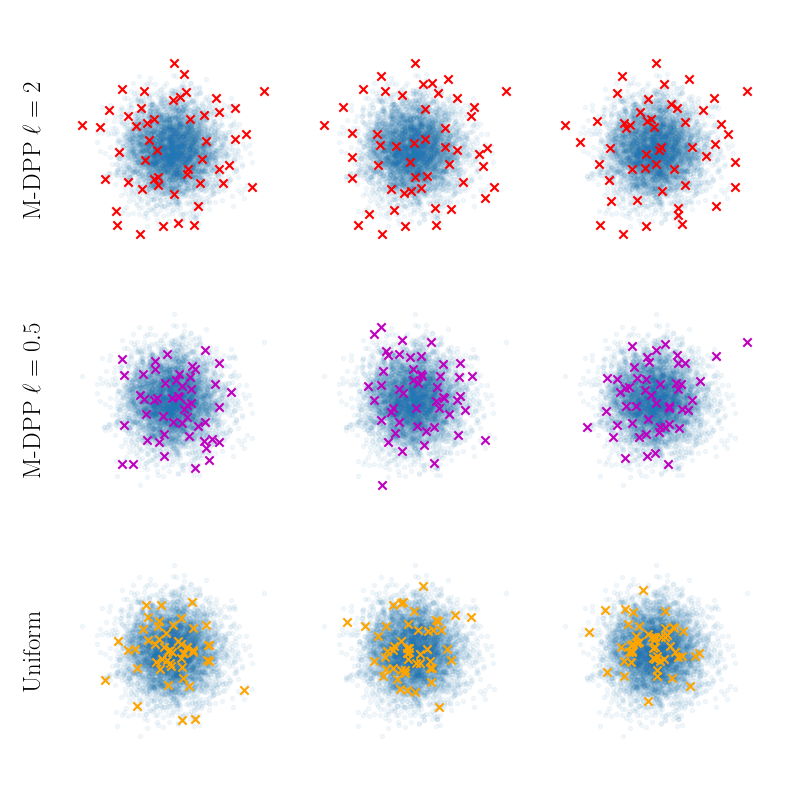}
    \caption{$M$-Determinantal point processes introduce a strong form of negative correlation between points, leading to samples that appear over-dispersed compared to uniform sampling. In the top two rows, we show (approximate) samples from a $M$-DPP with kernel matrix determined by a SE-kernel with two different length scales and Gaussian distributed covariates, with $50$ points drawn for each sample. Sampling is performed via MCMC after a greedy initialization. In the bottom row, we show subsets of size 50 selected uniformly from the covariates. }
    \label{fig:dpp_init}
\end{figure}

The next important question to address is whether a $M$-DPP can be sampled with sufficiently low computational complexity for this to be a practical method for selecting inducing inputs. Naively computing the probability distribution over all $\binom{N}{M}$ subsets of size $M$ is prohibitively expensive. \citet{kulesza2011k} gave an algorithm that runs in polynomial time and yields exact samples from an $M$-DPP.  Unfortunately, this algorithm involves computing an eigendecomposition of the $N\times N$ kernel matrix ($\Kff$ in our case), which is computationally prohibitive. 

Recently, \cite{NIPS2019_9330} gave an algorithm for obtaining an exact sample from an $M$-DPP in time that is polynomial in $M$ and nearly-linear in $N$. However, the polynomial in $M$ is high. We instead consider an approximate algorithm and therefore derive the following simple corollary of \cref{lem:belabbas}.
\begin{cor}\label{cor:approximate-kdpp}
Let $\rho$ denote an $M$-DPP with kernel matrix $L$, satisfying $L_{i,i}\leq v$. Let $\rho'$ denote a measure on subsets of columns of $L$ with cardinality $M$ such that $\mathrm{TV}(\rho, \rho') \leq \epsilon$ for some $\epsilon>0$, where $\mathrm{TV}(\rho, \rho') \coloneqq \frac{1}{2} \sum_{|Z|=M} \rho(Z) - \rho'(Z)$. Then
\[
\Exp{\rho'}{\text{\emph{tr}}(\textup{L}-\textup{L}_\bfZ)} \leq 2Nv \epsilon + (M+1) \sum_{m=M+1}^N \eta_m,
\]
where $\eta_m$ is the $m^{th}$ largest eigenvalue of $L$.
\end{cor}
\begin{proof}
\begin{align*}
\Exp{\rho'}{\Tr(\textup{L}-\textup{L}_\bfZ)} & = \sum_{|Z| =M} \Tr(\textup{L}-\textup{L}_Z) (\rho'(Z)+ \rho(Z)-\rho(Z)) \\
&= \Exp{\rho}{\Tr(\textup{L}-\textup{L}_\bfZ)} + \sum_{|Z| =M} \Tr(\textup{L}-\textup{L}_Z) (\rho'(Z)-\rho(Z)) \\
& \leq  \Exp{\rho}{\Tr(\textup{L}-\textup{L}_\bfZ)} + 2\mathrm{TV}(\rho, \rho')\max_{\substack{Z: |Z|=M}}\left( \Tr(\textup{L}-\textup{L}_Z)\right).
\end{align*}
The corollary is completed by noting that for all $Z$, $\Tr(\textup{L}-\textup{L}_Z) \leq \Tr(\textup{L}) \leq Nv$. 
\end{proof}

 \begin{algorithm}[tb]
  \caption{MCMC algorithm for approximately sampling from an $M$-DPP \citep{pmlr-v49-anari16}}
  \label{alg:det_init_mcmc}
  \begin{algorithmic}
    \STATE {\bfseries Input:} Training inputs $\dataX=\{x_i\}_{i=1}^N$, number of points to choose, $M$, kernel $k$, $T$ number of steps of MCMC to run.
  \STATE {\bfseries Returns:} An (approximate) sample from a $M$-DPP with kernel matrix $\Kff$ formed by evaluating $k$ at $\dataX$.
  \STATE Initialize $M$ columns by greedily selecting columns to maximize the determinant of the resulting submatrix. Call this set of indices of these columns $Z_0$.
  \FOR{$\tau \leq T$}
  \STATE Sample $i$ uniformly from $Z_{\tau}$ and $j$ uniformly from $\dataX \setminus Z_{\tau}.$ Define $Z'= Z_{\tau}\setminus \{i\}\cup\{j\},$
  \STATE Compute $p_{i \to j}:= \frac{1}{2}\min\{1, \text{det}(\textup{K}_{Z'})/\text{det}(\textup{K}_{Z_\tau})\}$
  \STATE With probability $p_{i \to j}$, $Z_{\tau+1}=Z'$ otherwise, $Z_{\tau+1}=Z_\tau$
  \ENDFOR
  \STATE {\bfseries Return:} $Z_{T}$
\end{algorithmic}
\end{algorithm}

\Cref{cor:approximate-kdpp} shows that sufficiently accurate approximate sampling from an $M$-DPP only has a small effect on the quality of the resulting $\Qff$. High quality approximate samples can be drawn using a simple Markov Chain algorithm described in \citet{pmlr-v49-anari16}, given as \cref{alg:det_init_mcmc}. 
This MCMC algorithm is well-studied in the context of $M$-DPPs and their generalizations, and is known to be rapidly mixing \citep{pmlr-v49-anari16, Hermon2019ModifiedLI}.
\begin{lem}[\citealp{Hermon2019ModifiedLI}, Corollary 1]\label{lem:anari-approx-k-dpp}
Let $\rho$ be an $M$-DPP with $N\times N$ kernel matrix $L$. Fix $\epsilon \in (0,1)$. Then \cref{alg:det_init_mcmc} produces a sample from a distribution $\rho'$ satisfying
\[
\mathrm{TV}(\rho, \rho') \leq \epsilon
\]
in not more than $T(\epsilon)=2MN\left( 
\log\log\left(\frac{1}{\rho(Z_0)}\right)+ \log \frac{2}{\epsilon^2}\right)$ iterations, where $Z_0$ is the subset of columns at which the Markov chain is initialized.
\end{lem}
Since the determinant of a matrix is equal to the product of the determinant of a principal submatrix times the determinant of the Schur complement of this submatrix, the greedy initialization used in \cref{alg:det_init_mcmc} is equivalent to starting with $U=\emptyset$ and iteratively adding $\argmax_{x\in X} k(x,x) - \textup{k}_{f(x)u}\Kuu\inv \textup{k}_{uf(x)}$ to $U$. This can be performed in time $\mcO(NM^2)$, for example by computing the pivot rules of a rank-$M$ incomplete Cholesky decomposition of $\Kff$ \citep[][Algorithm 1]{chen2018fast}.

The per iteration cost of \cref{alg:det_init_mcmc} is dominated by computing the acceptance ratio, which can be performed in $\mcO(M^2)$, by iteratively updating a Cholesky or QR factorization of the matrix associated to the current set of columns. This makes the total cost of obtaining an $\epsilon$-approximate sample $\mcO\left(NM^3\log \log\left(1/\rho(Z_{\textup{greedy}})\right)+ NM^3\log 2/\epsilon^2\right)$, where $Z_{greedy}$ denotes the set of columns selected by greedily maximizing the determinant of the submatrix. Moreover, the subset selected by the algorithm is known to have a probability at least $1/(M!)^2$ of the maximum probability subset \citep{civril2009selecting,pmlr-v49-anari16}. By using the fact that the the maximum probability subset is more probable than the uniformly distributed probability, we obtain
\[
\rho(Z_{\textup{greedy}}) \geq \left(M!^2\binom{N}{M}\right)^{-1} \geq (MN)^{-M},\]
giving an overall complexity of not more than $\mcO\left(NM^3(\log\log N+\log M+\log 1/\epsilon^2)\right)$.
This gives us a method for initializing inducing points conditioned on input locations, such that we can relate $\conditionalexp{\Tr(\RKff-\RQff)}{\bfX=X}$ to $\Tr(\Kff-\mathrm{K}_{M})$. 

We now take a brief detour to consider a different approach to initializing inducing inputs before completing the proof of a priori bounds on the KL-divergence.

\subsubsection{Ridge Leverage Scores}
While using an $M$-DPP to select inducing inputs allows us to bound $\conditionalexp{\Tr(\Kff-\RQff)}{\bfX=X}$, this method has a significant drawback as opposed to other methods of initialization: the computational cost of running the MCMC algorithm to obtain approximate samples dominates the cost of sparse inference. Ridge leverage score (RLS) sampling offers an alternative that runs in $\mcO(NM^2)$, while retaining strong theoretical guarantees on the quality of the resulting approximation. In this section, we give a brief discussion of ridge leverage scores as well an algorithm of \citet{Musco_ridge_leverage} that allows for efficient approximations to ridge leverage scores.

The $\omega$-ridge leverage score of a point $x_n \in \dataX$ of a Gaussian process regressor, which we denote by $\ell^\omega(x_n)$ is defined as $1/\omega$ times the posterior variance at $x_n$ of the process with noise variance $\omega$, i.e.
\begin{equation*}
   \ell^\omega(x_n) = \frac{1}{\omega}\left(k(x_n,x_n)-\mathrm{k}_{\textup {nf}}\transpose(\Kff+\omega\bfI)^{-1}\mathrm{k}_{\textup{nf}}\right), 
\end{equation*}
where $\mathrm{k}_{\textup {nf}}\transpose= [k(x_1,x_n), k(x_2,x_n), \dotsc, k(x_n,x_n)]$. 
RLS sampling uses these values as an importance distribution for selecting which points to include in sparse kernel methods. Intuitively, points at which there is high posterior uncertainty must be `far' from other points, and therefore informative. 

Computing the ridge leverage scores exactly is too computationally expensive, as it involves inverting the kernel matrix. However, practical approximate versions of leverage sampling algorithms that retain strong theoretical guarantees have been developed. 

Ridge leverage based sampling algorithms select a subset of training data to use as inducing points. Each point is sampled independently into the subset with probability proportional to its leverage score. Approximate versions of this algorithm generally rely on overestimating the ridge leverage scores, which lead to equally strong accuracy guarantees compared to using the exact ridge leverage scores, at the cost of sampling more points in the approximation.

We consider the application of Algorithm 3 in \citet{Musco_ridge_leverage} to the problem of selecting inducing inputs for sparse variational inference in GP models. This algorithm comes with the following bounds on the quality of the resulting Nystr\"om approximation. 
\begin{lem}[\citealp{Musco_ridge_leverage}, Theorem 14, Appendix D]\label{lem:ridge-leverage}
Given $X \in \mcX^N$ and a kernel $k$, let $\Kff$ denote the $N \times N$ covariance matrix associated to $X$ and $k$. Fix $\delta \in (0,\frac{1}{32})$ and $S \in \mathbb{N}$. There exists a universal constant $c$ and algorithm with run time $\mcO(N\bfM^2)$ and memory complexity $\mcO(N\bfM)$ that with probability $1-3\delta$ returns $\bfM \leq cS \log (S/\delta)$ columns of $\Kff$ such that the resulting Nystr\"om approximation, $\RQff$, satisfies
\[
\|\Kff-\RQff\|_{\textup{op}} \leq \frac{1}{S}\sum_{m=S+1}^N \kffeigenvalue.
\]
\end{lem}

While in \cref{sec:kdpp} $M$ was fixed and the quality of the resulting approximation was random, in the algorithm discussed above $\bfM$ is additionally random. 

An alternative approach to sampling using ridge leverage scores specifies a desired level of accuracy of the resulting approximation, and the number of points selected is chosen to obtain this approximation quality with fixed probability. This has the advantage of not requiring the user to manually select the number of inducing points, but may lead to a number of inducing points being used that exceeds a practical computational budget. We discuss the application of this approach to variational Gaussian process regression in \cref{app:adaptive-leverage}.

\subsection{A-Priori Bounds on the KL-divergence}

In the previous sections, the results on the quality of approximation depended on the eigenvalues of $\Kff$. As these eigenvalues depend on the covariates $X$, and we would like to make statements that apply to a wide-range of data sets, we assume $X$ is a realization of a random variable $\bfX$, and make assumptions about the distribution of $\bfX$.

If each $\bfx \in \bfX$ is i.i.d.~distributed, according to some measure with continuous density $p(x)$, in the limit as the amount of data tends to infinity, the matrix $\frac{1}{N}\RKff$ behaves like the operator $\mcK$ \citep{koltchinskii2000random} defined with respect to this $p$. For finite sample sizes, the large eigenvalues of $\frac{1}{N}\RKff$ tend to overestimate the corresponding eigenvalues of $\mcK$ and the small eigenvalues of $\frac{1}{N}\RKff$ tend to underestimate the small eigenvalues of $\mcK$. We make this precise through a minor generalization of a lemma of \citet{shawe2005eigenspectrum}.

\begin{lem}\label{lem:average-eigenvalues}
Suppose that $N$ covariates are distributed in $\R^D$ such that the marginal distribution, $\mu_{x_n}$, of each covariate, $\bfx_n$ has a continuous density $p_n(x)$, and there exists a distribution with continuous density $q(x)$ satisfying $p_n(x)<c_nq(x)$ for some $c_n>0$ for all $n$. Let $\tilde{\bm{\lambda}}_m$ denote the $m^{th}$ largest eigenvalue of the random matrix $\RKff$ formed by a continuous, bounded kernel and these covariates. Let $\lambda_m$ denote the $m^{th}$ largest eigenvalue of the integral operator corresponding to the distribution with density $q$, $\mcK_q$. Then, for any $M \geq 1$
\[
\Exp{}{\frac{1}{N}\sum_{m=M+1}^N \tilde{\bm{\lambda}}_m} \leq \bar{c}\sum_{m=M+1}^\infty \lambda_m,
\]
where $\bar{c} = \frac{1}{N} \sum_{n=1}^Nc_n$.
\end{lem}
\begin{proof}
For $\bfX$ taking values in $(\R^D)^N$, and any rank-$M$ matrix SPSD $\bm{\Phi} \prec \RKff$, we have 
\[
\frac{1}{N}\sum_{m=M+1}^N \tilde{\bm{\lambda}}_m = \frac{1}{N}\Tr(\RKff-\bm{\mathrm{K}}_M) \leq \frac{1}{N}\Tr(\RKff - \bm{\Phi}),
\]
with $\bfK_M$ defined as in \cref{sec:eigenvec-features}. The inequality follows form the optimality of $\bm{\mathrm{K}}_M$ as a rank-$M$ approximation to $\RKff$ in the Schatten-1 norm (sum of absolute value of singular values).

As $k$ is a continuous bounded kernel we can apply Mercer's theorem to represent $k(x,x')$ with respect to the eigenfunctions of the operator $\mcK_q$ giving $[\RKff]_{i,j} = \sum_{m=1}^\infty \lambda_m \phi_m(\bfx_i)\phi_m(\bfx_j)$, and  $\Exp{q}{\phi(\bfx_n)^2}=1$. 

Consider the rank-$M$ approximation to $\RKff$ given by truncating this Mercer expansion, $[\bm{\Phi}]_{i,j} = \sum_{m=1}^M \lambda_m \phi_m(\bfx_i)\phi_m(\bfx_j)$. Then $[\RKff-\bm{\Phi}]_{i,j} = \sum_{m=M+1}^\infty \lambda_m \phi_m(\bfx_i)\phi_m(\bfx_j)$, so $\RKff-\bm{\Phi} \succ 0$. 

For any covariates $\{\bfx_n\}_{n=1}^N$ satisfying the conditions of the lemma, 
\begin{align*}
\frac{1}{N}\sum_{m=M+1}^N \tilde{\bm{\lambda}}_m \leq \frac{1}{N} \Tr(\RKff - \bm{\Phi}) =  \frac{1}{N}\sum_{n=1}^N\sum_{m=M+1}^\infty \lambda_m \phi_m(\bfx_n)^2.
\end{align*}
Taking expectations on both sides with respect to the covariate distribution,
\begin{align*}
\Exp{}{\frac{1}{N}\sum_{m=M+1}^N \tilde{\bm{\lambda}}_m}&  \leq \frac{1}{N}\sum_{n=1}^N\sum_{m=M+1}^\infty\lambda_m\int  \phi_m(x)^2 p_n(x)\calcd x \\
&\leq \frac{1}{N}\sum_{n=1}^Nc_n\sum_{m=M+1}^\infty\lambda_m\int  \phi_m(x)^2 q(x)\calcd x
= \bar{c}\sum_{m=M+1}^\infty\lambda_m\ ,
\end{align*}
The interchanging of integral and sum is justified by Fubini's theorem as each $\phi_m$ is square integrable, each eigenvalue is non-negative, and the sum converges by Mercer's theorem. We used the non-negativity of $\phi_m(x)^2$ in the second inequality to bound the expectation of $\phi_m(x)^2$ under $p_n$ in terms of its expectation under $q$.
\end{proof}
\begin{cor}\label{cor:avg-eigvals}
Suppose the covariate distribution has identically distributed marginals, each with density $p(x)$, then 
\[
\Exp{}{\frac{1}{N}\sum_{m=M+1}^N \ \tilde{\bm{\lambda}}_m}  \leq \sum_{m=M+1}^\infty \lambda_m,
\]
where $\lambda_m$ is the $m^{th}$ largest eigenvalue of the operator associated to the kernel and the distribution with continuous density $p(x)$.
\end{cor}
This corollary follows from \cref{lem:average-eigenvalues} by taking $q=p$ and $c_n=1$ for all $n$. For simplicity, we will state our main results using the assumptions of this corollary, though the generalization to cases with non-identical marginals satisfying the conditions of \cref{lem:average-eigenvalues} is immediate.
We have now accumulated the necessary preliminaries to prove our main theorems.
\subsubsection{Bounds on the KL-divergence for \emph{M}-DPP Sampling}
\begin{thm}\label{thm:upper-bound-fixed-y}
Suppose $N$ training inputs are drawn according to a distribution on $\R^D$ with identical marginal distributions, each with density $p(x)$. Let $k$ be a continuous kernel such that $k(x,x)<v$ for all $x \in \R^D$. Suppose $\bfy$ is distributed such that $\conditionalexp{\|\bfy\|_2^2}{\bfX} \leq  RN$ almost surely for some $R\geq 0$. Sample $M$ inducing points from the training data according to an $\epsilon$-approximation to a $M$-DPP with kernel matrix $\Kff$. Then,
\begin{equation}
	\Exp{}{\KL{Q}{P}} \leq \frac{1}{2}\left(1+ \frac{RN}{\noisevariance}\right) \frac{(M+1)N\sum_{m=M+1}^\infty \lambda_m+2Nv\epsilon}{\noisevariance},
\end{equation}
where the expectation is taken over the covariates, the mechanism for initializing inducing points and the observations.
\end{thm}

\begin{proof}[Proof of \cref{thm:upper-bound-fixed-y}]
We use \cref{lem:agnostic-upper-bound}, \cref{cor:approximate-kdpp} and take expectations with respect to $\bfZ$, noting that $\bfZ|\bfX$ is independent of $\bfy|\bfX$ so that $\conditionalexp{\Tr(\RKff-\RQff)}{\bfX}=\conditionalexp{\Tr(\RKff-\RQff)}{\bfy,\bfX}$,
\begin{align}
    \conditionalexp{\KL{Q}{P}}{\bfy,\bfX}  \leq \frac{N}{2\noisevariance}\left(1+\frac{\|\bfy\|^2_2}{\noisevariance}\right)\left(\frac{M+1}{N}\sum_{m=M+1}^N \tilde{\bm{\lambda}}_m + 2v \epsilon\right).  
\end{align}
Now using the assumption that $\conditionalexp{\|\bfy\|_2^2}{\bfX} \leq RN$ almost surely, and taking expectation over $\bfy$,
\begin{align}
    \conditionalexp{\KL{Q}{P}}{\bfX}  \leq \frac{N}{2\noisevariance}\left(1+\frac{RN}{\noisevariance}\right)\left(\frac{M+1}{N}\sum_{m=M+1}^N \tilde{\bm{\lambda}}_m + 2v \epsilon\right).  
\end{align}
Finally, taking expectation with respect to the covariate distribution over the covariate distribution and applying \cref{lem:average-eigenvalues},
\begin{align}
    \Exp{}{\KL{Q}{P}}  \leq  \frac{1}{2}\left(1+\frac{RN}{\noisevariance}\right)\left(\frac{N(M+1)\sum_{m=M+1}^\infty \lambda_m + 2Nv \epsilon}{\noisevariance}\right). 
\end{align}
\end{proof}
\begin{thm}\label{thm:upper-bound-average-y}
With the same assumptions on the covariates and inducing point distributions as in \cref{thm:upper-bound-fixed-y}, but with the assumption that $\bfy|\bfX$ is conditionally Gaussian distributed with mean zero and covariance matrix $\Kff+\noisevariance \bfI$,
\begin{equation}
	\Exp{}{\KL{Q}{P}} \leq \frac{(M+1)N\sum_{m=M+1}^\infty \lambda_m+2Nv\epsilon}{\noisevariance}
\end{equation}
where the expectation is taken over the covariate distribution, the observation distribution and the initialization mechanism. 
\end{thm}
The proof of \cref{thm:upper-bound-average-y} is nearly identical to the proof of \cref{thm:upper-bound-fixed-y}, applying \cref{lem:kl-averaged} instead of \cref{lem:agnostic-upper-bound} in the first line.

In certain instances, it may be desirable to have a bound that holds with fixed probability instead of in expectation. As $\KL{Q}{P}\geq 0$, such a bound can be derived through applying Markov's inequality to \cref{thm:upper-bound-fixed-y} or \cref{thm:upper-bound-average-y} leading to the following corollaries:
\begin{cor}\label{cor:probabilistic-fixed-y}
Under the assumptions of \cref{thm:upper-bound-fixed-y}, with probability at least $1-\delta$,
\[
\KL{Q}{P} \leq \frac{1}{2}\left(1+ \frac{RN}{\noisevariance}\right) \frac{(M+1)N\sum_{m=M+1}^\infty \lambda_m+2Nv\epsilon}{\delta\noisevariance}. 
\]
\end{cor}
\begin{cor}\label{cor:probabilistic-average-y}
Under the assumptions of \cref{thm:upper-bound-average-y}, with probability at least $1-\delta$,
\begin{equation*}
\KL{Q}{P}\leq \frac{(M+1)N\sum_{m=M+1}^\infty \lambda_m+2Nv\epsilon}{\delta\noisevariance}.
\end{equation*}
\end{cor}

\subsubsection{Bounds for Ridge Leverage Score Sampling}
We now state and derive statements similar to \cref{cor:probabilistic-fixed-y,cor:probabilistic-average-y} for a ridge leverage score initialization utilizing \citet[Algorithm 3]{Musco_ridge_leverage}. In order to this we us that for any SPSD $A$, $\Tr(A) \leq N\|A\|_{\textup{op}}$, so that
\Cref{lem:agnostic-upper-bound} implies
\begin{equation}\label{eqn:ridge-lev-fixed-y-KL}
    \KL{Q}{P} \leq  \frac{\|\RKff-\RQff\|_{\textup{op}}}{2\noisevariance}\left(N+\frac{\|\bfy\|_2^2}{\noisevariance}\right),
\end{equation}
and \cref{lem:kl-averaged} implies,
\begin{equation}\label{eqn:ridge-lev-avg-KL}
\conditionalexp{\KL{Q}{P}}{\bfZ,\bfX} \leq N\frac{\|\RKff-\RQff\|_{\textup{op}}}{\noisevariance}.
\end{equation}
Combining \cref{lem:ridge-leverage,cor:avg-eigvals} and using Markov's inequality twice with \cref{eqn:ridge-lev-fixed-y-KL} or \cref{eqn:ridge-lev-avg-KL} and a union bound respectively leads to the following bounds on the performance of sparse inference using ridge leverage scores:
\begin{thm}\label{thm:ridge-leverage-agnostic-y}
Take the same assumptions on $\bfX$ and $\bfy|\bfX$ as in \cref{thm:upper-bound-fixed-y}. Fix $\delta \in (0,1/32)$ and $S \in \mathbb{N}$. There exists a universal constant $c$ such with probability $1-5\delta$, we have $\bfM<cS\log(S/\delta)$ and
\[
\KL{Q}{P} \leq\frac{1}{2}\left(N+\frac{RN}{\noisevariance}\right) \frac{N\sum_{m=S+1}^\infty \lambda_m}{S\delta^2\noisevariance }
\]
when inducing points are initialized using \citet[Algorithm 3]{Musco_ridge_leverage}.
\end{thm}
\begin{thm}\label{thm:ridge-leverage-average-y}
Take the same assumptions on $\bfX$ and $\bfy|\bfX$ as in \cref{thm:upper-bound-average-y}. Fix $\delta \in (0,1/32)$ and $S \in \mathbb{N}$. There exists a universal constant $c$ such with probability $1-5\delta$, we have $\bfM<cS\log(S/\delta)$ and
\[
\KL{Q}{P} \leq \frac{N^2\sum_{m=S+1}^\infty \lambda_m}{S\delta^2\noisevariance}
\]
when inducing points are initialized using \citet[Algorithm 3]{Musco_ridge_leverage}.
\end{thm}

\subsubsection{Are these bounds useful?}
Having established probabilistic upper bounds on the KL-divergence resulting from sparse approximation, a simple question is whether these bounds offer any insight into the efficacy of sparse inference. If in order for the upper bounds to be small, we need to take $M=N$, then they would not be useful, as it is already known that by taking $Z=\dataX$, exact inference is recovered. In the next section, we discuss bounds on the eigenvalues of $\mcK$ for common kernels and input distribution. These bounds show that for many inference problems, the upper bounds in \cref{thm:upper-bound-fixed-y,thm:upper-bound-average-y,thm:ridge-leverage-agnostic-y,thm:ridge-leverage-average-y} imply that the KL-divergence can be made small with $M \ll N$ inducing points.

%%%%%%%%%%%%%%%%%%%%%%%%%%EXAMPLES OF Upper BOUNDS %%%%%%%%%%%%%%%%%%%%%%%%%%%%%%
\section{Bounds for Specific Kernels and Covariate Distributions}\label{sec:examples}

\begin{table}
\begin{center}
\begin{small}
\begin{sc}
\begin{tabular}{lcccr}
\toprule
Kernel & Covariate Distribution & M, Theorem \ref{thm:upper-bound-fixed-y} & M, Theorem \ref{thm:ridge-leverage-agnostic-y}  \\
\midrule
SE-Kernel  & Compact support  & $\mcO((\log N)^D)$ & $\mcO((\log N)^D\log \log N)$\\
SE-Kernel &  Gaussian & $\mcO((\log N)^D)$ & $\mcO((\log N)^D\log \log N)$\\
Mat\'{e}rn $\nu$  & Compact support &  $\mcO(N^\frac{2D}{2\nu-D})$ & $\mcO(N^\frac{2D}{2\nu+D}\log N)$\\
\bottomrule
\end{tabular}
\end{sc}
\end{small}
\end{center}
\caption{The number of features needed for upper bounds to converge in $D-$dimensions. We assume the compactly supported distributions have bounded density functions.\protect\footnotemark}\label{table:numfeats}
\end{table}
\footnotetext{{\citet{burt2019rates} stated bounds for a product of one-dimensional Mat\'ern kernels, which differs from the commonly used multivariate Mat\'ern kernel.}}

In this section, we consider specific covariate distributions and commonly used kernels, and investigate the implications of the upper bounds derived in \cref{sec:regression-rates}. These results are summarized in \cref{table:numfeats}. We begin with the case of the popular squared exponential kernel and Gaussian covariates in one-dimension. This kernel and covariate distribution are one of the few instances in which the eigenvalues of $\mcK$ have a simple analytic form. In \cref{sec:upper-multivariate-se-gauss}, we consider the analogous multi-dimensional problem. In \cref{sec:stationary_kernels} we discuss implications for stationary kernels with compactly supported inputs, including the well-studied Mat\'ern kernels.

%%%GAUSSIAN KERNEL%%
\subsection{Squared Exponential Kernel and Gaussian Covariate Distribution}\label{sec:se-gaussian}
\begin{figure}
    \centering
    \includegraphics[width=.92\textwidth]{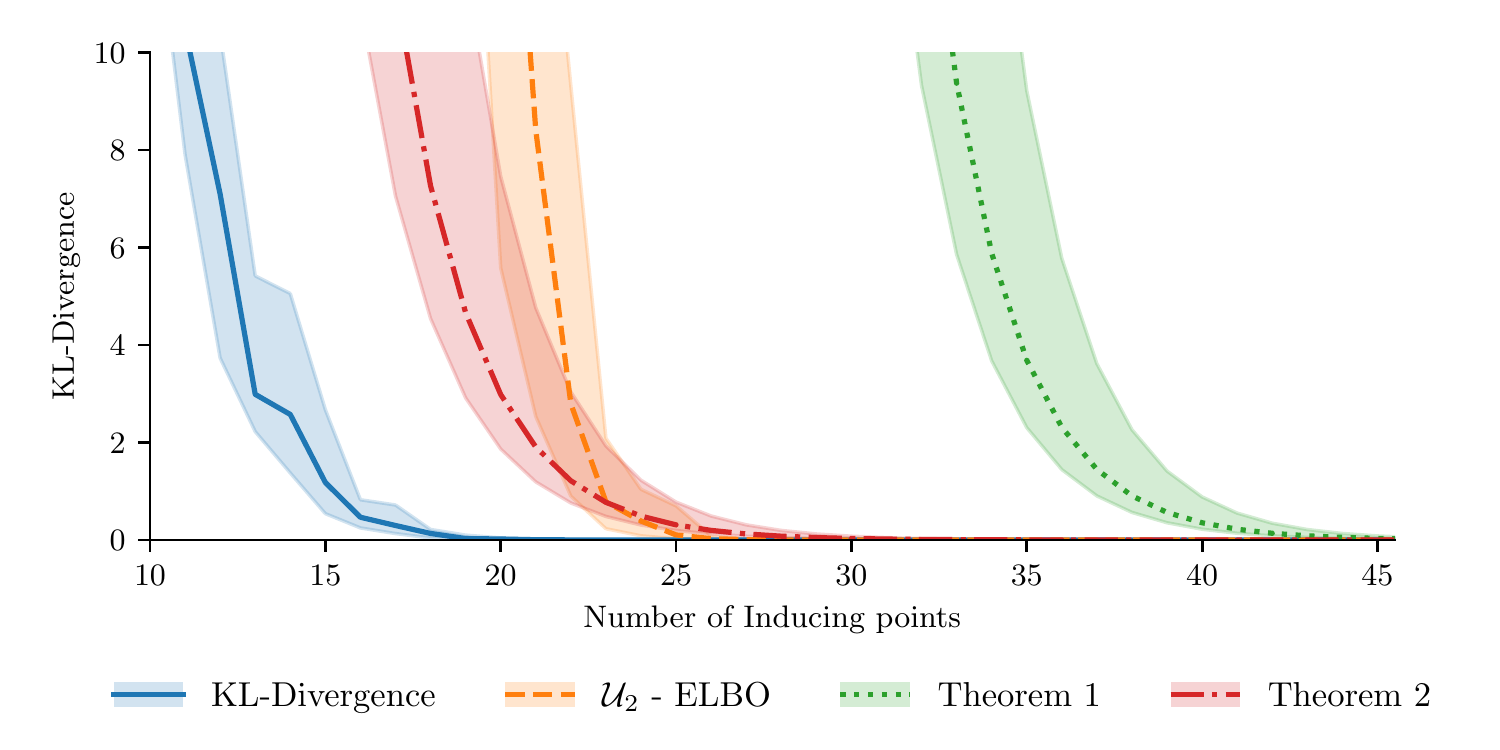}
    \caption{A comparison of the actual KL-divergence achieved, the bound given by \cref{lem:titsias-aposteriori} and the bounds derived in \cref{cor:probabilistic-fixed-y,cor:probabilistic-average-y} squared exponential kernel. For the actual KL-divegence and the bound given by \cref{lem:titsias-aposteriori}, the dotted line shows the median of 20 independent trials and the shaded region shows the $20\%-80\%$ regions, while for the bounds, the dotted line represents $\delta=.5$ and the shaded region $\delta \in [.2,.8]$.}
    \label{fig:se-gaussian1d-fixedn}
\end{figure}

In the case of the squared exponential kernel, with lengthscale parameter $\ell$ and variance $v$, that is
\[
k_{\textup{SE}}(x,x') = v \exp\left(-\frac{(x-x')^2}{2\ell^2}\right)
\]
and one-dimensional covariates distributed according to  $\mcN(0,\beta^2)$, the eigenvalues of $\mcK$ are \citep{zhu_gaussian_1997}
\begin{equation}
    \lambda_m = \sqrt{\frac{2a}{A}}B^{m-1} \label{eqn:se-1d-eigvals} \,,
\end{equation}
where $a=(4\beta^2)^{-1}\!, b = (2\ell^2)^{-1}\!, A = a+ b + \sqrt{a^2+4ab}$ and $B=b/A$. Note that $B<1$ for any $\ell^2,\beta^2>0$, so the eigenvalues of this operator decay geometrically. However, the exact value of $B$ depends on the lengthscale of the kernel and the variance of the covariate distribution. Short lengthscales and high standard deviations lead to values of $B$ close to $1$, which means that the eigenvalues decay more slowly. From a practical perspective, it is important to keep this in mind, as while the particular rates we obtain on how $M$ should grow as a function of $N$ do not depend on the model hyperparameters, the implicit constants do.

\begin{restatable}[]{cor}{seoned}\label{cor:gaussian-1D}
Let $k$ be a squared exponential kernel. Suppose that $N$ real-valued (one-dimensional) covariates are observed, with identical Gaussian marginal distributions. Suppose the conditions of \cref{thm:upper-bound-fixed-y} are satisfied for some $R>0$. Fix any $\gamma\in (0,1]$. Then there exists an $M=\mcO(\log (N^3/\gamma))$ and an $\epsilon= \Theta(\gamma/N^2)$ such if inducing points are distributed according to an $\epsilon$-approximate $M$-DPP with kernel matrix $\Kff$,
\[
\Exp{}{\KL{Q}{P}} \leq \gamma.
\]
Similarly, for any $\delta \in (0,1/32)$ using the ridge leverage algorithm of \cite{Musco_ridge_leverage} and choosing $S$ appropriately, with probability $1-5\delta$, $\bfM=\mcO\left(\log \frac{N^2}{\delta^2\gamma}\log \frac{\log (N^2/\delta^2\gamma)}{\delta}\right)$ and
\[
\KL{Q}{P} \leq \gamma.
\]
The implicit constants depend on the kernel hyperparameters, the likelihood variance, the variance of the covariate distribution and $R$. 
\end{restatable}

\begin{rem}
If we consider $\gamma$ and $\delta$ as fixed constants (independent of $N$), this implies that if inducing points are placed using an approximate $M$-DPP we can choose $M=\mcO(\log(N))$ inducing points leading to a computational cost of $\mcO(N(\log N)^4)$ while for approximate ridge leverage scores sampling $\mcO(\log N \log \log N)$ inducing points suffice leading to a cost at most $\mcO(N(\log N)^2 (\log \log N)^2)$.
\end{rem}

The proof (\cref{app:sqexp-gaussian}) consists of applying the geometric series formula to evaluate the sum of eigenvalues and choosing $M, \epsilon$ and $S$ appropriately. All dependencies of the implicit constants on hyperparameters can be made explicit. \Cref{fig:se-gaussian1d-fixedn} illustrates the KL-divergence, the a posteriori bound given by $\mcU_2 -\mathrm{ELBO}$ and the bounds from \cref{thm:upper-bound-fixed-y,thm:upper-bound-average-y} in the case of a SE kernel and synthetic 1D distributed covariates.

\Cref{cor:gaussian-1D} is illustrated in \cref{fig:se_gaussian_1D}, in which we increase $N$ and increase $M$ logarithmically as a function of $N$ in such a way that $\KL{Q}{P}$ can be bounded above by a decreasing function in $N$. 

\begin{figure}
    \centering
    \includegraphics[width=.85\textwidth]{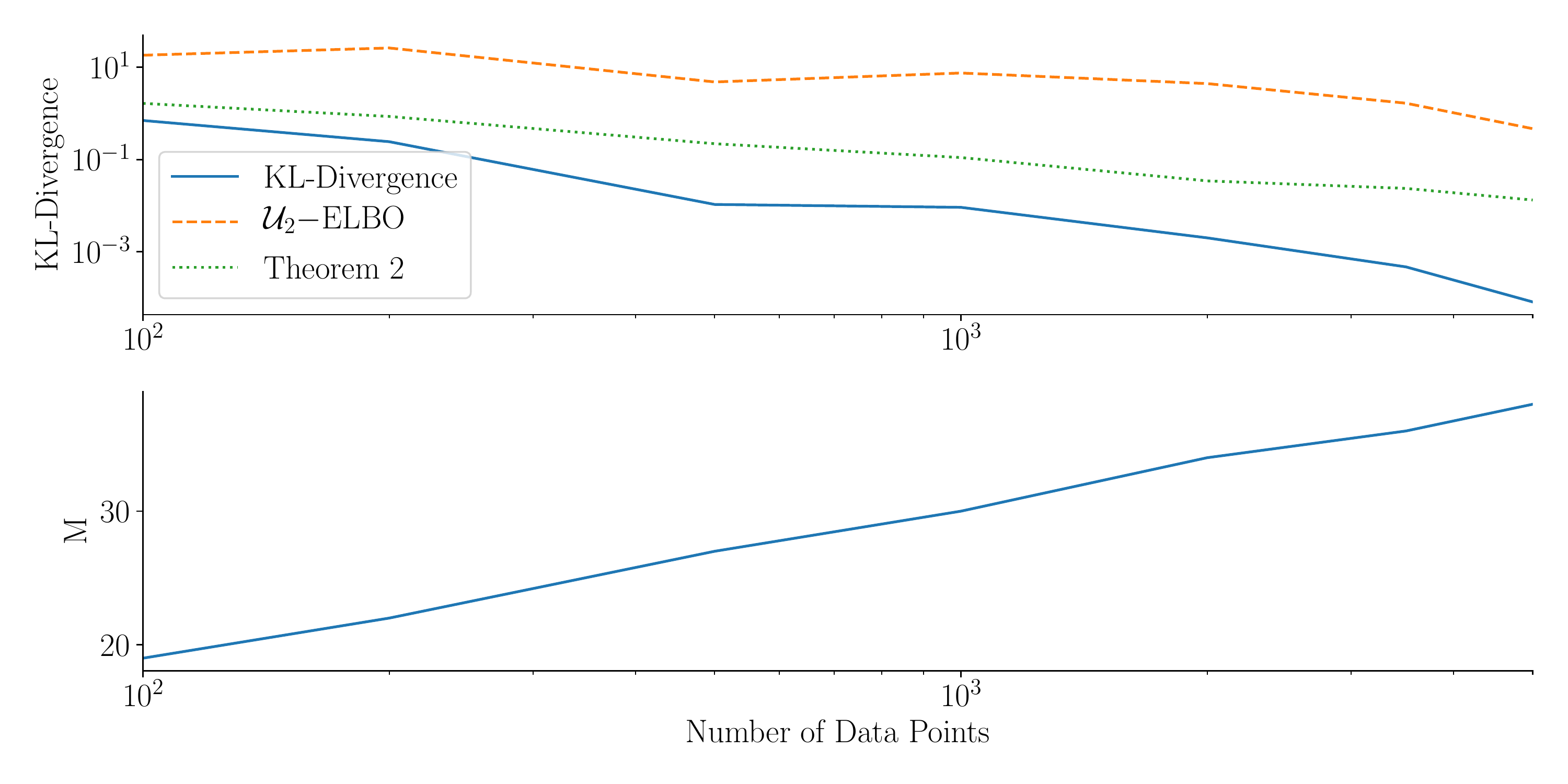}
    \caption{To illustrate \cref{cor:gaussian-1D}, we incrementally increase the size of the data set, and set the number of inducing points to grow as the log of the data set size. The top plot shows the KL-divergence (blue) for each data set as $N$ increases, plotted on a log-log scale. The bound in \cref{thm:upper-bound-average-y} (in green) tends to $0$. We also compute the a posteriori upper bound, \cref{lem:titsias-aposteriori} (yellow). The bottom plot shows the number of inducing points, plotted against the number $N$ with $N$ on a log-scale.}
    \label{fig:se_gaussian_1D}
\end{figure}

\subsubsection{The Multivariate Case}\label{sec:upper-multivariate-se-gauss}
 The generalization of \cref{cor:gaussian-1D} to the case of multi-dimensional input distributions is relatively straightforward. The multi-dimensional version of the squared exponential kernel can be written as a product of one dimensional kernels, i.e.
\[
k_{\mathrm{SEARD}}(x,x') = v\exp\left(-\sum_{d=1}^D\frac{(x_d-x'_d)^2}{\ell^2_d}\right) = v\prod_{d=1}^D \exp\left(-\frac{(x_d-x'_d)^2}{\ell^2_d}\right),
\]
where $\ell_d>0$ for all $d$.

For any kernel that can be expressed as a product of one-dimensional kernels, and for any covariate distribution that is a product of one-dimensional covariate distributions, the eigenvalues of the multi-dimensional covariance operator is the product of the one-dimensional analogues. 
When obtaining rates of convergence, we lose no generality in assuming that the kernel is isotropic as is the covariate distribution. Otherwise, consider the direction with the shortest lengthscale, and the covariate distribution with the largest standard deviation and the eigenvalues of this operator are larger than a constant multiple of the corresponding eigenvalues of the non-isotropic operator.

In the isotropic case, each eigenvalue is of the form,
\[
\lambda_m = \left(2a/A\right)^{D/2} B^{m'},
\]
for some integer $m'$ with $a,A$ and $B$ defined as in the one-dimensional case. Note that $m$ and $m'$ are no longer equal. The number of times each eigenvalue with $m'$ in the exponent is repeated is equal to the number of ways to write $m'$ as a sum of $D$ non-negative integers. 
By counting the multiplicity of each eigenvalue,  \citet{seeger2008information} arrived at the bound
\[
\lambda_{m+D-1} \leq \left(2a/A\right)^{D/2} B^{m^{1/D}}.
\]
In order to prove a multi-dimensional analogue of \cref{cor:gaussian-1D} we need an upper bound on $\sum_{m=M+1}^\infty \lambda_m$. This can be derived with following an argument made by \citet[Appendix II]{seeger2008information}.
\begin{restatable}{prop}{taileigvalsgauss}\label{prop:tail-eigenvalues-gaussian}
For a SE-kernel and Gaussian distributed covariates in $\R^D$, for $M\geq \frac{1}{\alpha}D^D+D-1$, $\sum_{m=M+1}^\infty \lambda_m = \mcO(M\exp(-\alpha M^{1/D})),$ where $\alpha = -\log B > 0$ and the implicit constant depends on the dimension of the covariates, the kernel parameters and the covariance matrix of the covariate distribution.
\end{restatable}

The proof of \cref{prop:tail-eigenvalues-gaussian} is in \cref{app:sqexp-gaussian}. 

\begin{restatable}{cor}{multivariategaussiansecor}\label{cor:gaussian-multi-D}
Let $k$ be a SE-ARD kernel in $D$-dimensions. Suppose that $N$ $D$-dimensional covariates are observed, so that each covariate has an identical multivariate Gaussian distribution, and that the distribution of training outputs satisfies $\conditionalexp{\|\bfy\|^2}{\bfX} \leq RN$. Fix any $\gamma \in (0,1]$. Then there exists an $M=\mcO((\log N/\gamma)^D)$ and an $\epsilon= \mcO(N^2/\gamma)$ such if inducing inputs are distributed according to an $\epsilon$-approximate $M$-DPP with kernel matrix $\Kff$,
\[
\Exp{}{\KL{Q}{P}} \leq \gamma.
\]
The implicit constant depends on the kernel hyperparameters, the variance matrix of the covariate distribution, $D$ and $R$. 
With the same assumptions but applying the RLS algorithm of \citet{Musco_ridge_leverage} to selecting inducing inputs, for any $\delta \in (0,1/32)$ there exists a choice of $S$ such that with probability $1-5 \delta$, $\bfM =  \mcO\left(\left(\log \frac{N^2}{\delta \gamma}\right)^D(\log\log \frac{N^2}{\delta \gamma} + \log (1/\delta) \right)$ and 
\[
\KL{Q}{P} \leq \gamma.
\]
\end{restatable}
 The proof follows from \cref{prop:tail-eigenvalues-gaussian} and \cref{thm:upper-bound-fixed-y} or \cref{thm:ridge-leverage-agnostic-y}, by choosing parameters appropriately.
\begin{rem}
If we allow the implicit constant to depend on $\gamma$ and $\delta$, this implies that for inducing ploints distributed accoding to an approximate $M$-DPP we can choose $M=\mcO((\log N)^D)$ inducing points leading to a computational cost of $\mcO(N(\log N)^{3D+1})$ while for approximate ridge leverage scores sampling $\mcO((\log N)^D \log \log N)$ inducing points suffice leading to a $\mcO(N(\log N)^{2D} (\log \log N)^2)$ computational cost.
\end{rem}

In order for the KL-divergence to be less than a fixed constant, the exponential scaling of the number of inducing points in the dimensions of the covariates is inevitable, as we will show in \cref{sec:lower-bounds}. However, practically the situation may not be quite so dire. First, many practioners use a SE-ARD kernel. If the data is essentially constant over many dimensions, then when training with empirical Bayes, the lengthscales of these dimensions tends to become large, effectively reducing the dimensionality of the inference problem. Additionally, in the case when covariates fall on a smooth, low-dimensional manifold, the decay of the eigenvalues only depends on the dimensionality and smoothness properties of this manifold, see \citet[Theorem 4]{Altschuler_sinkhorn_2019}. In addition, for a given problem, the dimensionality $D$ is fixed, meaning that the dependence of the number of inducing points $M$ depends polylogarithmically on $N$. This growth is slower than any polynomial, i.e.~$(\log N)^D = o(N^\epsilon)$ for $\epsilon > 0$.

 We also note that \cref{cor:gaussian-multi-D} can easily be adapted using \cref{lem:average-eigenvalues} to show that if all of the $\bfx_n$ are drawn from any compactly supported distributions with continuous densities that are all bounded by some universal constant, the same asymptotic bound on the number of inducing points applies. This follows from noting that under these assumptions, $p_n(x)$, satisfies $p_n(x)<cq(x)$ where $q(x)$ is a Gaussian density for some $c>0$, so we can apply \cref{lem:average-eigenvalues} to bound the expectation of the sum of the matrix eigenvalues associated to $p_n$ in terms of the eigenvalues associated to $q$.

%%%%%%%%%%%%%%Widom's Theorem%%%%%%%%%%%%
\subsection{Compactly Supported Inputs and Stationary Kernels}\label{sec:stationary_kernels}

For most kernels and covariate distributions, solutions to the eigenfunction problem, $\mcK \phi = \lambda \phi$ cannot be found in closed form. In the case of stationary kernels defined on $\R^D$, that is kernels satisfying $k(x,x') = \kappa(x-x')$ for some $\kappa: \R^D \to \R$, the asymptotic properties of the eigenvalues are often understood \citep{widom_asymptotic_1963,widom1964asymptotic}.

Stationary, continuous kernels can be characterized through Bochner's theorem, which states that any such kernel is the Fourier transform of a positive measure, i.e.
\[
\kappa(x-x') =  \int_{\R^D} s(\omega)\exp( i \omega \cdot (x-x')) d \omega. 
\]
We will refer to $s(\omega)$ as the spectral density of $k$.\footnote{We assume $\kappa(x-x')$ decays sufficiently rapidly so that such a continuous spectral density exists.} The decay of the spectral density conveys information about how smooth the kernel function is. 

Widom's theorem \citep{widom_asymptotic_1963} relates the decay of the eigenvalues of $\mcK$ to the decay of $s$.
Widom's theorem applies to input distributions with compact support and stationary kernels with spectral density satisfying several regularit conditions (stated in \cref{app:examples}). \Citet{seeger2008information} give a corollary of Widom's theorem, which is sufficient in many instances to obtain bounds on the number of inducing points needed for \cref{thm:upper-bound-average-y,thm:upper-bound-fixed-y} to converge.
\begin{lem}[\citealp{seeger2008information}, Theorem 2]\label{lem:simplified-widom}
Let $k$ be an isotropic kernel (i.e. $\kappa(\alpha)= \kappa(\alpha')$ if $\|\alpha\| = \|\alpha'\|$). Suppose $k$ satisfies the criteria of Widom's theorem, the covariate distribution has density zero outside a ball of radius $T$ around the origin, and is bounded above by $\tau$, then
\[
\lambda_m \leq \tau (2\pi)^D s\left(\frac{2\Gamma(D/2+1)^{2/D}}{T}m^{1/D}\right)(1+o(1)).
\]
\end{lem}
\subsubsection{Mat\'ern kernels and compactly supported input distribution}
Mat\'ern kernels are widely applied to problems where the data generating process is believed to lead to non-smooth functions, and are known to satisfy the conditions of Widom's theorem \citep{seeger2008information}. These kernels are defined as \citep{gpml},
\begin{equation}
    k_{\text{Mat}}(x,x') = \frac{2^{1-\nu}}{\Gamma(\nu)}\left(\frac{\|x-x'\|_2}{\ell}\right)^\nu K_{\nu}\left(\frac{\|x-x'\|_2}{\ell}\right),
\end{equation}
where $K_{\nu}$ is a modified Bessel function. The spectral density of the Mat\'ern kernel is 
\[
s(\omega) = \frac{\ell^D\Gamma(\nu+D/2)}{\pi^{D/2}\Gamma(\nu)}\left(1 + (\ell\omega)^2\right)^{-(\nu+\frac{D}{2})}
\]
which is proportional to a Student's t-distribution with $2\nu+D$ degrees of freedom. This spectral density only decays polynomially, with the degree of the polynomial depending on $\nu$ and $D$. Here $\ell>0$ is the lengthscale and $\nu>0$ is a `smoothness' parameter often chosen as $\nu \in \{\frac{1}{2},\frac{3}{2},\frac{5}{2}\}$. The posterior mean is $\lfloor \nu \rfloor$-times differentiable, which relates to the slower decay of the spectral density.

 \Cref{lem:simplified-widom} tells us that for compactly supported covariates with bounded density and the Mat\'ern kernel with smoothness paramater $\nu$
\[
\lambda_m = \mcO(m^\frac{-2\nu-D}{D}).
\]
It follows that $\sum_{m=M+1}^D \lambda_m = \mcO(M^\frac{-2\nu}{D})$. From this, we can derive a result of the same form as \cref{cor:gaussian-multi-D} for Mat\'ern kernels and compactly supported input distributions. 
\begin{cor}
Suppose the conditions of \cref{thm:upper-bound-fixed-y} on $\bfX$ and $\bfy$ are satisfied for some $R>0$. Let $k$ be a Mat\'ern kernel with smoothness parameter $\nu$. Suppose that $N$ covariates are observed, each with an identical distribution with bounded density and compact support on $\R^D$. Fix any $\gamma\in (0,1]$. Then for $\nu >D/2$ if inducing points are initialized using an $\epsilon$-approximate $M$-DPP with and $\epsilon= \mcO(N^2/\gamma)$ there exists an $M=\mcO(N^\frac{2D}{2\nu-D}\gamma^{\frac{D}{2\nu-D}})$ such that 
\[
\Exp{}{\KL{Q}{P}} \leq \gamma.
\]
Under the same assumptions if inducing points are initialized using the RLS algorithm of \citet{Musco_ridge_leverage} with $\delta \in (0,1/32)$ there exists an $S=\mcO\left(N^\frac{2D}{2\nu+D}(\gamma\delta^2)^\frac{-D}{2\nu+D}\right)$ such that with probability at least $1-5\delta$,
\[
\KL{Q}{P} \leq \gamma,
\]
and $\bfM \leq S\log \frac{S}{\delta}$.
\end{cor}
\begin{rem}
If we consider $\gamma$ and $\delta$ as fixed constants (independent of $N$), this implies that for an initialization with $M$-DPP we can choose $M=\mcO(N^{\frac{2D}{2\nu-D}})$ inducing points leading to a computational cost of $\mcO(N^{\frac{2\nu+5D}{2\nu-D}}\log(N))$ while for approximate ridge leverage scores sampling $\mcO(N^\frac{2D}{2\nu+D} \log N)$ inducing points suffice leading to a cost at most $\mcO(N^\frac{2\nu+5D}{2\nu+D} (\log N)^2)$.
\end{rem}
The first part corollary follows from \cref{thm:upper-bound-fixed-y}, noting that we need to choose $M$ such that
\[
CN^2M\sum_{m=M+1}^\infty \lambda_m \leq C'N^2M^\frac{D-2\nu}{D} \leq \gamma\delta/2
\]
for some constants $C,C'$. The second part follows from similar considerations applied to \cref{thm:ridge-leverage-agnostic-y}.

These bounds on $M$ are vacuous (i.e.~are no smaller than $M=\mcO(N)$) for Mat\'ern kernels in high dimensional spaces or with low smoothness parameters. Additionally, the cost of sampling the $M$-DPP using \cref{alg:det_init_mcmc} makes this inference scheme less expensive than exact GP inference only when $\nu > 2D$. If we instead make the stronger assumptions required by \cref{thm:upper-bound-average-y}, we can choose $M=\mcO(N^\frac{D}{2\nu-D})$, which implies a computational complexity less than exact GP regression if $\nu > \frac{5}{4}D$. 

The bounds for the RLS initialization are generally sharper, and are non-vacuous for all $\nu > D/2$ with the weaker assumptions on $\bfy$. Additionally, the computational complexity of choosing inducing points using the RLS algorithm is the same as the cost of inference up to logarithmic factors, so that for $\nu > D/2$ the cost of sparse inference with the RLS initialization is (asymptotically) smaller than the cubic cost of exact GP regression.

%%%%%%%%%%%%%%%%%%%%%%% Lower Bounds%%%%%%%%%%%%%%%%%%%%%%%%%%%%%%%%%%%%%%%%%%%%%%%%%%%%%%
\section{Lower Bounds on the Number of Inducing Points Needed}\label{sec:lower-bounds}

In \cref{sec:regression-rates,sec:examples}, we showed that for many problems the number of inducing points can grow sub-linearly with the number of data points, while maintaining a small KL-divergence between the approximate and exact posteriors. In this section we consider the inverse question, i.e.~how many inducing points are necessary to avoid having the KL-divergence grow as the amount of data increases? In this section, we prove a-priori lower bounds on the KL-divergence under similar assumptions to those used in proving the upper bounds in \cref{sec:regression-rates}. 

Naively, it appears that the lower bound in \cref{lem:kl-averaged} gives us a starting place for a lower bound on the KL-divergence. From this bound, 
\begin{align}
    \conditionalexp{\KL{Q}{P}}{\bfZ,\bfX} \geq \frac{\bft}{2\noisevariance}.
\end{align}
While lower bounding this quantity can be done using the approach taken in this section, it is not the most interesting quantity to study, as we average over $\bfy$ \emph{conditioned} on $\bfX$ and $\bfZ$. This would not give a valid lower bound if the locations of the inducing points depend on $\bfy$, as illustrated in \cref{fig:kl_trace_comparison}. This approach would establish a lower bound for initialization schemes considered in the previous sections, as well as any initialization scheme that does not take the observed $y$ into account, but not the common practice of performing gradient ascent on the ELBO with respect to inducing inputs. 

In this section, we establish a lower bound on the number of inducing variables needed for the KL-divergence not to become large, which is valid regardless of the method for selecting inducing variables or the distribution of $\bfy$. These bounds assume that the covariates are independent and identically distributed (in contrast to the upper bounds, which do not require independence and require a slightly weaker condition than identical marginals). The independence assumption is necessary in order to lower bound the eigenvalues of the covariance matrix. For example, if all of the covariates were identically distributed and equal, the covariance matrix would be rank-1 and so a single inducing point could be used regardless of the size of the data set.

The proof of the lower bounds proceeds in two parts:
\begin{enumerate}
    \item First, we derive a lower bound on $\KL{Q}{P}$ that holds for any $\datay$ and $Z$, but depends on the eigenvalues of $\Kff$.
    \item Second, we use a result on the concentration of eigenvalues of the kernel matrix to those of the corresponding operator due to \citet{braun2006accurate} to derive a lower bound that holds with fixed probability under the assumption that the covariates are independent and identically distributed.
\end{enumerate}

\begin{table}
\begin{center}
\begin{small}
\begin{sc}
\begin{tabular}{lcc}
\toprule
Kernel & Input Distribution & M  \\
\midrule
SE-Kernel &  Gaussian &  $\Omega((\log N)^D)$\\
Mat\'{e}rn $\nu$  & Uniform &  $\Omega\left(N^{\frac{2\nu D}{(2\nu+5D)(2\nu+D)}-\epsilon}\right)$\\
\bottomrule
\end{tabular}
\end{sc}
\end{small}
\end{center}
\caption{The lower bounds we establish on the number of features needed so that the KL-divergence does not increase as a function of $N$. Here $\epsilon$ is a positive constant that can be chosen arbitrarily close to $0$. While the bound for the the SE-kernel matches our upper bounds up to terms that are constant in $N$, we expect the lower bound for the Mat\'ern kernel can be raised significantly.}\label{table:numfeats_lower}
\end{table}

In the case of SE-kernel and Gaussian covariates, we establish a lower bound with the same dependence on $N$ as our upper bounds, that is we need $M = \Omega((\log N)^D)$. In the case of Mat\'ern kernels with uniform covariates and $\nu>1$, we establish a lower bound that increases as a power of $N$. However, there is a large gap between our upper and lower bounds for Mat\'ern kernels, indicating room for improvement. These results are summarized in \cref{table:numfeats_lower}. While our results are stated in terms of inducing points, they hold for more general inducing variables.

\subsection{A Lower Bound}
In this section, we derive a lower bound on the KL-divergence that holds for any $y$ and $Z$ and depends on $\bfX$. 
\begin{restatable}{lem}{firstlowerbound}\label{lem:min-min-lower-bound}
Given a kernel $k$, likelihood model with variance $\noisevariance$ and random covariates $\bfX$. Then,
\[
 \min_{Z \in \mcX^M}\min_{y \in \R^N} \KL{Q}{P} \geq  \frac{1}{2}\sum_{m=M+1}^N \frac{\tilde{\bm{\lambda}}_m }{\noisevariance} - \log \left(1+ \frac{\tilde{\bm{\lambda}}_m}{\noisevariance}\right)
\]
where $\tilde{\bm{\lambda}}_m$ denotes the $m^{th}$ largest eigenvalue of the matrix $\RKff$ determined by the covariates and kernel.
\end{restatable}

The proof (\cref{app:lower-bounds}) follows from noting that for any $X,y,Z$ we have $\KL{Q}{P} = \log p(\datay) - \ELBO(\datay,Z) \geq \log p(0) - \ELBO(0,Z)$, where we have defined $\ELBO(\datay,Z)$ to be the evidence bound resulting from the triple $X,y,Z$ (and suppressed dependence on $X$). The bound follows from writing both the trace and log determinant in terms of eigenvalues and using that $\RKff \succ \RQff$, so that the $m^{th}$ largest eigenvalue of $\RKff$ is greater than the $m^{th}$ largest eigenvalue of $\RQff$ for all $1\leq m \leq N$.

In order to establish lower bounds on the number of inducing variables needed to ensure the KL-divergence does not grow as a function of $N$, it suffices to analyze the behavior of the lower bound in \cref{lem:min-min-lower-bound} for random covariates as a function of both $M$ and $N$. 

\subsection{Structure of the Argument}

In order to derive a lower bound on the KL-divergence \cref{lem:min-min-lower-bound}, we can consider just the largest term in the sum appearing in \cref{lem:min-min-lower-bound}, as all the terms are non-negative. For $a>3$, $\log (1+a) \leq a/2$. Therefore, if $\rkffeigenvalueM/\noisevariance > 3$, we have $\KL{Q}{P} \geq \frac{\rkffeigenvalueM}{4\noisevariance}$.

Under the supposition that $\rkffeigenvalueM/\noisevariance > 3$, we can apply the triangle inequality to \cref{lem:min-min-lower-bound} to give,

\[
\KL{Q}{P} \geq \frac{N(\operatoreigenvalueM-|\operatoreigenvalueM-\frac{1}{N}\rkffeigenvalueM|)}{4\noisevariance}
\]

Therefore, for any $M=M(N)$ such that:
\begin{enumerate}
    \item We can show a relative error bound on the approximation of matrix eigenvalues with operator eigenvalues of the form $\frac{|\operatoreigenvalueM-\frac{1}{N}\rkffeigenvalueM|}{\operatoreigenvalueM}<1-\gamma_N$ for some $\gamma_N \in (0,1)$,
    \item $N\gamma_N\operatoreigenvalueM$ tends to infinity as $N$ tends to infinity,
\end{enumerate}
it must be the case that the KL-divergence tends to infinity as a function of $N$ (at a rate $\Omega(N\gamma_N\lambda_{M+1})$).

\subsection{Concentration of Eigenvalues}

In order to complete the argument in the previous section, we need a more fine-grained understanding of the behavior of eigenvalues of $\RKff$ than given in \cref{lem:average-eigenvalues}. For this, we rely on the following result:

\begin{lem}[\citealp{braun2006accurate}, Theorem 4]\label{lem:braun-boundedkernel}
Let $k$ be a continuous kernel with $k(x,x)\leq v$ for all $x \in \R^D$. Fix $\delta \in (0,1)$. Suppose $\bfx_1, \cdots, \bfx_N$ are realizations of i.i.d.~random variables sampled according to some measure on $\R^D$ with density $p(x)$. Then for all $m$ and any $1\leq r \leq N$, with probability at least $1-\delta$,
\begin{align}
    |\lambda_{m} - \frac{1}{N}\rkffeigenvalue| &\leq \lambda_mr\sqrt{\frac{r(r+1)v}{\lambda_rN\delta}} + \sum_{s=r}^\infty \lambda_s + \sqrt{\frac{2v\sum_{s=r+1}^\infty \lambda_s}{N\delta}} \label{eqn:braun-bound-explicit}\\ 
    &= \mcO\left(\lambda_mr^2\lambda_{r}^{-1/2}N^{-1/2}\delta^{-1/2}+\sum_{s=r}^\infty \lambda_s + \sqrt{\frac{\sum_{s=r+1}^\infty \lambda_s}{N\delta}}\right) \nonumber.
\end{align}
where $\operatoreigenvalue$ is the $m^{th}$ eigenvalue of the integral operator $\mcK: (\mcK g)(x') =\int g(x)k(x,x')p(x)dx$ and $\rkffeigenvalue$ is the $m^{th}$ eigenvalue of $\RKff$.
\end{lem}

For the remainder of this section, we consider specific cases of kernel and input distributions for which we know properties of the spectrum, and derive lower bounds on the number of features needed so that the KL-divergence is not an increasing function of $N$.
%%%%%%%%%%%%%%%%%%%%%%%%%%%%%%%%%%%%%%%%%%%%%%%%%%%%%%%%%%%%%%%%%%%%%%%%%%%%%%%%%%%%%%%%%%%%%%%%%%%%%
\subsection{Squared Exponential Kernel and Gaussian Covariates}
We begin with the one-dimensional SE kernel and Gaussian covariates. The multivariate case follows a similar, though more involved argument and will be discussed in \cref{sec:multidim-gauss-se-lower}. Recall, if the covariates have variance $\beta^2$, then for any $r\in \mathbb{N}$
\[
\lambda_r = v\sqrt{\frac{2a}{A}}B^{r-1} \text{ \quad and \quad} \sum_{s=r}^\infty\lambda_s =  \frac{\lambda_r}{1-B}, 
\]

\noindent where $a=(4\beta^2)^{-1}, b = (2\ell^2)^{-1}, A = a+ b + \sqrt{a^2+4ab}$ and $B=b/A$. For some $\eta \in (0,1)$, choose $r = 1+\lceil\log_{B} (1-B)\sqrt{A/(2av^2)}N^{-\eta} \rceil$, so that $\sum_{s=r}^\infty \lambda_r \leq N^{-\eta}$ and $\lambda_r \geq B(1-B)N^{-\eta}$. Hence \cref{eqn:braun-bound-explicit} implies that for all $m$ with probability at least $1-\delta$,
\begin{align}
    \frac{|\lambda_{m} - \frac{1}{N}\rkffeigenvalue|}{\lambda_m} \leq  r\sqrt{\frac{r(r+1)v}{B(1-B)N^{1-\eta}\delta}} + \lambda_m^{-1}N^{-\eta} + \sqrt{\frac{2v}{\lambda_m^2N^{1+\eta}\delta}}.
\end{align}
For any fixed $\delta \in (0,1)$ the first term on the right hand side tends to zero with $N$ since $\eta <1$. For any $M \leq\log_{B} (\sqrt{A/(2av^2)}N^{-\eta}\sqrt{\delta})$, 
\begin{align}
    \frac{|\lambda_{M+1} - \frac{1}{N}\rkffeigenvalueM|}{\lambda_{M+1}} \leq  r\sqrt{\frac{r(r+1)v}{B(1-B)N^{1-\eta}\delta}} + \frac{\sqrt{\delta}}{2}+ \sqrt{\frac{v}{2N^{1-\eta}}}.
\end{align}
The second term is less than $1/2$ and the last term tends to $0$ for large $N$. We conclude that for any such $M$, the KL-divergence is bounded below by $c_N\frac{N\lambda_{M+1}}{8\noisevariance}=\Omega(N^{1-\eta})$, where $\lim_{N\to\infty} c_N=1$. 
\begin{rem}
For univariate Gaussian kernels, if we choose $\eta= .01$ in the above argument, we get that for any $M \leq \frac{1}{\log(1/B)}\log(N^{.01}\sqrt{\frac{2av^2}{\delta A}})=\Omega(\log N)$, the KL-divergence is $\Omega(N^{.99})$, and will therefore be large as $N$ increases. We therefore need $M$ to grow faster than this to avoid this if we want the KL-divergence to be small for large $N$.
\end{rem}

\subsection{The Isotropic SE-kernel and Multidimensional Gaussian Covariates }\label{sec:multidim-gauss-se-lower}
In order to obtain lower bounds in the multivariate case, we first obtain a lower bound on the individual eigenvalues of the operator $\mcK$.
\begin{restatable}[]{prop}{lbgauss}\label{prop:lower-bound-eigenvals-gauss} 
Suppose $k$ is an isotropic SE-kernel in $D$ dimensions with lengthscale $\ell$ and variance $v$. Suppose the training covariates are independently identically distributed according to an isotropic Gaussian measure, $\mu$, on $\R^D$ with covariance matrix $\beta^2\bfI$. For any $r \in \N$, we have 
\[
\lambda_r \geq \left(\frac{2a}{A}\right)^{D/2}B^{Dr^{1/D}}.
\]
where $\lambda_r$ denotes the $r^{th}$ largest eigenvalue of the operator $\mcK: L^2(\R^D,\mu) \to L^2(\R^D,\mu)$ defined by $(\mcK g)(x') = \int g(x)k(x,x')p(x)dx$ with $p(x)$ the density of the multivariate Gaussian at $x$.
\end{restatable}
The proof (\cref{app:lower-bounds}) relies on a counting argument and standard bounds on binomial coefficients.  We can now combine \cref{lem:braun-boundedkernel,prop:tail-eigenvalues-gaussian,prop:lower-bound-eigenvals-gauss} in order to bound the multivariate SE-kernel with Gaussian inputs. 

\begin{restatable}[]{prop}{lbMgauss}\label{prop:lower-bound-features-gauss} 
Let $k$ be an isotropic SE-kernel. Suppose $N$ covariates are sampled independent and identically from an isotropic Gaussian density with variance $\beta^2$ along each dimension. Define $M(N)$ to be any function of $N$ such that $\lim_{N \to \infty} M(N)/(\log N)^D = 0$; i.e. $M(N)=o((\log N)^D)$. Suppose inference is performed using any set of inducing inputs, $Z$ such that $|Z|=M(N)$. Then for any $\datay \in \R^N$, for any $\epsilon>0$ and for any $\delta \in (0,1)$, with probability at least $1-\delta$, $\KL{Q}{P} = \Omega(N^{1-\epsilon})$.
\end{restatable}
\begin{rem}
To illustrate the meaning of this lower bound, consider the case when $M(N) = (\log N)^{D-1}$. Then for $N$ sufficiently large, we have $\KL{Q}{P} > N^{.99}$ implying for large $N$ the KL-divergence must be large. On the other hand from our upper bounds in \cref{sec:examples}, we know that if we fix any positive constant $\gamma$ there exists a constant $C$ (that does not depend on $N$) such that if inference is performed with $M(N) = c(\log N)^{D}$ inducing inputs placed according to an approximate $M$-DPP, the KL-divergence is less than $\gamma$ in expectation.
\end{rem}
The proof follows from \cref{lem:min-min-lower-bound} by choosing $r$ appropriately in \cref{lem:braun-boundedkernel} to bound the empirical eigenvalues and using \Cref{prop:tail-eigenvalues-gaussian,prop:lower-bound-eigenvals-gauss} to bound eigenvalues in the appropriate directions to control the error term. Details are given in \cref{app:lower-bounds}.
%%%%%%%%%%%%%%%%%%%%%%%%%%%%%%%%%MATERN%%%%%%%%%%%%%%%%%%%%%%%%%%%%%%%%%%%%%%%%%%%%%%%%%%%%%%%%%%
\subsection{Lower Bounds for Kernels with Polynomial Decay}
As discussed in \cref{sec:examples} some popular choices of kernels lead to eigenvalues that decay polynomially instead of exponentially. For example, \citet[Theorem 2.1]{widom_asymptotic_1963} implies that the eigenvalues of the operator associated to the Mat\'ern kernel with smoothness parameter $\nu$ and covariates uniformly distributed in the unit cube has eigenvalues satisfying $C_1 m^\frac{-2\nu+D}{D}\leq \lambda_m\leq C_2 m^\frac{-2\nu+D}{D}$ for some constant $C_1$ and $C_2$ independent of $m$ i.e. $\lambda_m= \Theta(m^\frac{-2\nu+D}{D})$.\footnote{See \citet{seeger2008information} for more details on the derivation of this from Widom's Theorem.}

\begin{restatable}{prop}{lbpolykernels}
Let $k$ be a continuous kernel, and $\mu$ a measure on $\R^D$ with density $p$ such that the associated operator $\mcK$ has eigenvalue satisfying $C_1 m^{-\eta}\leq \lambda_m\leq C_2 m^{-\eta}$ for all $m \geq 1$, some $\eta>1$ and constants $C_1,C_2>0$. Suppose inference is performed using any set of inducing inputs $Z$ such that $|Z|=M(N)$ with $M(N)$ any function such that $\lim_{N\to \infty}\frac{M(N)}{N^\zeta}=c$ (i.e.~$M=O(N^\zeta))$ for some $c < \infty$ and  $\zeta\in (0,\frac{\eta-1}{\eta(4+\eta)})$. Then for any $\delta \in (0,1)$ with probability at least $1-\delta$, $\KL{Q}{P}=\Omega(N^{1-\eta\zeta})$.
\end{restatable}
\begin{proof}
We have $\lambda_r = \Theta(r^{-\eta})$ so $\sum_{s=r}^\infty \lambda_s = \Theta(r^{1-\eta})$. Choose $r=N^\gamma$ for some $\gamma \in (0,1)$ and $M+1=N^{\zeta}$. In this case the error term in \cref{lem:braun-boundedkernel} becomes:
\[
    \frac{|\operatoreigenvalueM - \frac{1}{N}\rkffeigenvalueM|}{\operatoreigenvalueM} = \mcO\left(\delta^{-1/2}(N^{2\gamma+\gamma\eta/2-1/2}+N^{\zeta\eta+\gamma(1-\eta)} +N^{\zeta\eta+\gamma(1-\eta)/2-1/2})\right).
\]
Following the earlier proof sketch, we must show the RHS tends to a value less than $1$. To ensure the first of the three summands is small, we choose $\gamma\in (0,\frac{1}{4+\eta})$. Given this choice, for large $N$ the third summand in the error term is always smaller than the second, so that this entire term is $o(1)$ given the supposition that $\zeta\leq \gamma\frac{\eta-1}{\eta}$. We conclude that if $M=N^{-\zeta}$ for $\zeta\in (0,\gamma\frac{\eta-1}{\eta})$ with probability $1-\delta$, the lower bound in \cref{lem:min-min-lower-bound} is at least $N\lambda_{M+1} = \Omega(N^{1-\zeta\eta})$.
\end{proof}
In the case of $D$-dimensional Mat\'ern kernels and a uniform covariate distribution ($\eta=\frac{2\nu+D}{D}$),  by choosing $\zeta$ as large as possible, this means that for an arbitrary $\epsilon>0$, the KL-divergence is lower bounded by an increasing function of  $N$ if fewer than $\Omega\left(N^{\frac{2\nu D}{(2\nu+5D)(2\nu+D)}-\epsilon}\right)$ inducing variables are used. This lower bound on the number of inducing variables becomes vacuous (i.e.~the exponent tends to $0$) as $\eta \to 1$ from above, meaning it is not useful when applied to many kernels that we expect would be very difficult to approximate. There is a large gap between the upper and lower bound, particularly when $\eta$ is near $1$ (i.e. for non-smooth kernels). The gap  between the bounds is in part introduced by needing to choose $M$ so that the error term from \cref{lem:braun-boundedkernel} remains lower order. If we heuristically allow ourselves to replace matrix eigenvalues with the corresponding scaled operator eigenvalues and neglect the error term, we obtain a lower bound of $\Omega(N^{\frac{1}{\eta}})$, bringing the lower bound more closely in line with the upper bound of $\mcO(N^{\frac{1}{\eta -1}})$. The remaining gap between these bounds is essentially due to only bounding a single eigenvalue in the lower bound, while bounding the sum of eigenvalues in the upper bound. Improving the analysis to close the gap between the upper and lower bounds is important for better understanding the efficacy of sparse methods with non-smooth kernels.

%%%%%%%%%%%%%%%%%%%%%%%%%%%%% Experiments and Practical Details%%%%%%%%%%%%%%%%%%%%%%%%
\section{Practical Considerations}\label{sec:experiments}

Up to this point, we proved statements about the \emph{asymptotic} scaling properties of variational sparse inference. Our results indicated which models could be well-approximated with relatively few inducing points for sufficiently large data sets. In this section, we investigate the limitations and practical implications of our results to real situations with finite amounts of data. We consider the applicability of our results to practical implementations, and perform empirical analyzes on how marginal likelihood bounds converge. Additionally, our proof suggests a specific procedure for choosing inducing points that differs from methods that are currently commonly applied. We empirically investigate this procedure, and provide recommendations on how to initialize inducing points.

\subsection{Finite Precision in Practical Implementations}
Any practical implementation of a Gaussian process method will be influenced by the finite precision with which floating-point numbers are represented in a computer. These issues are not explicitly addressed in our mathematical analysis, which assume calculations are in exact arithmetic. Here, we briefly discuss the effects of this finite precision on 1) the implementation, 2) the precision to which we can expect convergence in practice compared to our analysis, and 3) the way that this is quantified by marginal likelihood bounds.

\subsubsection{Ill-conditioning \& Cholesky Decomposition Failure}
Finding various quantities for Gaussian process regression requires computing log determinants and matrix inverses. When the smallest and largest eigenvalues of the kernel matrix are many orders of magnitude apart, these computations become \emph{ill-conditioned}, meaning that small changes on the input can lead to large changes to the output. For example, tiny changes in the elements of the vector $\vf_X$ can lead to huge variations in the vector $\Kff\inv\vf_X$ when $\Kff$ has an eigenvalue close to zero \citep[see][\S 6.2 for a visual illustration]{deisenroth2019distill}. This typically occurs when considering many highly-correlated inputs to the GP (e.g~ clusters of nearby points with similar input values). These points have a high probability of having very similar function outputs under the prior. This ill-conditioning arises naturally in GPs when considering e.g.~evaluating the prior density on function values: small differences in the function values result in huge changes to the value of the probability density. If the sensitivity of the calculations becomes too large, then the finite precision with which numbers are represented can lead to considerable error.

In the variational methods we consider, determinants and inverses are found based on the Cholesky decomposition of the kernel matrix: $\Kff + \sigma^2\bfI$ for exact implementations, and $\Kuu$ for sparse approximations.\footnote{Conjugate gradient and Lanczos methods also give exact answers when they are run for sufficient iterations, and have been successfully applied in practice \citep{gibbs1997efficient,davies2016thesis,gardner2018gpytorch}.} When faced with a problem that is too ill-conditioned, most Cholesky implementations terminate with an exception. This can be seen as desirable from the point of view that a successful run usually indicates an accurate result.

Even in cases when the data set can be well-described by a GP model with hyperparameters that lead to reasonably well-conditioned matrices, conditioning problems frequently arise during training, when the log marginal likelihood or ELBO is values for other candidate hyperparameter values. For example, for stationary kernels, large lengthscales contribute to conditioning problems, as they increase the correlation between distant points. Hyperparameters are typically found by (approximately) maximizing the log marginal likelihood (Eqs.~\ref{eqn:full-marginal-likelihood} and \ref{eqn:elbo}). Since these objective and their derivatives can be evaluated in closed-form, fast-converging quasi-Newton methods such as (L-)BFGS are commonly used. These methods often propose large steps, which lead to the evaluation of hyperparameter settings where the Cholesky decomposition raises an exception. Even though these hyperparameter settings are often of poor quality, (L-)BFGS still requires an evaluation of the objective function to continue the search. The Cholesky errors must therefore be avoided to successfully complete the entire optimization procedure.

\subsubsection{Improving Matrix Conditioning}
Increasing the smallest eigenvalue of the kernel matrix improves the conditioning. In exact implementations this can be done by increasing the likelihood noise variance, as we need to decompose $\Kff + \sigma^2\bfI$ which has eigenvalues that are lower-bounded by $\sigma^2$\footnote{This can be done by reparameterizing the noise to have a lower bound.}. On the other hand, the sparse variational approximation requires inverting $\Kuu$ \emph{without any noise}. However, it is important to note that the conditioning of $\Kuu$ is better than $\Kff$ for two reasons. Firstly, it is a smaller matrix, and often issues of conditioning are less severe for smaller matrices. Secondly, if inducing points are selected using a method that introduces negative correlations clusters of highly-correlated points are unlikely to appear in $\Kuu$. Nevertheless, it is still possible for the Cholesky decomposition to fail, particularly when trying different hyperparameter settings when maximizing the ELBO.

To improve robustness in the sparse approximation, a small diagonal ``jitter'' matrix $\epsilon\bfI$, with $\epsilon$ commonly around $10^{-6}$, is added to $\Kuu$, introducing a lower bound the on its eigenvalues. This change is often enough to avoid decomposition errors during optimization. While this modification changes the problem that is solved, the effect is typically small. Some software packages \citep[e.g.~][]{gpy2014} increase jitter adaptively by catching exceptions inside the optimization loop to only introduce bias where it is necessary.

\subsubsection{Quantifying the Effect of Jitter}
Adding jitter to the covariance matrix $\Kuu$ corresponds to defining the inducing variables as noisy observations of the GP. While this still produces a valid approximation to the posterior and ELBO \citep[]{titsias2009variationaltech}, the approximation obtained is generally of marginally lower quality and there is a small amount of corresponding slack in the ELBO \citep[][Theorem 4]{matthews_scalable_2016}. We summarize the effect on the ELBO and upper bound $\mcU_2$ in the following proposition. Define $\Qff(\epsilon) \coloneqq \Kuf\transpose(\Kuu+\epsilon\bfI)^{-1}\Kuf$, so that $\Qff(0) = \Qff$.
\newcommand{\Qjitter}{\Qff(\epsilon)}
\begin{restatable}[]{prop}{jitter}\label{prop:jitter-bound}
Let $\mcL_\epsilon$ denote the evidence lower bound computed with jitter $\epsilon \geq 0$ added to $\Kuu$, that is
\[ 
\mcL_\epsilon = -\frac{1}{2}\log\det(\Qjitter+\noisevariance\bfI) -\frac{1}{2} \datay\transpose(\Qjitter+\noisevariance\bfI)^{-1}\datay-\frac{N}{2}\log 2 \pi-\frac{1}{2\noisevariance}\Tr(\Kff-\Qjitter).
\]
 Then $\mcL_\epsilon$ is monotonically decreasing in $\epsilon$. Similarly if $\mcU_\epsilon$ denotes the upper bound \cref{eqn:upper2} computed with added jitter to $\Kuu$, that is
\[
\mcU_\epsilon \coloneqq -\frac{1}{2}\log\det(\Qjitter+\noisevariance\bfI) -\frac{1}{2} y\transpose(\Qjitter+\Tr(\Kff-\Qjitter)\bfI+\noisevariance\bfI)^{-1}y-\frac{N}{2}\log 2 \pi.
\]
Then $\mcU_\epsilon$ is monotonically increasing in $\epsilon$. In particular, adding jitter can only make the upper bound on the log marginal likelihood larger and the ELBO smaller. 
\end{restatable}
The proof (\cref{app:experiments}) is a consequence of $\Qff(\epsilon) + \noisevariance \bfI \succ \Qff(\epsilon') + \noisevariance \bfI$ for $\epsilon'>\epsilon\geq 0$.
\Cref{prop:jitter-bound} shows that even with jitter the upper and lower bounds are still valid. However, they are not exactly equal, even when $M=N$, due to the additional gap caused by the jitter. From a practical point of view, the impact is typically very small, with a gap being introduced on the order of a few nats. 

To summarize, we saw 1) that jitter was needed to stabilize the computation of the hyperparameter objective functions using standard implementations of Cholesky decomposition, and 2) that jitter and finite floating-point precision prevented the approximate posterior and bounds from converging to their exact values. Notably, we can quantify the effect of the finite precision calculations \emph{using the same bounds} as what is used to determine the effect of using an approximate inducing point posterior (\cref{prop:jitter-bound}). The variational bounds we analyzed in this work therefore provide a unified way of measuring the effect of both exact arithmetic approximate posteriors and the impact of finite precision arithmetic on the quality of the approximation.

\subsection{Placement of Inducing Inputs}
When training a sparse Gaussian process regression \citep{titsias_variational_2009} model, we need to select the kernel hyperparameters as well as the inducing inputs, with the hyperparameters determining the generalization characteristics of the model, and the inducing inputs the quality of the sparse approximation. In the time since \citet{snelson_sparse_2006} and \citet{titsias_variational_2009} introduced joint objective functions for all parameters, it has become commonplace to find the final set of inducing variables by optimizing the objective function together with the hyperparameters. Because this makes the inducing input initialization procedure less critical for final performance, less attention has been placed on it in recent years than in e.g.~the kernel ridge regression literature.\footnote{Kernel ridge regression lacks a joint objective function for the approximation and hyperparameters. Hyperparameters are commonly selected through cross-validation.} However, the number of optimization parameters added by the inducing inputs is often large, and convergence can be slow, which makes the optimization cumbersome.

Our results suggest that gradient-based optimization of the ELBO is not necessary for setting the inducing variables. Selecting enough of inducing points is sufficient for obtaining an arbitrarily good approximation. For the common squared exponential kernel if hyperparameters are fixed, our upper and lower bounds imply that optimization of inducing points leads to at most a constant factor fewer inducing points than a good initialization. In this section, we investigate the performance of various inducing point selection methods in practice. We consider commonly used methods (uniform sampling, K-means, and gradient-based optimization), and methods that our proofs are based on ($M$-DPP and RLS). In addition, we propose using the intialization used for approximately sampling the $M$-DPP (\cref{alg:det_init_mcmc}) as an inducing point selection method. While our theoretical results do not prove anything for this method, it avoids the additional cost and complexity of running a Markov chain, while still being leading to negative correlations between inducing point locations. We refer to this method as \emph{greedy variance selection}, since it greedily selects the next inducing point based on which has the highest marginal variance in the conditioned prior $p(\vf\given\vu)$, i.e.~$\argmax \diag[\Kff - \Kuf\Kuu\inv\Kuf]$.\footnote{An equivalent approach, derived through different motivations, has been previously used for approximating the kernel matrix in SVMs \citep{fine2001efficient} and applied to sparse GP approximations \citep{foster2009stable}.}

We set the free parameters for each of the methods as follows. For K-means, we run the Scipy implementation of K-means++ with $M$ centres. Gradient-based optimization is initialized using greedy variance selection. This choice was made since it was found to perform better than uniform selection, and our goal is to quantify how much can be gained by doing gradient-based optimization, and whether it is worth the cost. We ran $10^4$ steps of L-BFGS, at which point any improvement was negligible compared to adding more inducing variables. Approximate $M$-DPP sampling was done following \cref{alg:det_init_mcmc}, using $10^4$ iterations of MCMC. For RLS, we use an adaptation of the public implementation of \citet[Algorithm 3]{Musco_ridge_leverage}, which omits many of the constants derived in the proofs, and therefore loses theoretical guarantees.\footnote{Their implementation is available at: \url{https://github.com/cnmusco/recursive-nystrom}.} We additionally modify the algorithm to ensure that it selects exactly $M$ inducing points.

We consider 3 data sets from the UCI repository that are commonly used in benchmarking regression algorithms, ``Naval" ($N_{train}=10740,N_{test}=1194,D=14$) , ``Elevators" ($N_{train}=14939,N_{test}=1660,D=18$) and ``Energy" ($N_{train}=691,N_{test}=77,D=8$). These data sets were chosen as near-exact sparse approximations could be found, so convergence could be illustrated.\footnote{Not all data sets exhibit this property. For instance, the ``kin40k'' data set still isn't near convergence when $M = \frac{N}{2}$ due to very short optimal lengthscales. A step functions being present would cause this, and would indicate that squared exponential kernels are inappropriate.} Naval is the result of a physical simulation and the observations are essentially noiseless. To make statistical estimation more difficult, we add independent Gaussian noise with standard deviation $0.0068$ to each observation. For all experiments, we use a squared exponential kernel with automatic relevance determination (ARD), i.e.~a separate lengthscale per input dimension. 

\subsubsection{Fixed Hyperparameters}
We first consider regression with fixed hyperparameters to illustrate convergence in a situation that is directly comparable to our theoretical results. We investigate which of the inducing point selection methods recovers the exact model with the fewest inducing points. The hyperparameters are set to the optimal values for an exact GP model, or for ``Naval" a sparse GP with 1000 inducing points. We find the hyperparameters by maximizing the exact GP log marginal likelihood using L-BFGS. This setting is for illustrative purposes only, as computing exact log marginal likelihood is not feasible in practical situations where sparse methods are of actual interest. In the next section we consider hyperparameters that are learned using the ELBO (Eq.~\ref{eqn:elbo}).

\begin{figure}[ht]
    \centering
    \includegraphics[width=.35\textwidth,trim=.5cm  .5cm .5cm .8cm, clip]{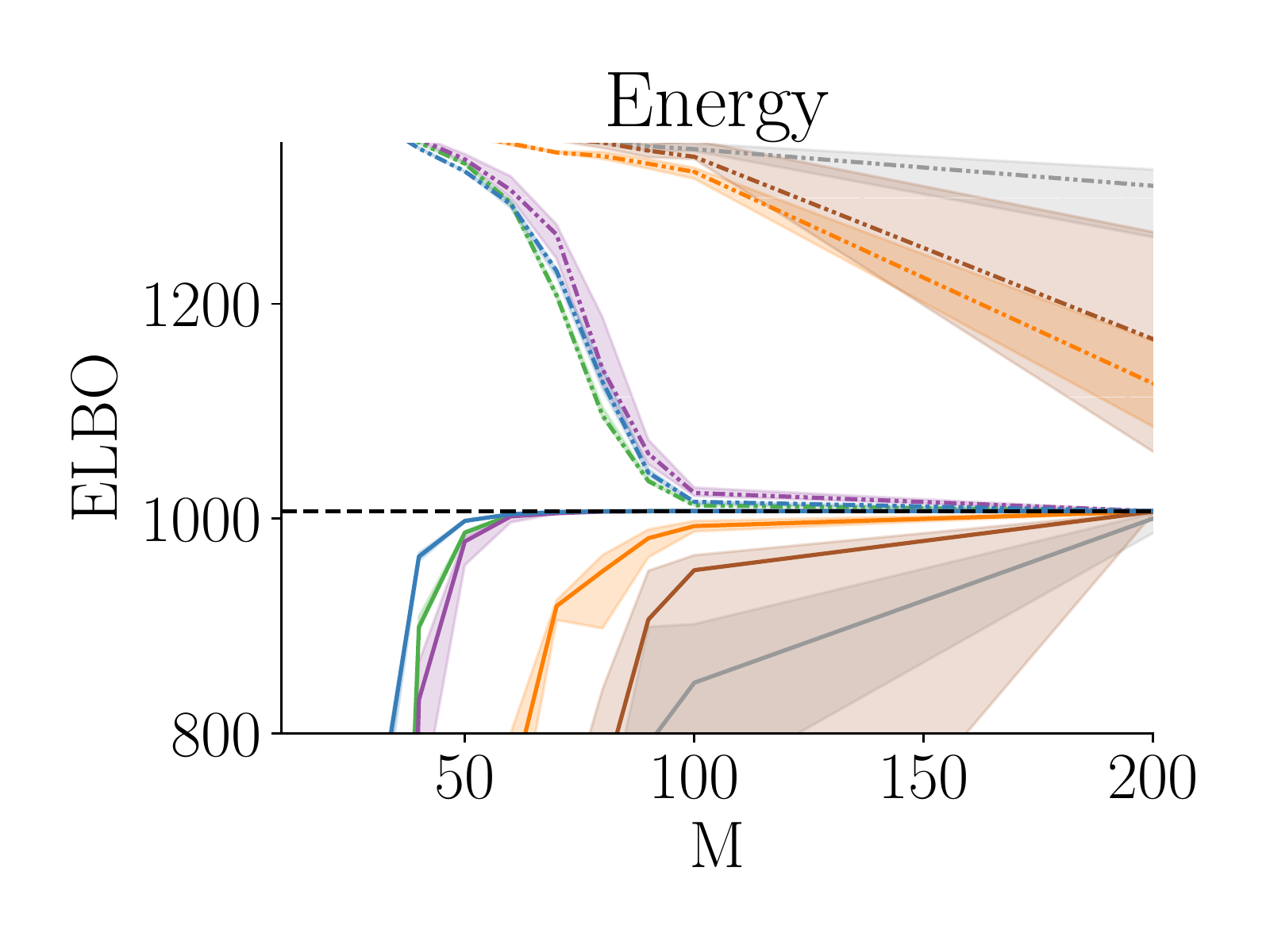}
    \includegraphics[width=.315\textwidth,trim=1.7cm  .5cm .5cm .8cm, clip]{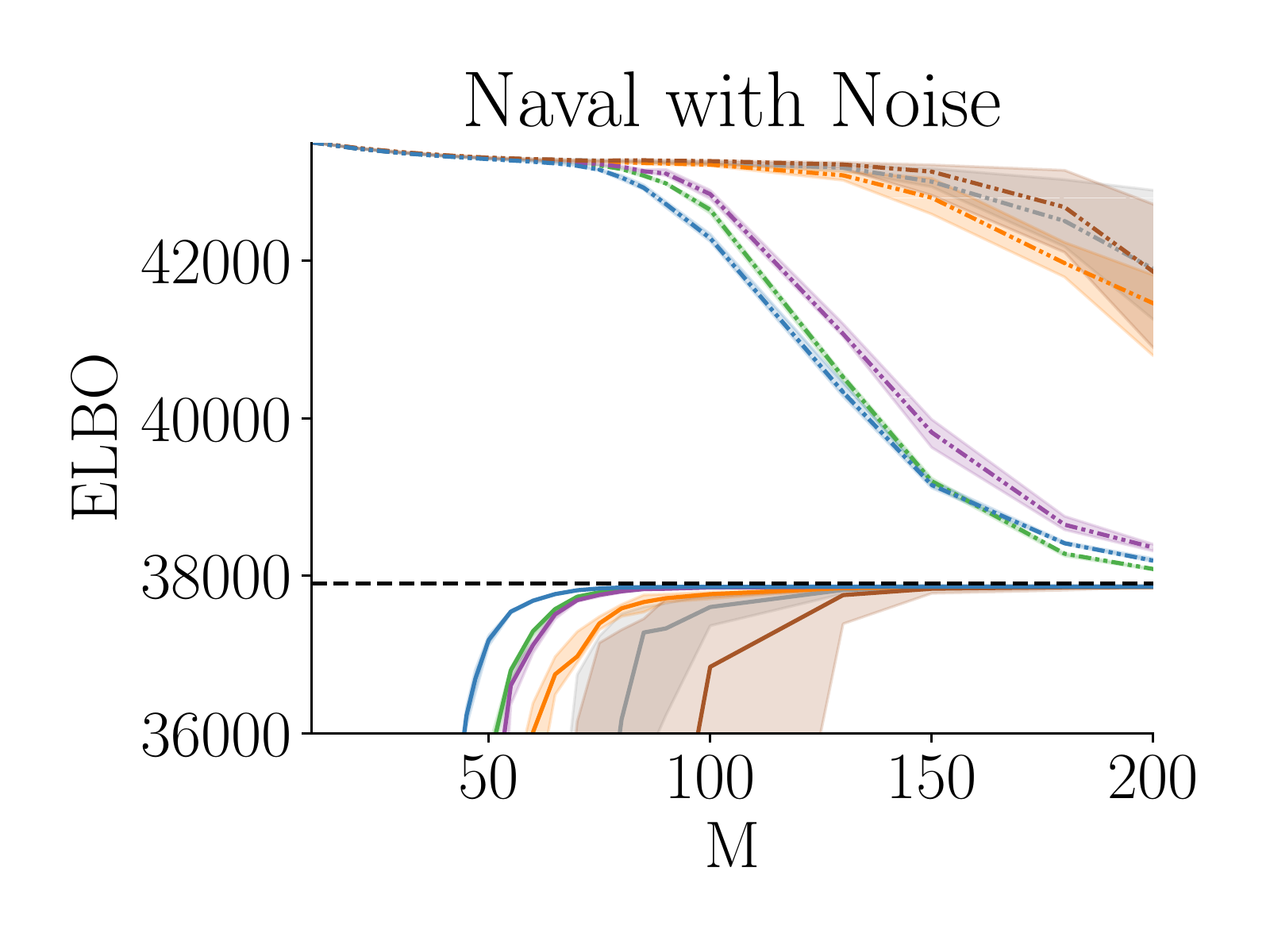}
    \includegraphics[width=.315\textwidth, trim=1.7cm  .5cm .5cm .8cm, clip]{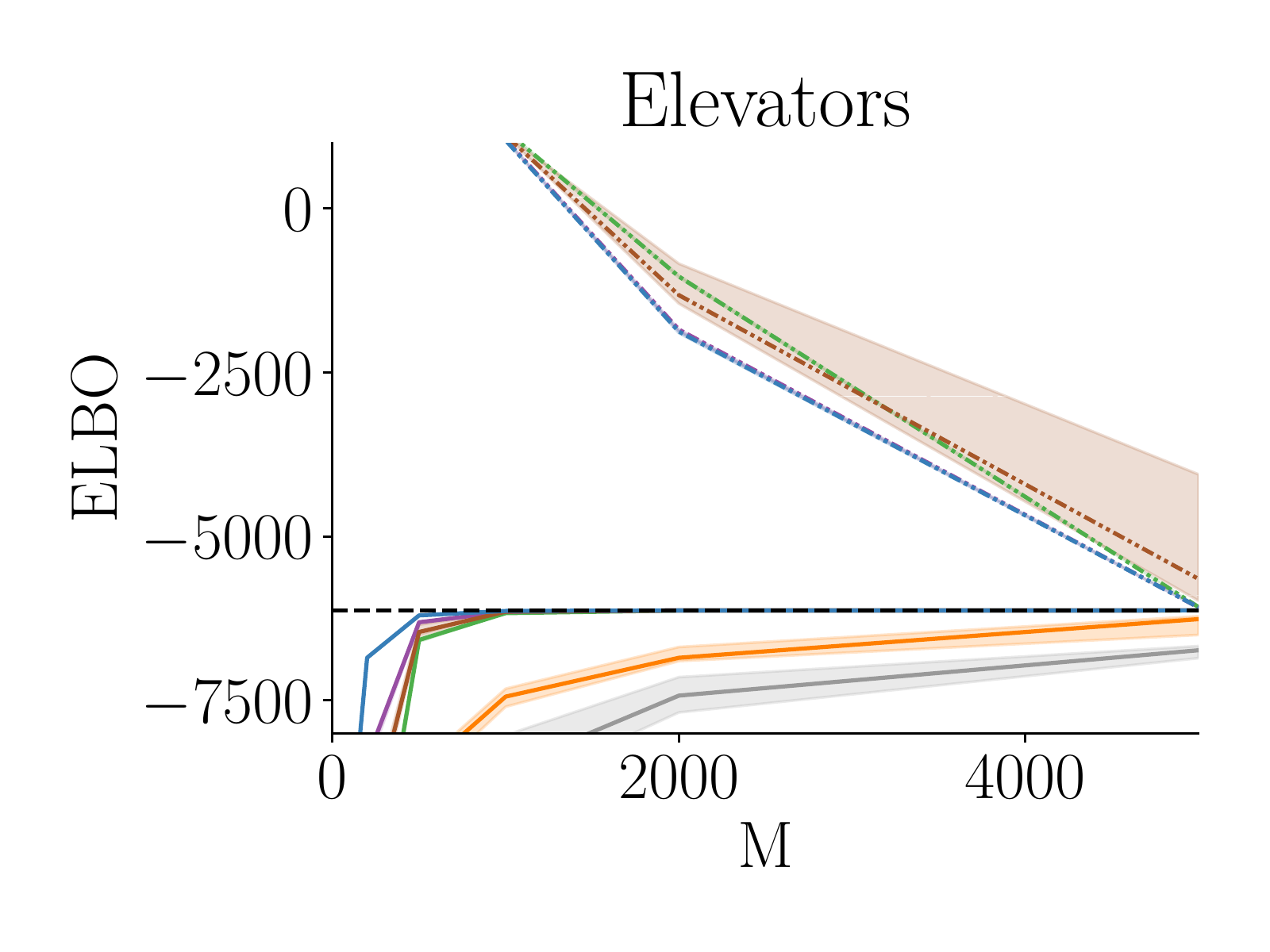}
    \includegraphics[width=.35\textwidth,trim=.6cm  .5cm .5cm 1.1cm, clip]{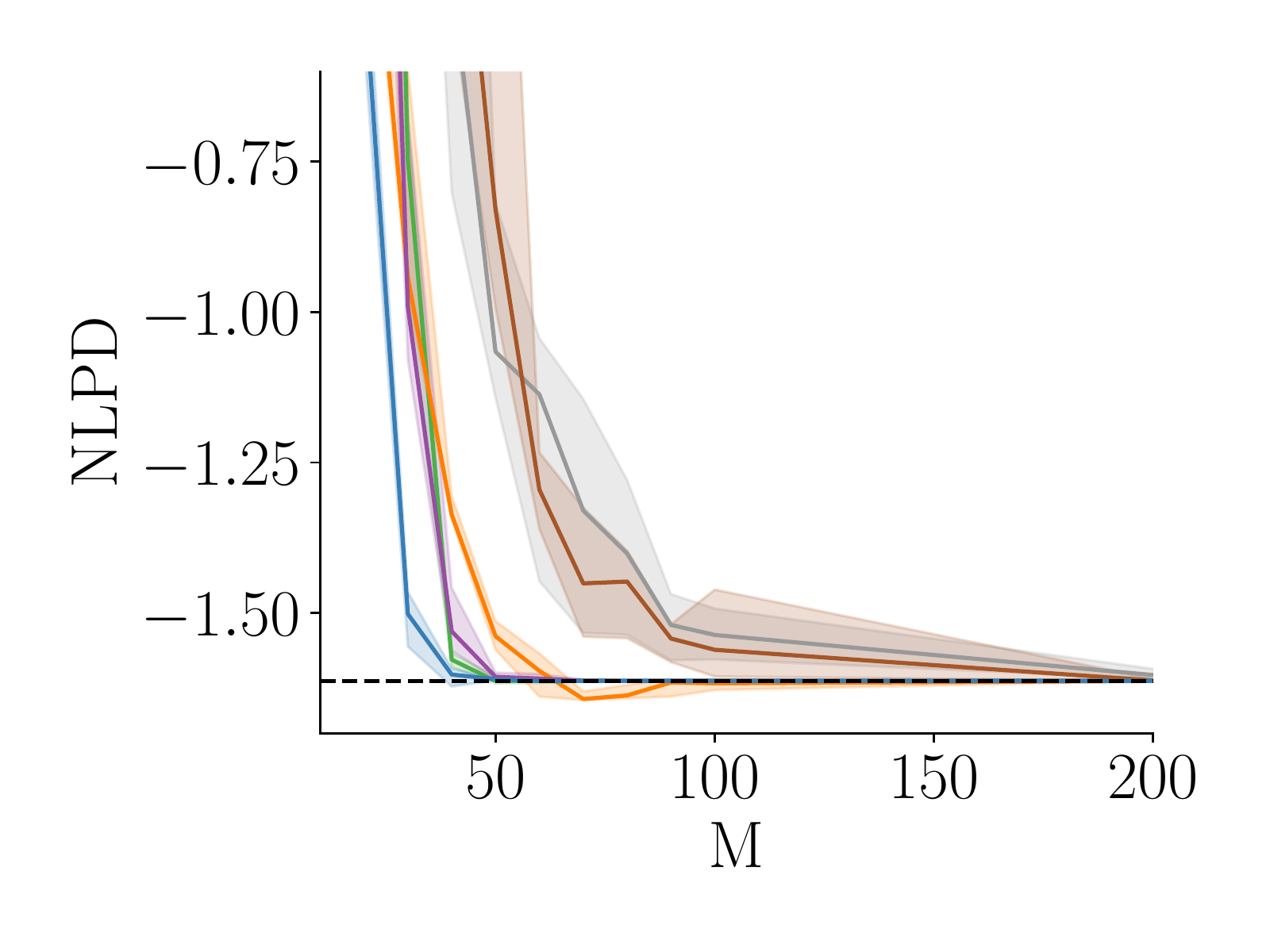}
    \includegraphics[width=.31\textwidth,trim=1.7cm  .5cm .5cm 1.1cm, clip]{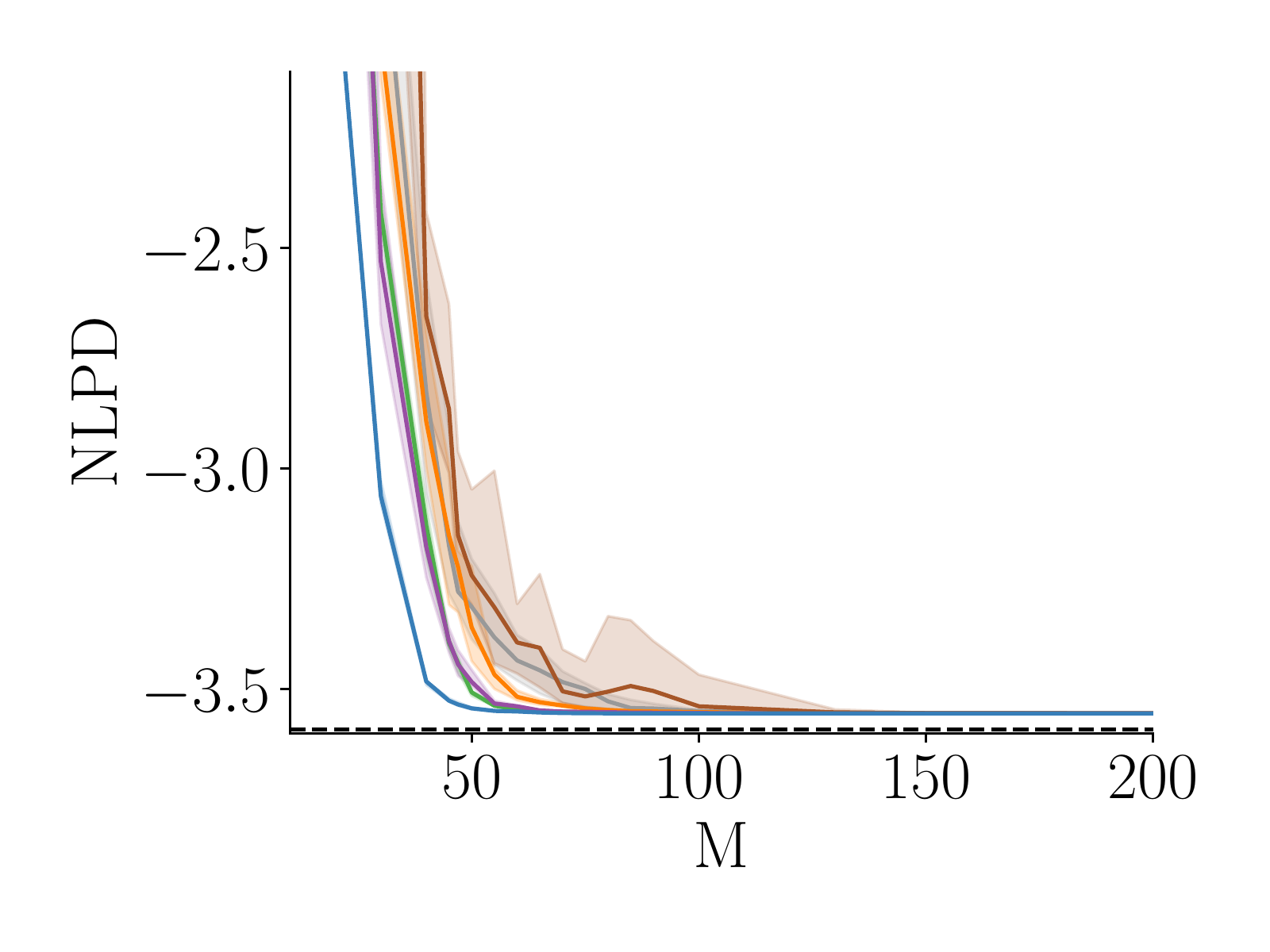}
    \includegraphics[width=.31\textwidth,trim=1.7cm  .5cm .5cm 1.1cm, clip]{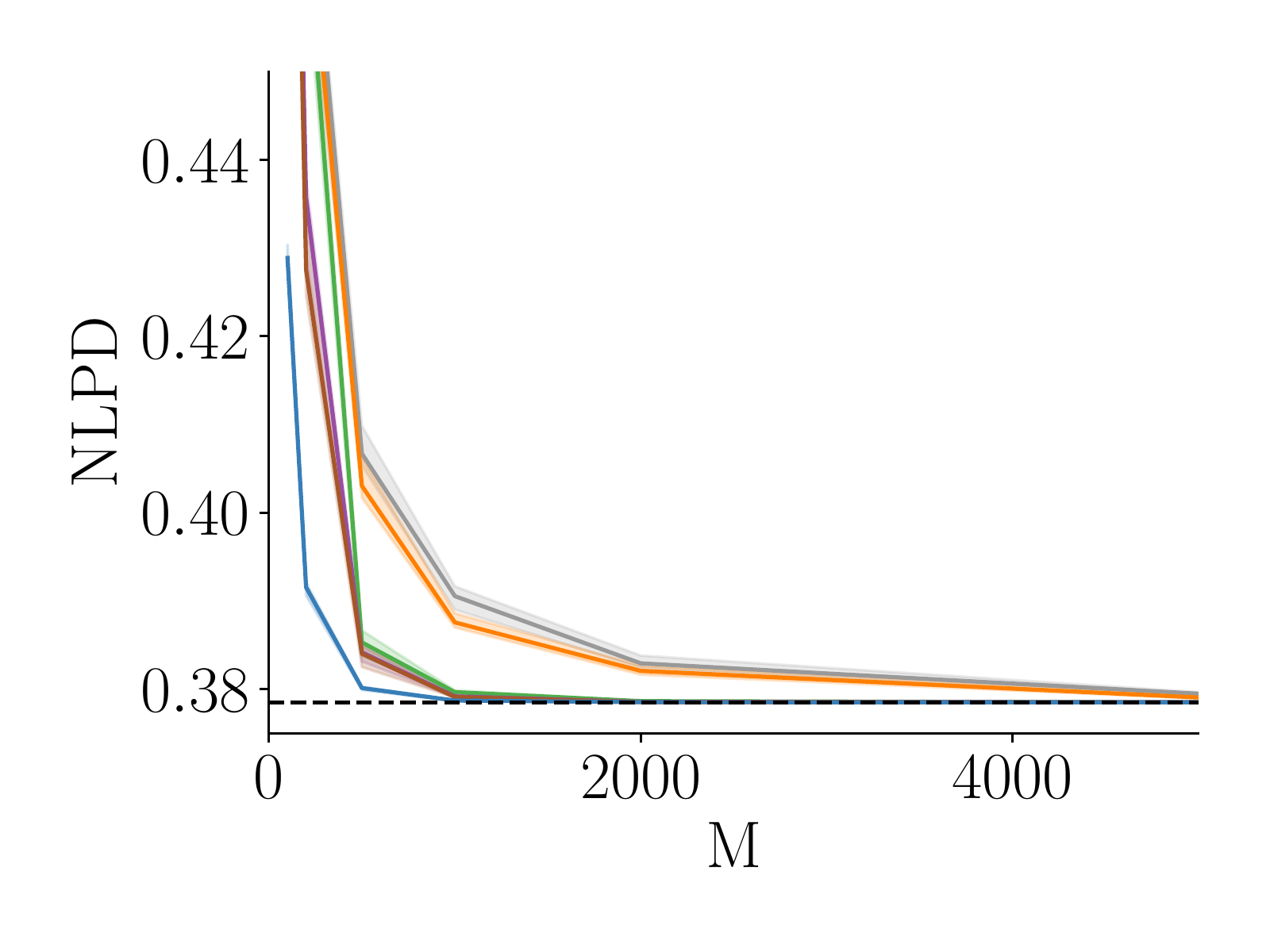}
    \includegraphics[width=.35\textwidth,trim=.6cm  .5cm .5cm .9cm, clip]{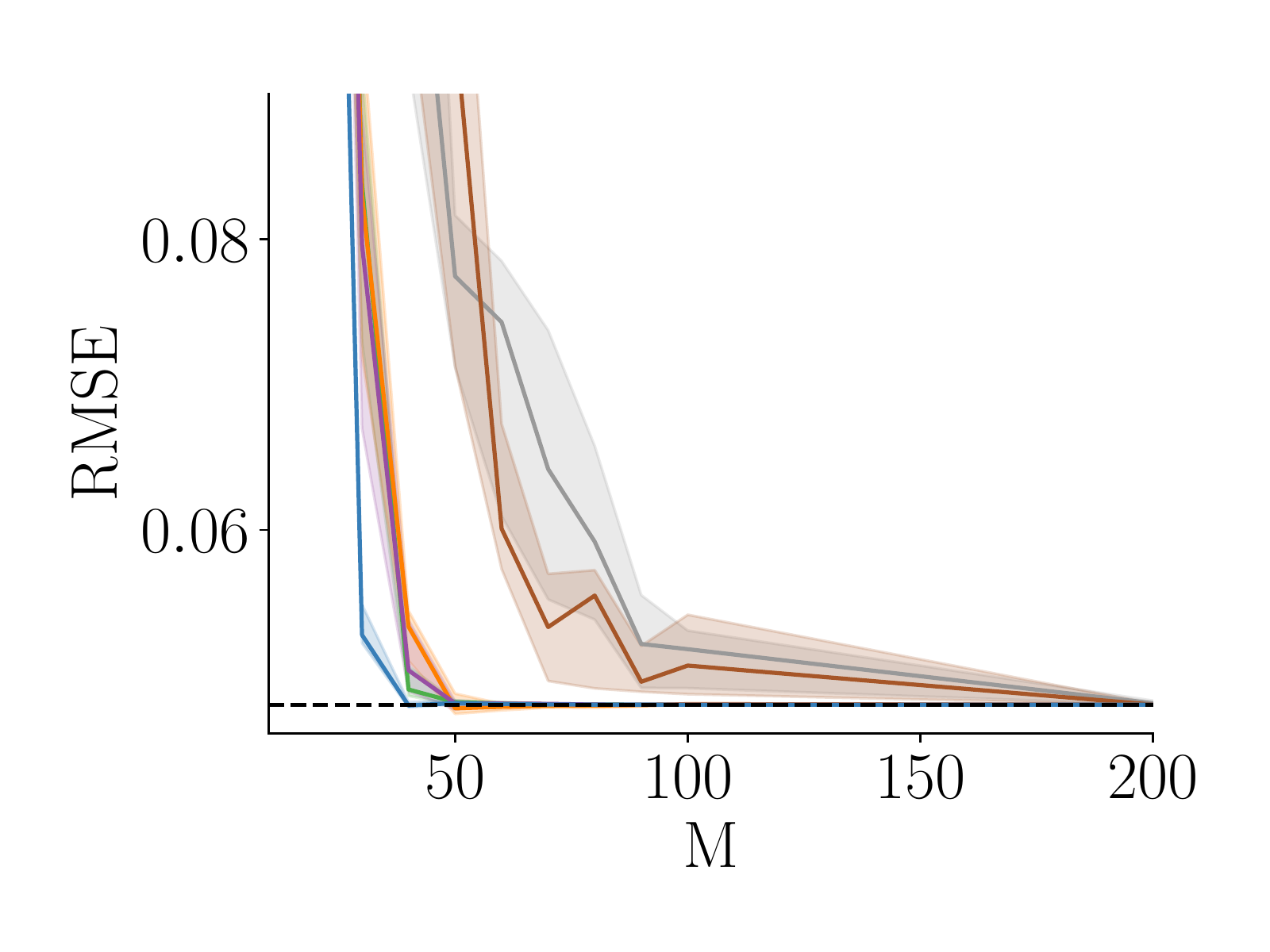}
    \includegraphics[width=.31\textwidth,trim=1.7cm  .5cm .5cm .9cm, clip]{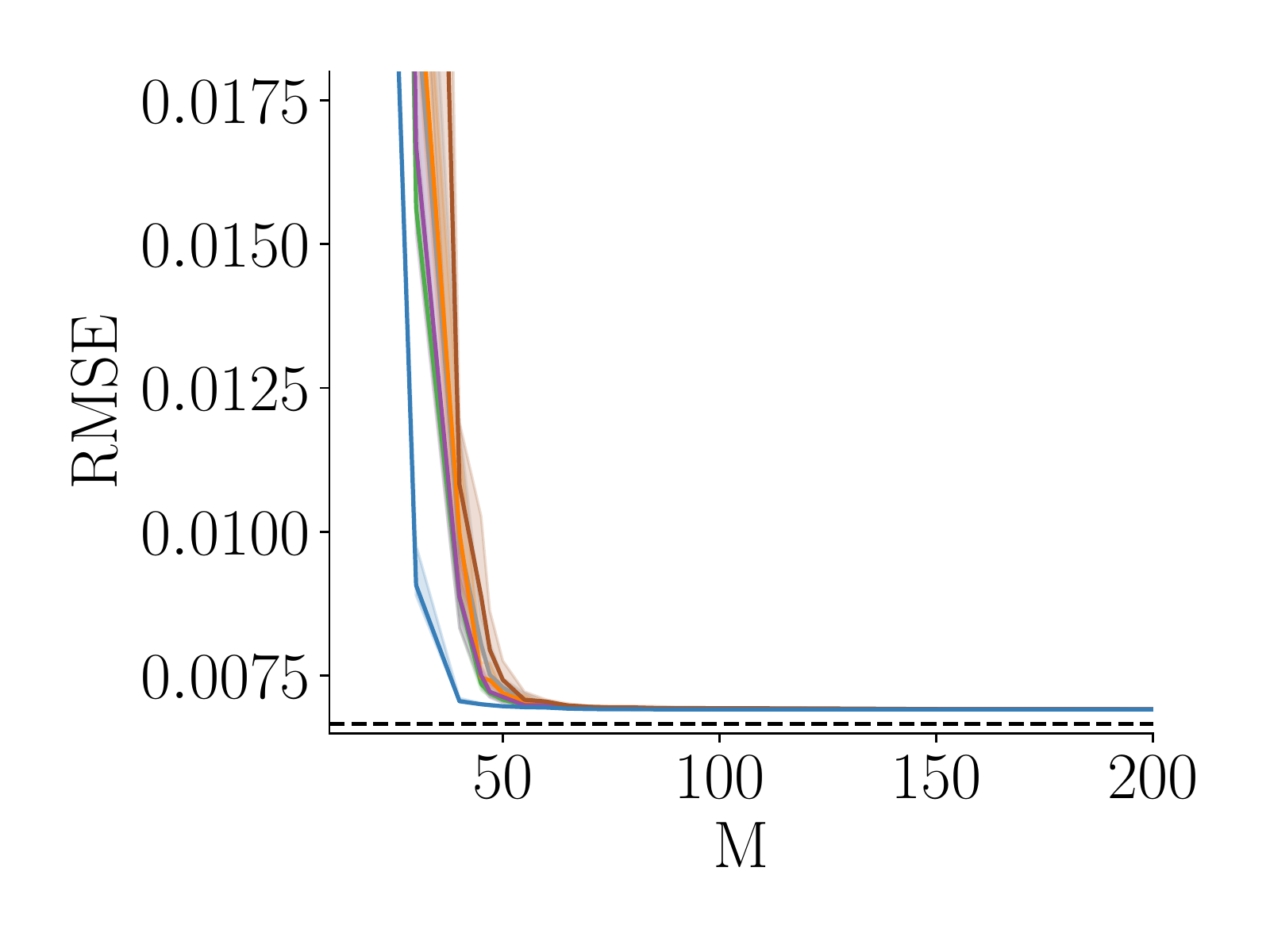}
    \includegraphics[width=.31\textwidth,trim=1.7cm  .5cm .5cm .9cm, clip]{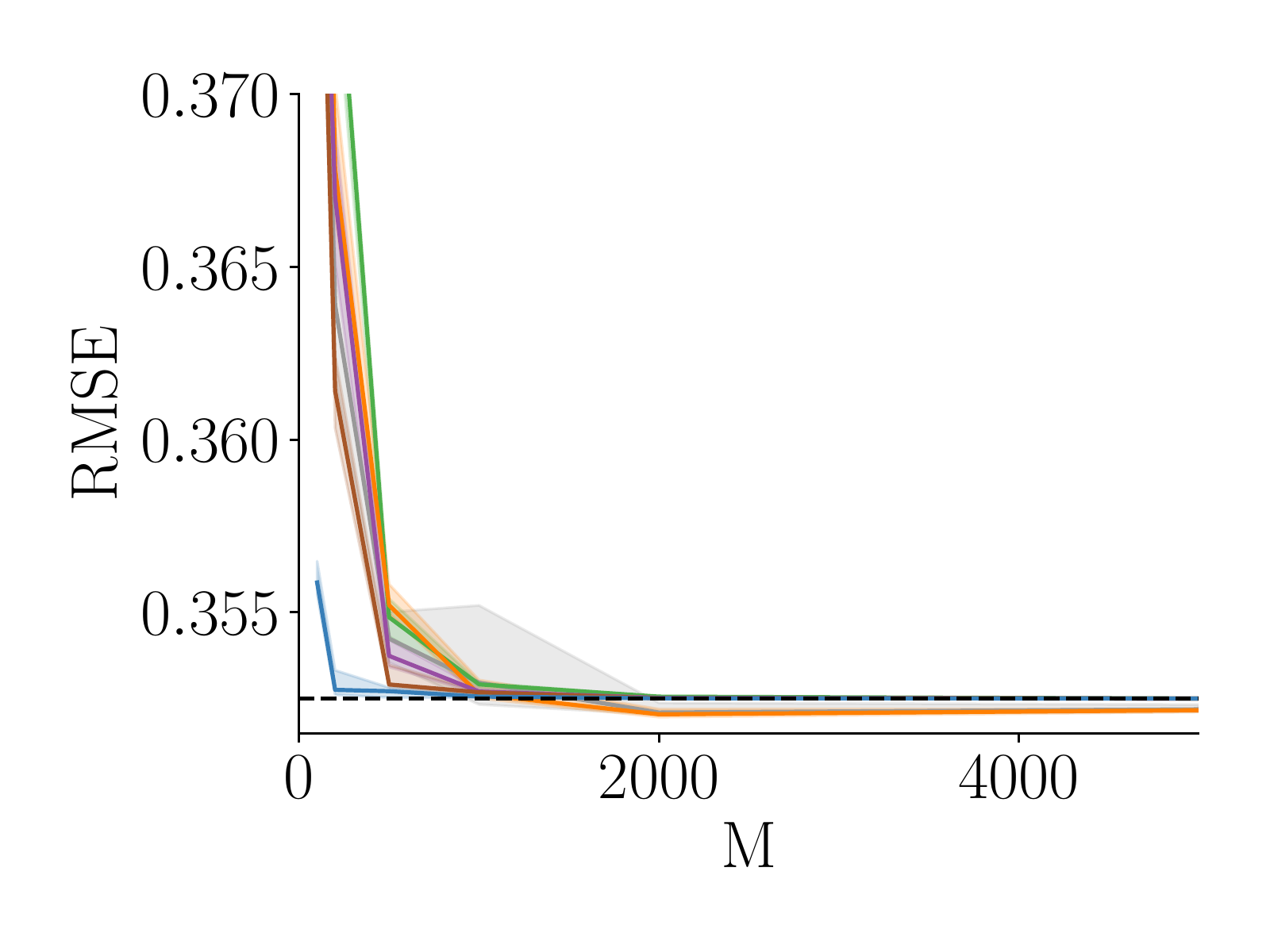}
    \includegraphics[width=\textwidth]{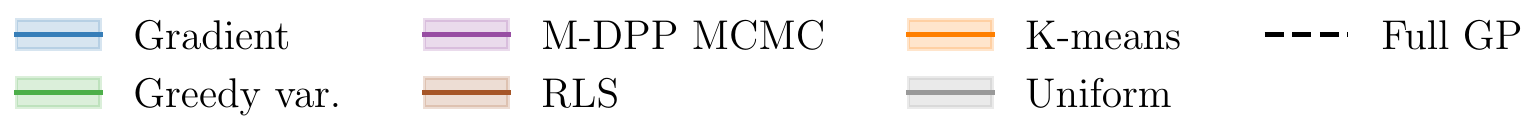}
    \caption{Performance of various methods for selecting inducing inputs on 3 data sets (each column corresponds to one data set) with fixed model hyperparameters. The top row shows the evidence lower bounds as well as an upper bound on the log marginal likelihood ($\mcU_2$ in \cref{lem:titsias-aposteriori}), the middle row the per datapoint negative log predictive density on held-out test points and the bottom row the root mean square error on test data. The solid lines show the median of 10 initializations using the given method for selecting inducing inputs, while the shaded region represents the $20-80\%$. The dashed black line shows the performance of the exact GP regressor on the given data set.}
    \label{fig:fixed-hyp}
\end{figure}

\Cref{fig:fixed-hyp} shows the performance of various methods of selecting inducing points as we vary $M$, as measured by the evidence lower bound, test root mean squared error and per data point test negative log predictive density. From the results, we can observe the following:
\begin{itemize}
    \item For very sparse models where the ELBO is considerably lower than the true marginal likelihood, gradient-based tuning of the inducing inputs consistently performs best in all metrics.
    \item The benefit of gradient-based tuning is small when many inducing points are added, \emph{provided they are added in the good locations}. Greedy variance selection and $M$-DPP find these good locations, as they consistently recover the true GP's performance with only a small number of additional inducing variables.
    \item K-means, uniform subsampling, and RLS tend to underperform, and require far more inducing variables to converge to the exact solution. In our experiment, they never converge quicker than greedy variance selection.
    \item In terms of the upper bound, greedy variance selection and $M$-DPP sampling both provide the best results.
\end{itemize}

Greedy variance selection seems to provide all the desirable properties in this case: convergence to the exact results with few inducing variables, simple to implement, and fast since it does not require as many expensive operations as optimization or sampling. The approximate ridge leverage score algorithm is also reasonably fast, and perhaps careful tuning of hyperparameters or different algorithms for approximate ridge leverage scores could lead to improved performance in practice.

\subsubsection{Training procedure and hyperparameter optimization}
In the previous section, the hyperparameters were fixed to values maximizing the log marginal likelihood. When sparse GP approximations are applied in practice, these optimal values are unknown, and they are instead found via maximizing the ELBO, as an approximation to maximizing the exact log marginal likelihood. This comes at the cost of introducing a bias in the hyperparameters towards models where $\KL{Q}{P}$ is small \citep{turner_sahani_2011}, as implied by \cref{eqn:lb-kl-relation}. The most noticeable effect in sparse GP regression is the overestimation of $\noisevariance$ and a bias toward models with smoother sample functions \citep{bauer_understanding_2016}.

Our results imply that for large enough data sets, an approximation with high sparsity can be found for the optimal hyperparameter setting that has a small KL-divergence to the posterior. This implies that the bias in the hyperparameter selection also is likely to be small. An impediment to finding the high-quality approximation for the optimal hyperparameters, is that our results depends on the inducing inputs being chosen based on properties of the kernel with the same hyperparameters that are used for inference. Here, we investigate several procedures for jointly choosing the hyperparameters and inducing inputs, in the regime where enough inducing variables are used to recover a close to exact model.

We propose a new procedure based on the greedy variance selection discussed in the previous section. To account for the changing hyperparameters, we alternately optimize the hyperparameters, and reinitialize the inducing inputs with greedy variance selection \emph{using the updated hyperparameters}. This avoids the high-dimensional non-convex optimization of the inducing inputs, while still being able to tailor the inducing inputs to the kernel. In effect, the method behaves a bit like variational Expectation-Maximization (EM) \citep{beal2003variational}, with the inducing input selection taking the place of finding the posterior. When enough inducing points are used, the reinitialization is good enough to make the ELBO almost tight for the current setting of hyperparameters. We terminate when the reinitialization does not improve the ELBO. We note that reinitialization would not benefit K-means or uniform initializations (beyond random chance), as the inducing points that are selected do not depend on the setting of the kernel hyperparameters. 

\begin{figure}[htb]
    \centering
    \includegraphics[width=.35\textwidth,trim=.5cm .5cm .5cm .5cm, clip]{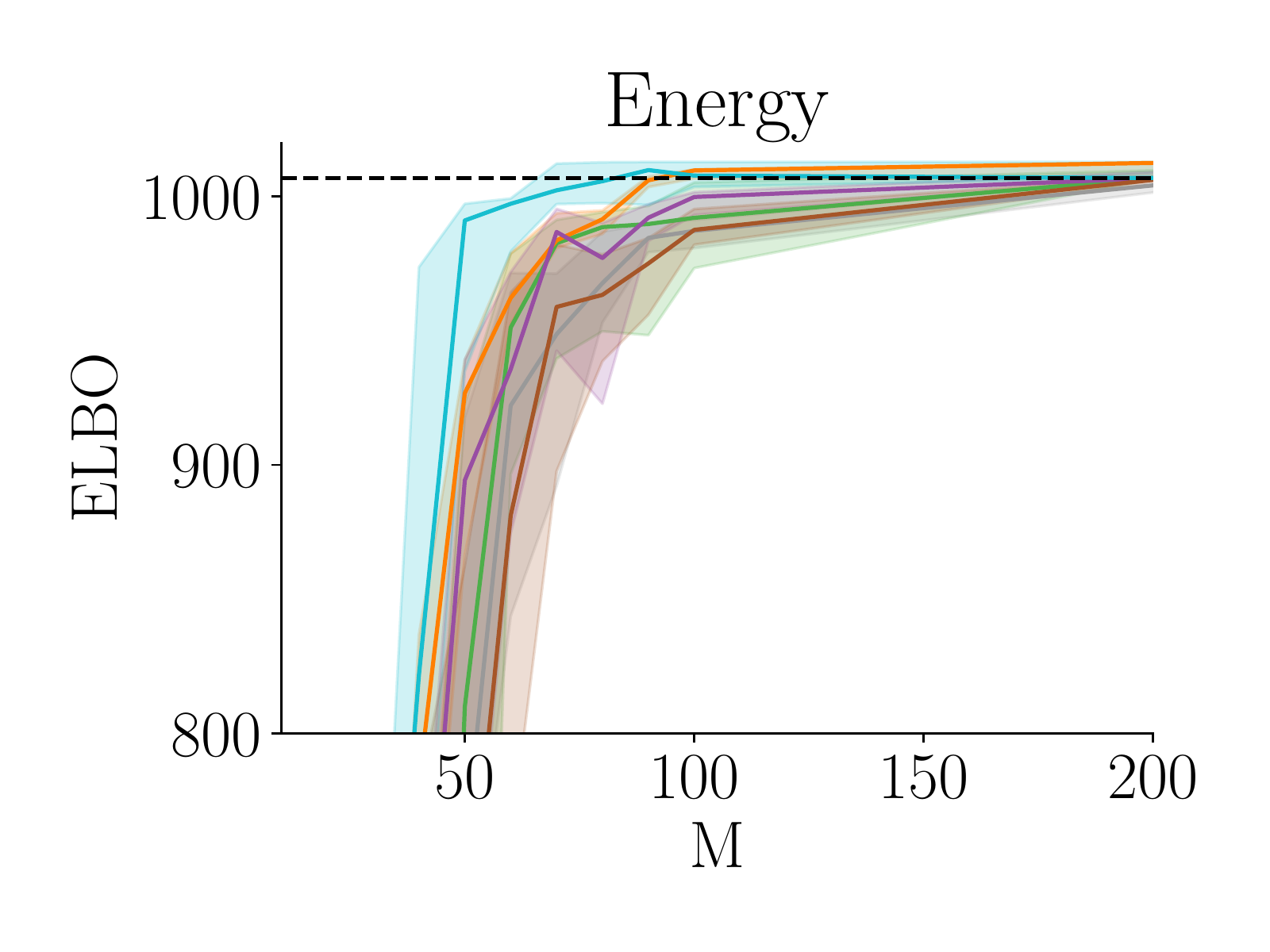}
    \includegraphics[width=.315\textwidth, trim=1.7cm .5cm .5cm .5cm, clip]{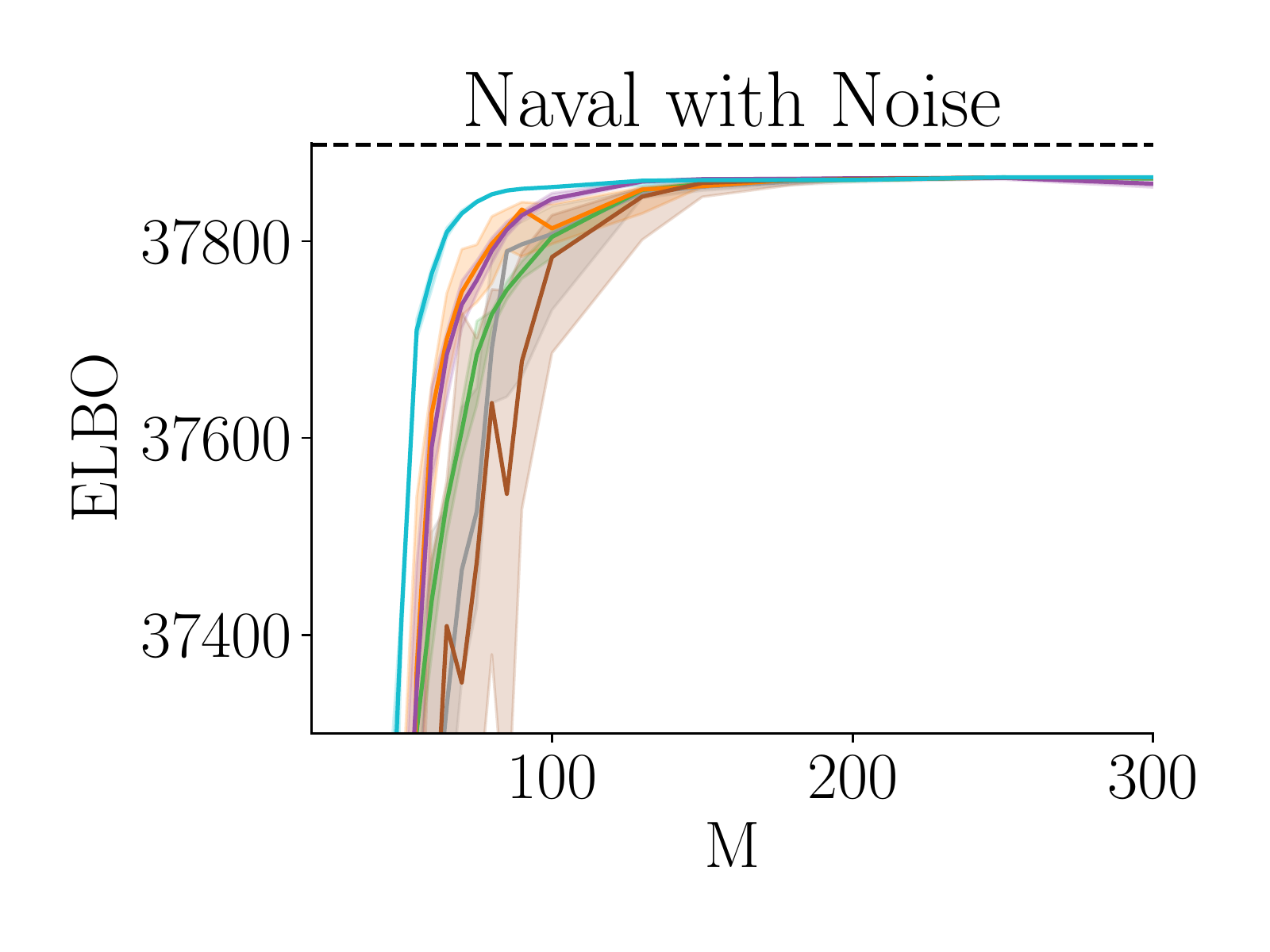}
    \includegraphics[width=.315\textwidth, trim=1.7cm .5cm .5cm .5cm, clip]{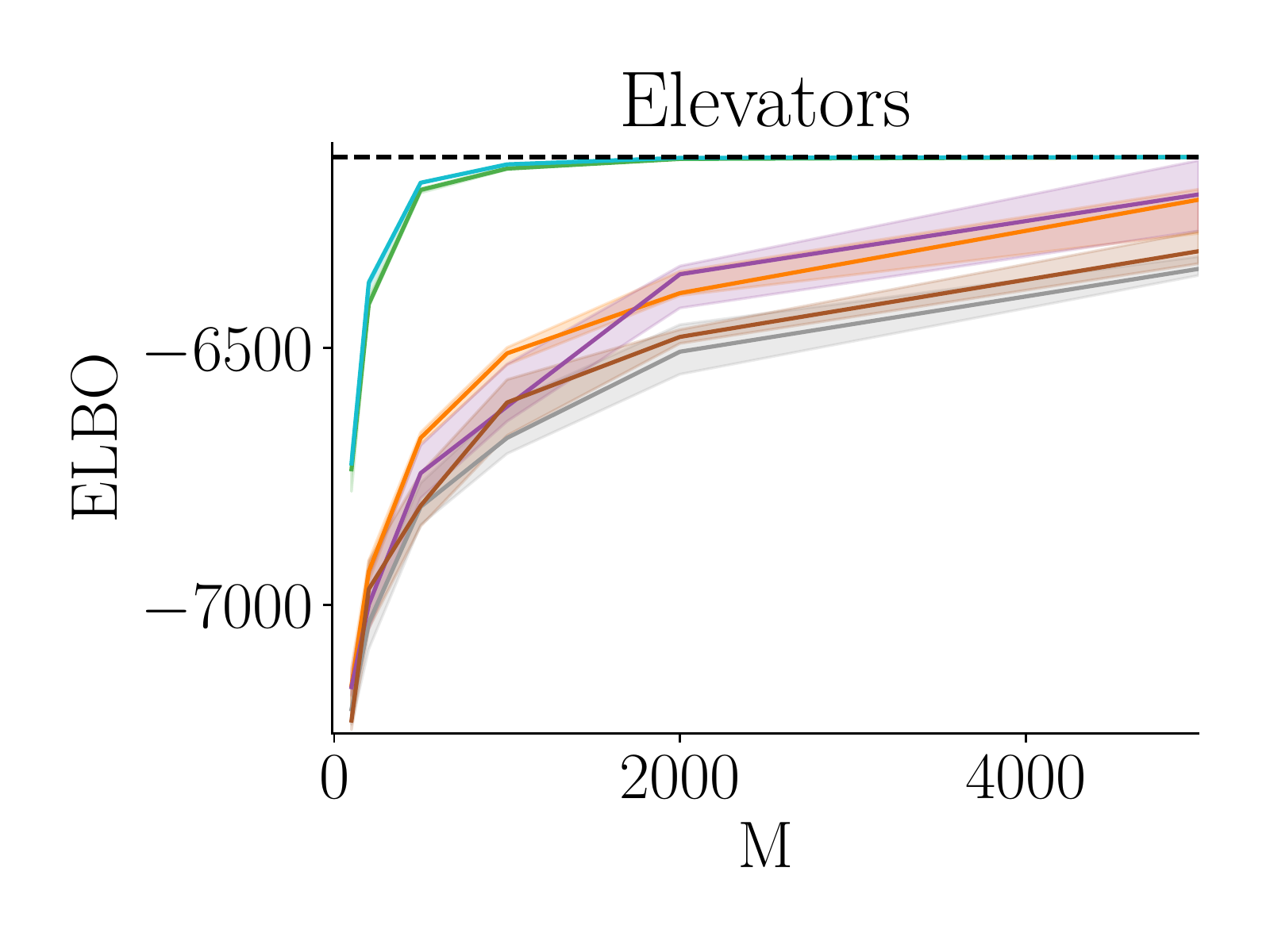}
    \includegraphics[width=.35\textwidth, trim=.5cm .5cm .5cm .5cm, clip]{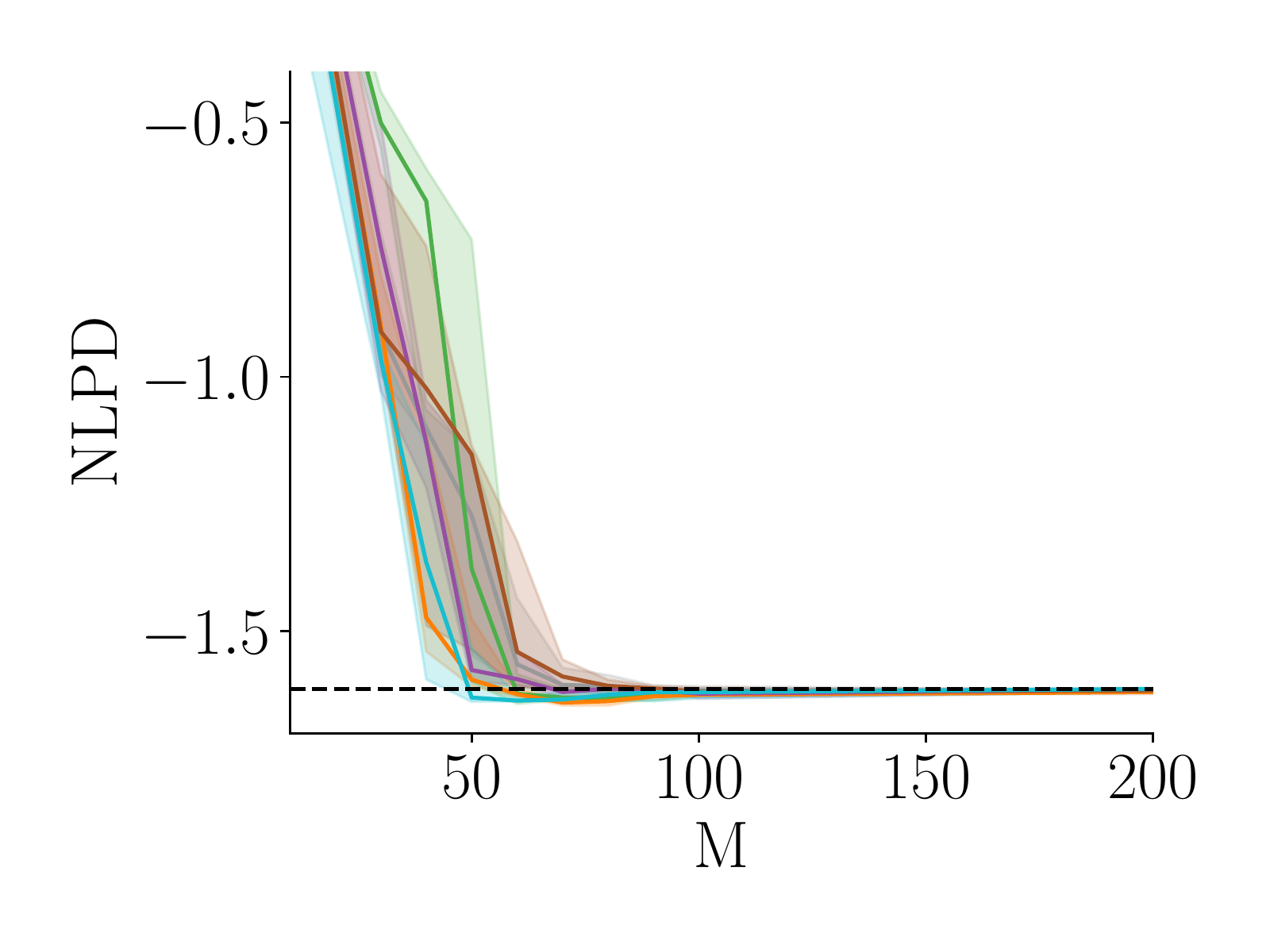}
    \includegraphics[width=.315\textwidth, trim=1.95cm .5cm .5cm .5cm, clip]{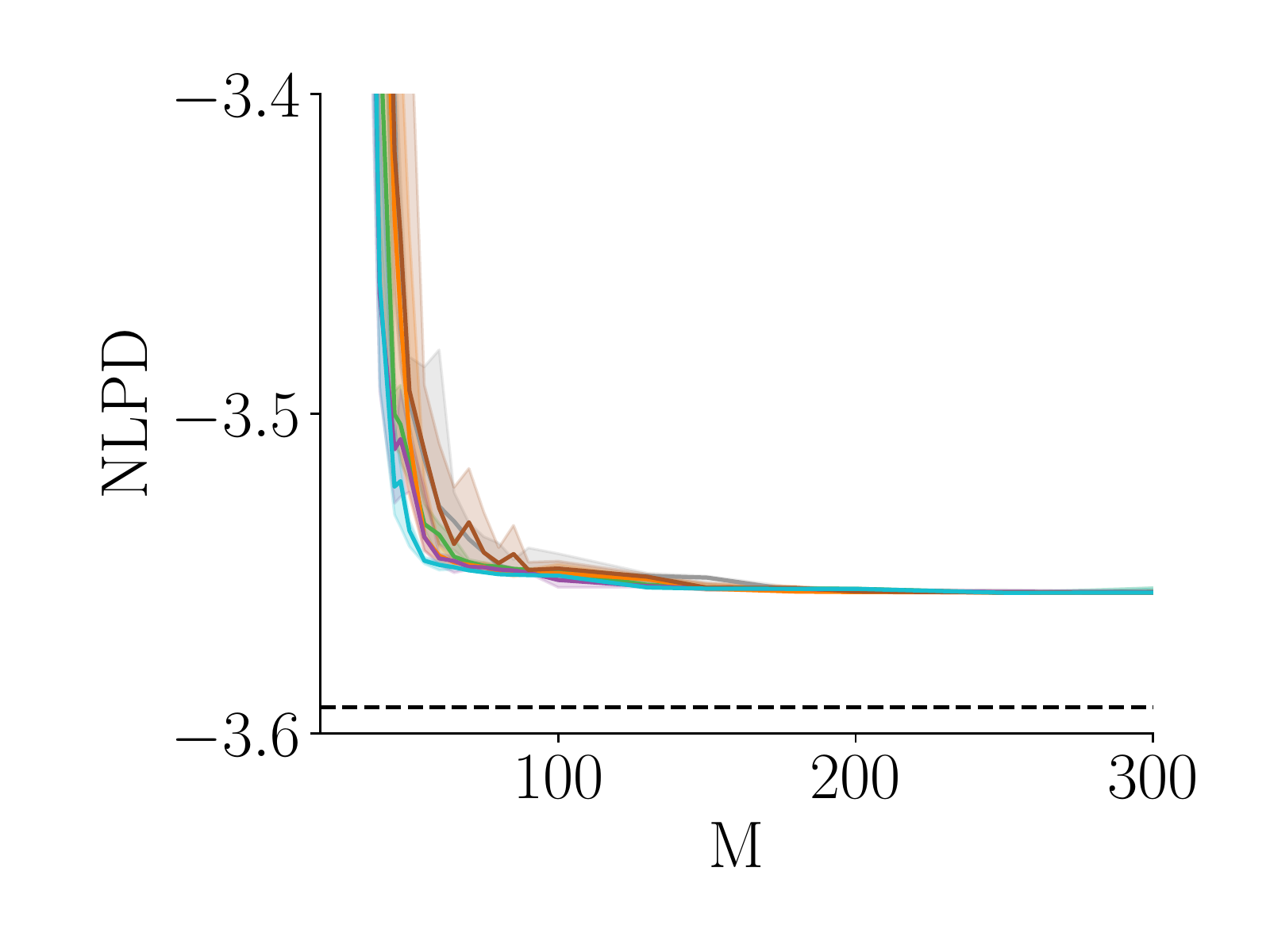}
    \includegraphics[width=.315\textwidth, trim=1.45cm .5cm .5cm .5cm, clip]{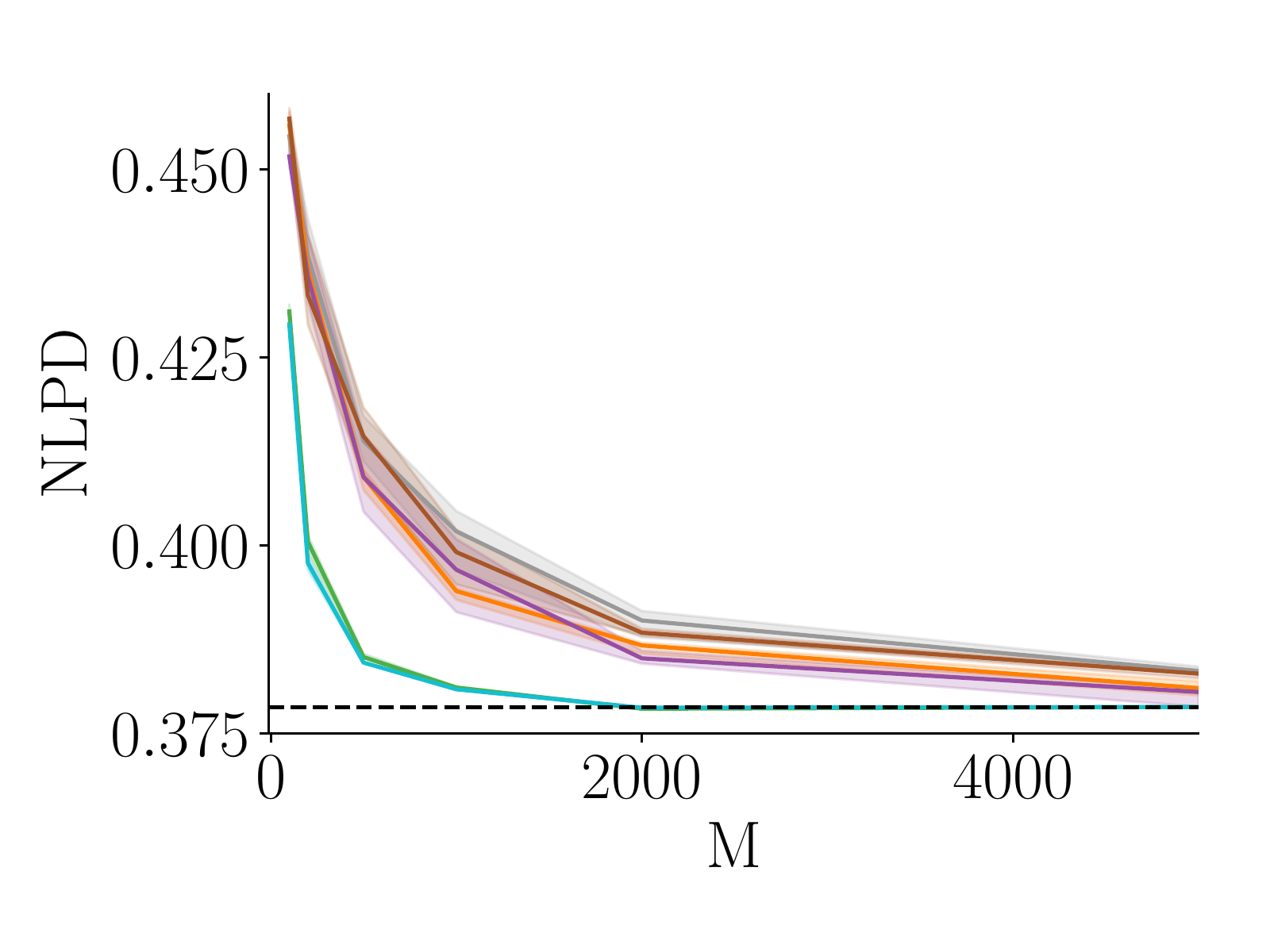}
    \includegraphics[width=.35\textwidth, trim=.5cm .5cm .5cm .5cm, clip]{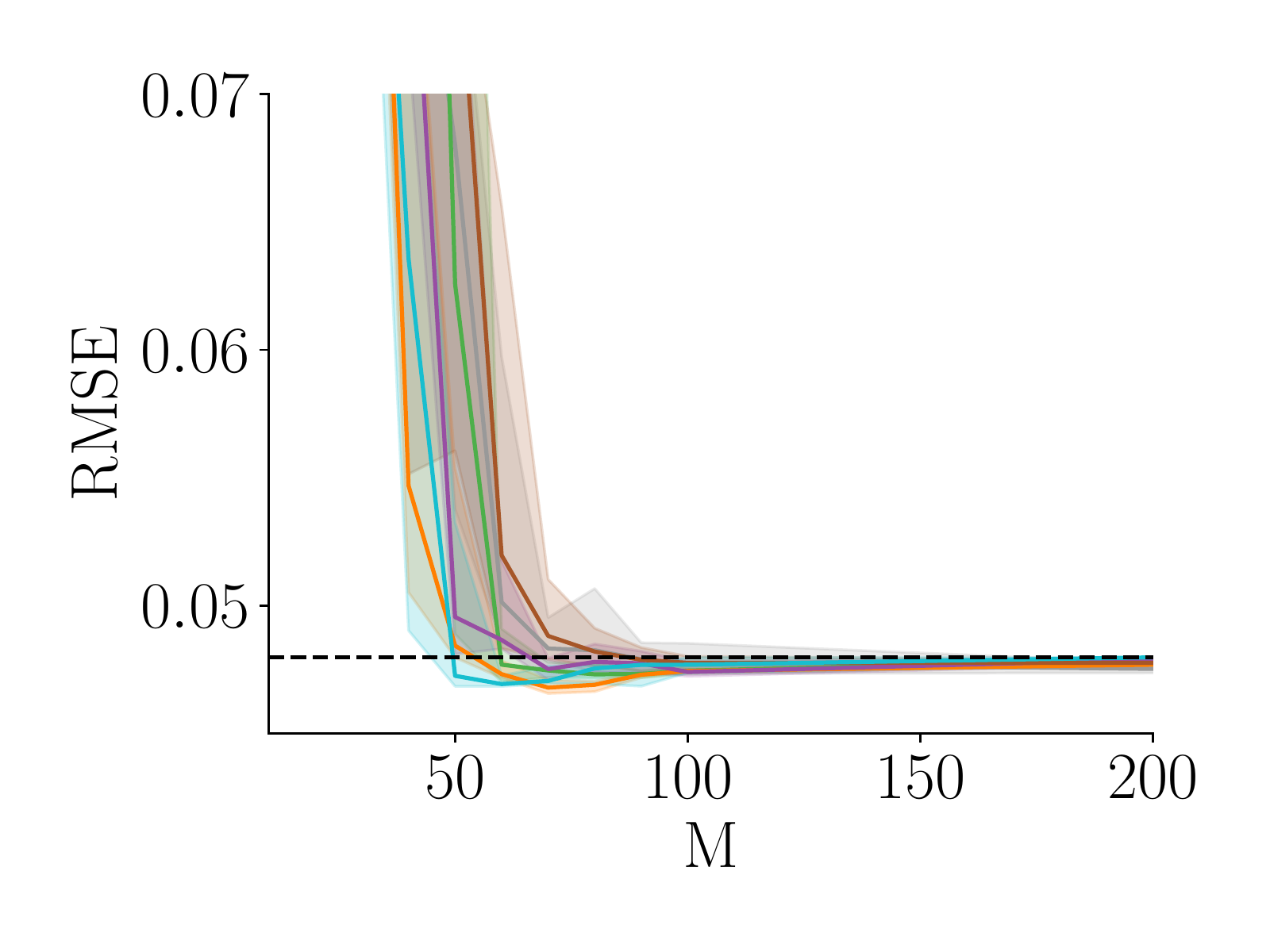}
    \includegraphics[width=.315\textwidth, trim=1.7cm .5cm .5cm .5cm, clip]{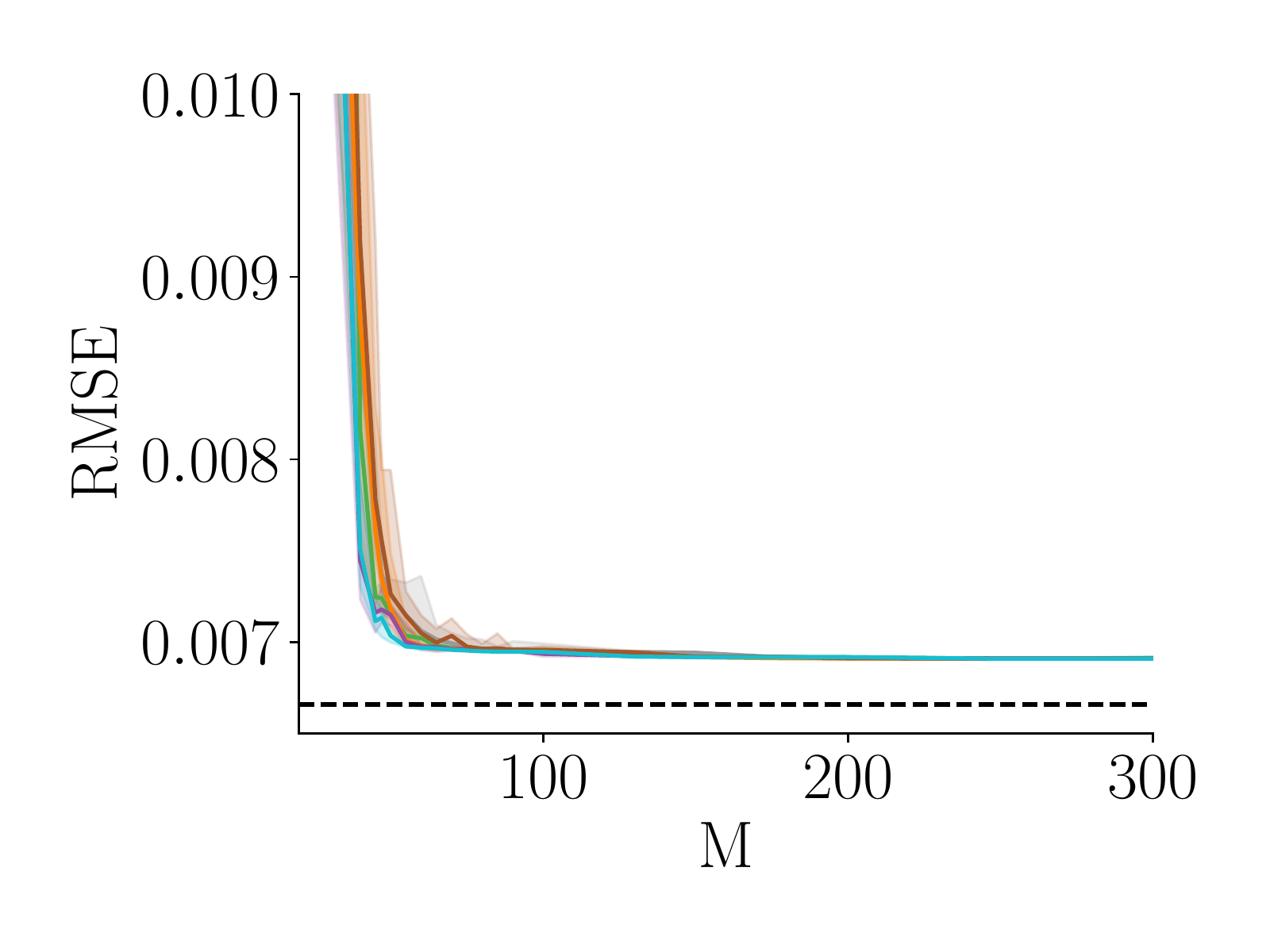}
    \includegraphics[width=.315\textwidth, trim=1.7cm .5cm .5cm .5cm, clip]{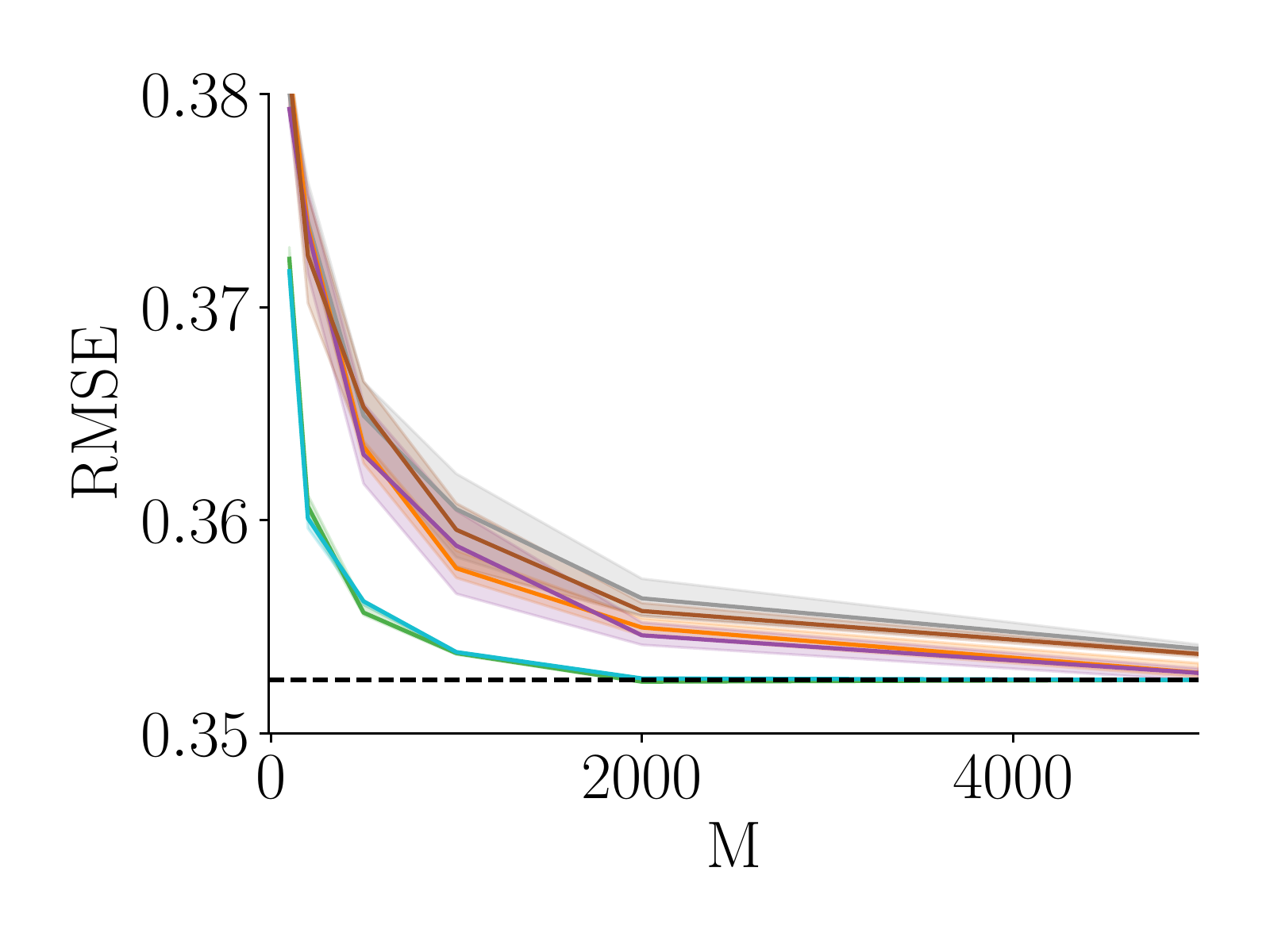}
    \includegraphics[width=\textwidth]{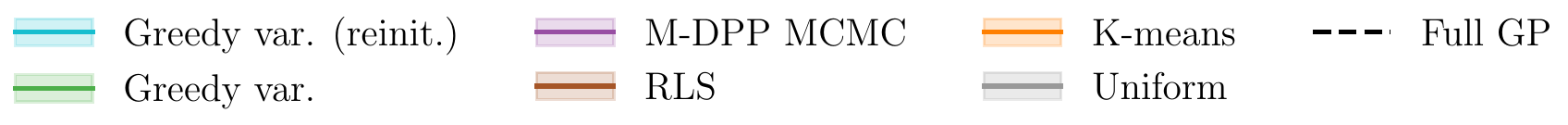}
    \caption{Performance of various methods for selecting inducing inputs on 3 data sets (each column corresponds to one data set) with model parameters learned via L-BFGS. The top row shows the evidence lower bounds, the middle row the per datapoint negative log predictive density on held-out test points and the bottom row the root mean square error on test data. The solid lines show the median of 10 initializations using the given method for selecting inducing inputs, while the shaded region represents the $20-80\%$. The dashed black line shows the performance of the exact GP regressor on the given data set.}
    \label{fig:opt-hyp}
\end{figure}
We start our evaluation by running all methods from the previous section in addition to greedy variance selection with reinitialization (\cref{fig:opt-hyp}). The initial inducing inputs are set with the untrained initialized hyperparameters, after which the hyperparameters are maximized w.r.t.~ELBO (Eq.~\ref{eqn:elbo}) using L-BFGS, with the reinitialization being applied for ``Greedy variance (reinit.)''. We observe that the reinitialized greedy variance method provides consistent fast convergence to the exact model.

To evaluate the benefit of gradient-based optimization, we compare it to the reinitialized greedy variance method (the best from \cref{fig:opt-hyp}), as well as K-means. For the initial setting of the inducing inputs when optimizing inducing inputs, we use the greedy variance selection (denoted ``gradient''). Since \cref{fig:opt-hyp} shows that optimization of the inducing inputs is not needed to converge to the exact solution, the question becomes whether it is \emph{faster} to perform gradient-based optimization. We choose $M$ to be the smallest value for which the ELBO given by the gradient method converges to within a few nats of the exact marginal likelihood based on \cref{fig:fixed-hyp}. We plot the optimization traces in \cref{fig:opt-trace} for several runs to account for random variation in the initializations.

In this constrained setting, we see different behaviours on the different data sets. One constant is that placing inducing points using K-means leads to sub-optimal performance compared to the best method. For the Energy data set, ``greedy var (reinit)'' suffers from convergence to local optima. This is caused by the low sparsity, and disappears if more inducing points are used (see \cref{fig:opt-hyp}). For the Naval data set, we see \emph{very} slow convergence when using gradient-based optimization initialized with greedy variance selection. K-means underperforms and also suffers from local optima, with reinitialization reliably reaching the best ELBO. For elevators, reinitialization reaches the optimal ELBO fastest.

We note that in the reinitialization method the hyperparameter optimization step was terminated when L-BFGS had determined convergence according to the default Scipy settings. This leads to a characteristic ``step'' pattern in the optimization traces, where progress halts for many iterations towards the end of a hyperparameter optimization phase, followed by large gains after a reinitialization of the inducing inputs. By terminating the hyperparameter optimization earlier after signs of stagnation, the reinitialization method could be significantly sped up. In addition, we measure computational cost through the number of function evaluations. This does not take into account the additional cost of computing the gradients for the inducing inputs, which make up the bulk of parameters that are to be optimized. As the amount of computation needed to reinitialization the inducing points is comparable to the computation required in as single iteration of gradient descent, \cref{fig:opt-trace} likely understates the computational savings of the reinitialization method.

\begin{figure}[t]
    \centering
    \includegraphics[width=.34\textwidth,trim=.5cm  .2cm .2cm .2cm, clip]{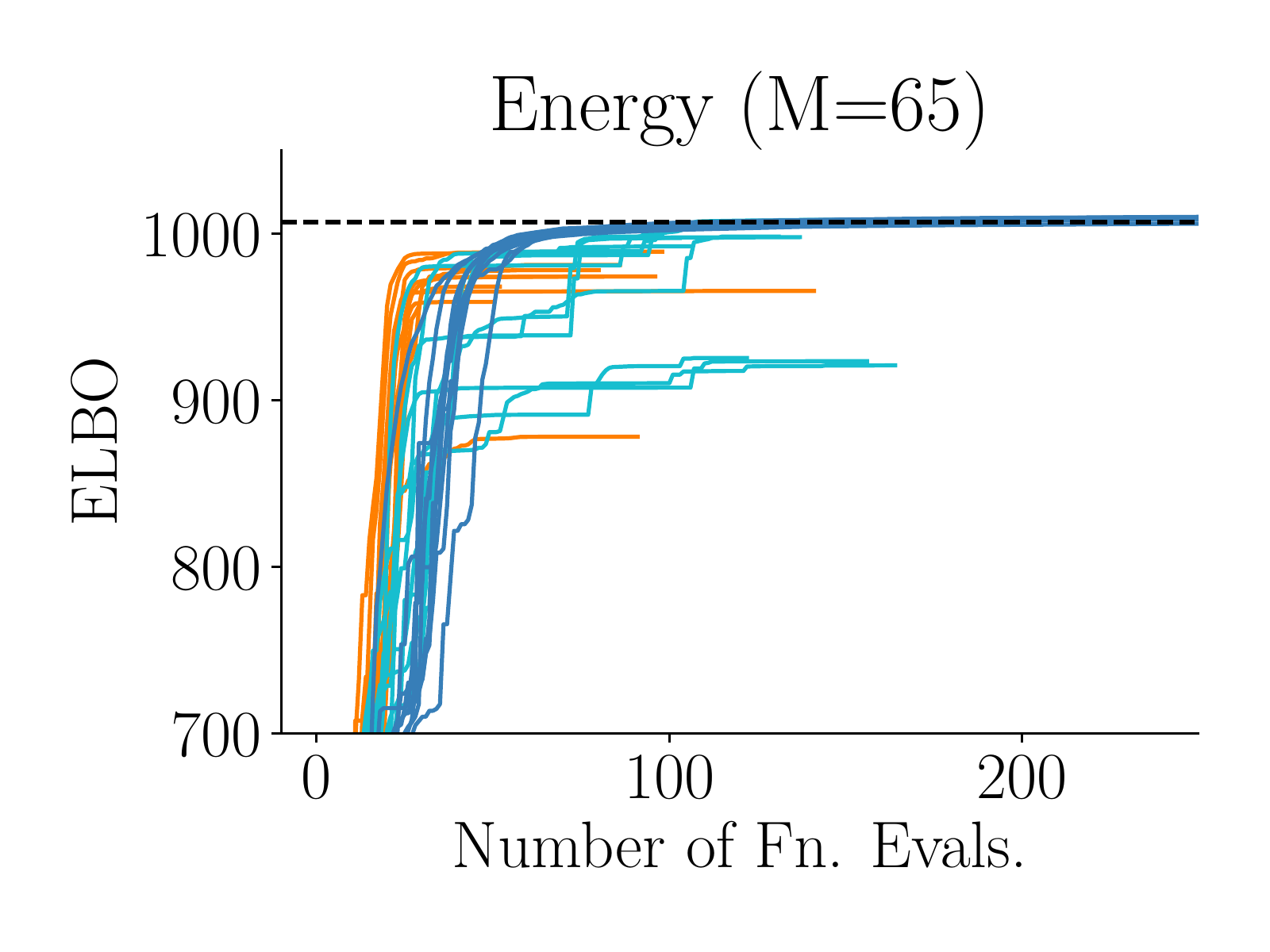}
    \includegraphics[width=.31\textwidth, trim=1.6cm  .2cm .2cm .2cm, clip]{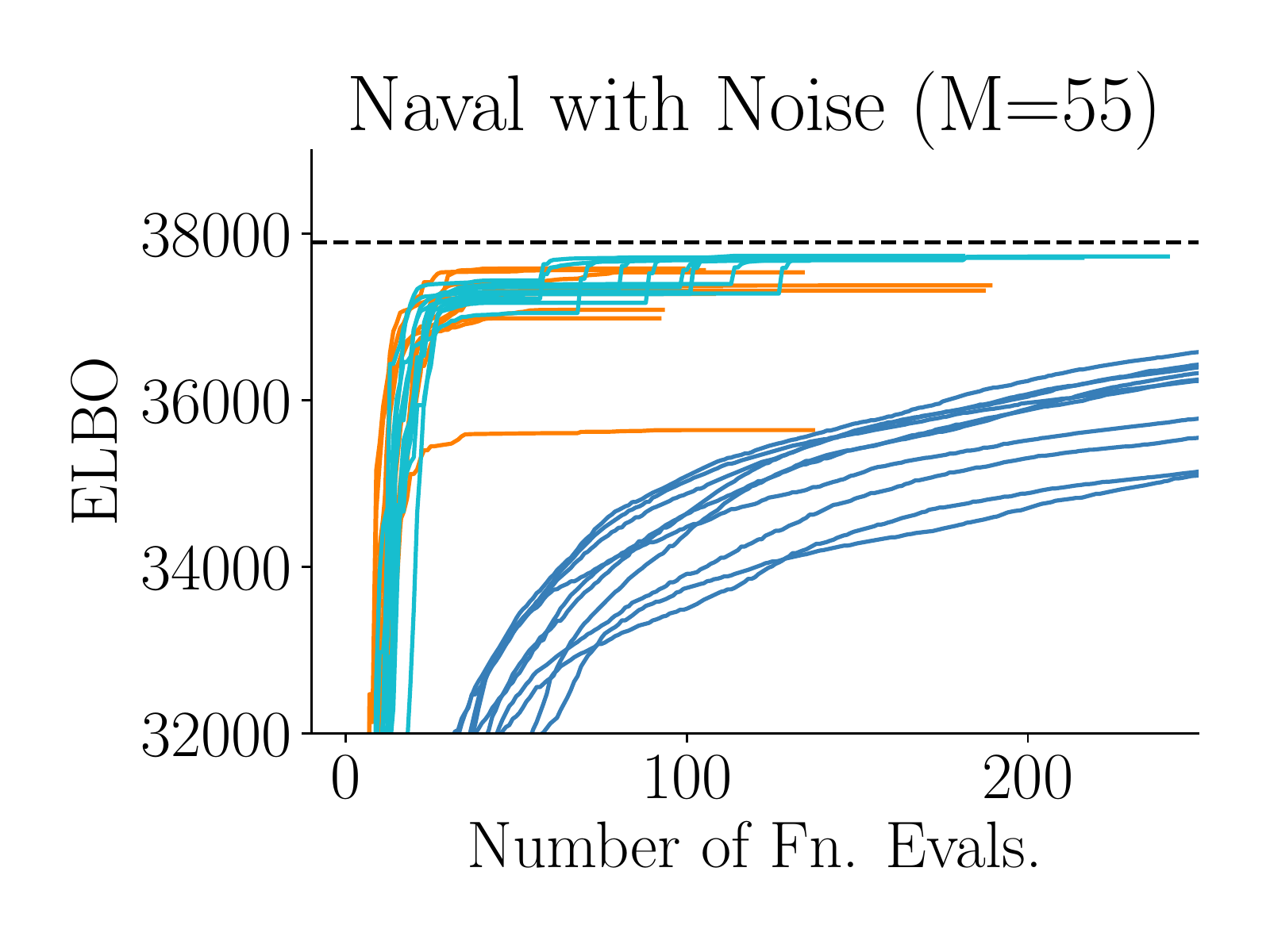}
     \includegraphics[width=.31\textwidth, trim=1.6cm  .2cm .2cm .2cm, clip]{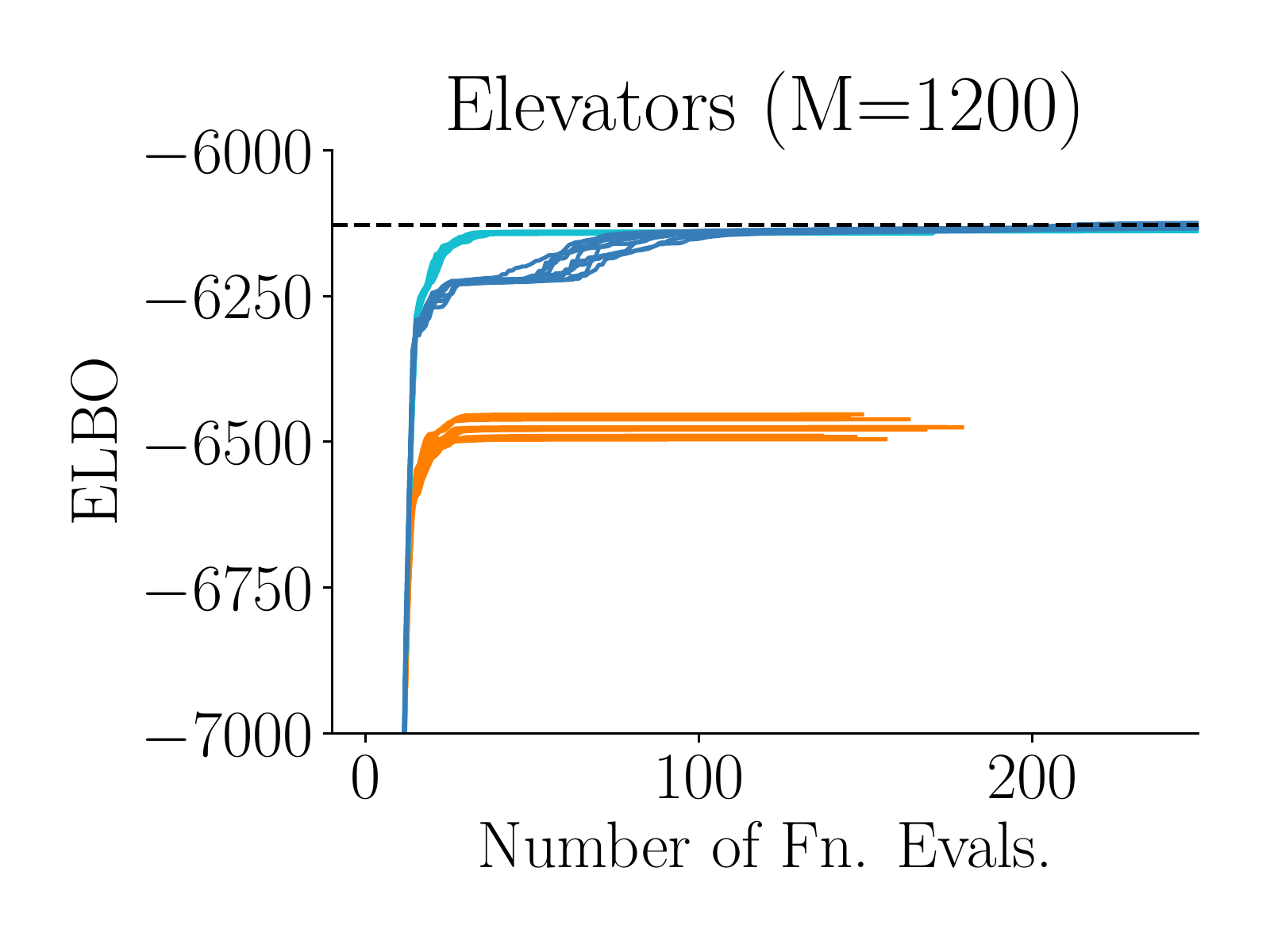} \includegraphics[width=\textwidth, trim=.5cm  .2cm .2cm .2cm, clip]{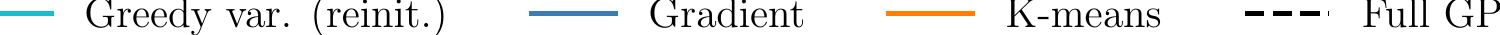}
    \caption{ELBO against the number of function evaluations called during a run of L-BFGS. ``Gradient'' uses gradient-based optimization for the inducing inputs as well as the hyperparameters, with the inducing inputs initialized using greedy variance selection.}
    \label{fig:opt-trace}
\end{figure}

\subsubsection{Recommendation for Inducing Input Selection}
The main conclusion from our empirical results is that well-chosen inducing inputs that are reinitialized during hyperparameter optimization give highly accurate variational approximations to the results of exact GPs. While performing gradient-based optimization of the inducing inputs may lead to improved performance in settings that are constrained to be very sparse, in some instances it is not worth the additional effort. We provide a GPflow-based \citep{GPflow2017} implementation of the initialization methods and experiments that builds on other open source software \citep{coelho2017jug,2020SciPy-NMeth}, available at \url{https://github.com/markvdw/RobustGP}.

It is important to note that we only considered data sets where sparse approximations were practically possible. The ``kin40k'' UCI data set is a notable example where a squared exponential GP regression model with learned hyperparameters could not accurately be approximated, due to a lengthscale that continuously decreased with increasing $M$. Given the underfitting and significant hyperparameter bias \citep{bauer_understanding_2016}, one can question whether variational approximations are appropriate. In cases where the covariates are less heavily correlated under the prior, conjugate gradient approaches \citep{gibbs1997efficient,davies2016thesis,gardner2018gpytorch} may be better. We choose to not make a recommendation for how to choose inducing variables in cases where the variational approximation is poor.

%%%%%%%%%%%%%%%%%%%%%%%%%%CONCLUSION%%%%%%%%%%%%%%%%%%%%%%%%%%%%%%%%%%%%%%%%%%%%%%%%%
\section{Conclusions}

We provide guarantees on the quality of variational sparse Gaussian process regression when many fewer inducing variables are used than data points. We also consider lower bounds on the number of inducing variables needed in order to ensure that the KL-divergence between the approximate posterior and the full posterior is not large. These bounds provide insight into the number of inducing points that should be used for a variety of tasks, as well as suggest the sorts of problems to which sparse variational inference is well-suited. We also include an empirical results comparing the efficacy of different methods for selecting inducing inputs, which is of practical importance to the Gaussian process community. We believe that there is a great deal of interesting future research to be done on the role of sparsity in variational Gaussian process inference; both in refining the bounds given in this work and in better understanding non-conjugate inference schemes, such as  those developed in \cite{hensman_scalable_2015}.

\acks{
We would particularly like to thank Guillaume Gautier for pointing out an error in the exact k-DPP sampling algorithm cited in an earlier version of this work, and for guiding us through recent work on sampling k-DPPs that led to an amended proof. MvdW would additionally like to thank James Hensman for his guidance, and PROWLER.io for providing an excellent research environment while this work was developed.
}

%%%%%%%%%%%%%%%%%%%%%%%%%%%%%%%%%%%%%%%%%%%%%%%%%%%%%%%%%%%%%%%%%%%%%%
\renewcommand{\theHsection}{A\arabic{section}}

%%%%%%%%%%%%%%%APPENDICES%%%%%%%%%%%%%%%%%%%%%%%%%%%%%%%%%%%%
\appendix
\section{Proof of Bound on Mean and Variance of One-dimensional Marginal Distributions}\label{app:proof-marginals}
\marginalbound*
\begin{proof}
By the chain rule of KL-divergence, we have 
\begin{align}\label{eqn:bound-chain-rule}
 2\KL{q(f(x^*))}{p(f(x^*)|\data)} \leq 2\KL{Q}{P} \leq \gamma
\end{align}
for any $x^\star \in \mcX$. 

For any $x^\star \in \mcX$, the KL-divergence on the right hand side of \cref{eqn:bound-chain-rule} is a KL-divergence between one-dimensional Gaussian distributions, and has the form,
\begin{equation}\label{eqn:univariate-gaussian-kl}
   \gamma \geq2\KL{q(f(x^*))}{p(f(x^*)| \mcD)} = \frac{\sigma_1^2}{\sigma_2^2}  - 1 - \log \frac{\sigma_1^2}{\sigma_2^2} +  \frac{(\mu_1-\mu_2)^2}{\sigma_2^2} \geq \frac{\sigma_1^2}{\sigma_2^2}  - 1 - \log \frac{\sigma_1^2}{\sigma_2^2}.
\end{equation}
 Define $r= \sigma_1^2/\sigma_2^2$, so \cref{eqn:univariate-gaussian-kl} becomes $\gamma \geq r- 1- \log r$. For $\gamma < \frac{1}{5}$, we have $r-\log(r)< 1.2$, so $r \in [.493,1.78]$. For $r$ in this range, we have, $\gamma \geq r- 1- \log r \geq (r-1)^2/3$.
Solving, for $r$, we obtain the bound,
\begin{equation}\label{eqn:bound-on-marginal-variance}
\left|1-\frac{\sigma_1^2}{\sigma_2^2} \right| \leq \sqrt{3 \gamma}.
\end{equation}

We now turn to the proof of the bound relating $\mu_1$ and $\mu_2$. From \cref{eqn:univariate-gaussian-kl} and because $r-1-\log r>0$ for $r>0$, $\frac{(\mu_1-\mu_2)^2}{\sigma_2^2} \leq \gamma$. Rearranging, $|\mu_1-\mu_2| \leq  \sigma_2 \sqrt{\gamma}$.
The final bound on the mean follows from \cref{eqn:bound-on-marginal-variance}, which implies that,
\[
\sigma_2 \leq \frac{\sigma_1}{\sqrt{1-\sqrt{3\gamma}}}. 
\]
\end{proof}
%%%%%%%%%%%%%%%%%%%%%%%%%%SECTION 3 APPENDIX%%%%%%%%%%%%%%%%%%%%%%%%%%%%
\section{Proofs of A-Posteriori Bounds}\label{app:proofs-section3}
In this section, we restate and prove the upper bound on the marginal likelihood given in \citet{titsias_variational_2014}.
\upperbound*

This result relies on several properties of symmetric positive semi-definite (SPSD) matrices, which we state in \cref{prop:spsdmatrices}.
\begin{prop}[\cite{horn_matrix_1990}, Corollary 7.7.4]\label{prop:spsdmatrices}
Let $\succ$ denote the partial order on SPSD matrices induced by $A \succ B \iff A-B$ is SPSD. Then if $A \succ B$ are $N \times N$ SPSD matrices,
\begin{enumerate}
    \item $\det(A) \geq \det(B)$,
    \item If $A^{-1}, B^{-1}$ exist, then $A^{-1} \prec B^{-1}$.
    \item If $\lambda_1(A) \geq \dotsc \geq \lambda_N(A)$, $\lambda_1(B) \geq \dotsc \geq \lambda_N(B)$ denote the eigenvalues of $A$ and $B$ respectively, $\lambda_i(A) \geq \lambda_i(B)$ for all $1 \leq i \leq N$.
\end{enumerate}
\end{prop}
We also use that $\Qff \prec \Kff$, which follows from properties of Schur complements of PSD matrices \cite[Proposition 2.1]{gallier2010schur}.
\begin{proof}[Proof of \cref{lem:titsias-aposteriori}]
This proof follows that of \citet{titsias_variational_2014}. Recall \cref{eqn:full-marginal-likelihood},
\begin{align}
\log p(y) &=-\frac{1}{2}\log\det(\Kff+\noisevariance\bfI) -\frac{1}{2} y\transpose(\Kff+\noisevariance\bfI)^{-1}y-\frac{N}{2}\log 2 \pi \nonumber\\
&\leq-\frac{1}{2}\log\det(\Qff+\noisevariance\bfI) -\frac{1}{2} y\transpose(\Kff+\noisevariance\bfI)^{-1}y-\frac{N}{2}\log 2 \pi \nonumber\\
&\leq -\frac{1}{2}\log\det(\Qff+\noisevariance\bfI) -\frac{1}{2} y\transpose(\Qff+\|\Kff-\Qff\|_{\textup{op}}\bfI+\noisevariance\bfI)^{-1}y-\frac{N}{2}\log 2 \pi \nonumber\\ &= \upperone.
\end{align}
 The first inequality uses $\Kff +\noisevariance \bfI \succ \Qff + +\noisevariance \bfI$, which implies $\log\det(\Kff+\noisevariance\bfI)\geq \log\det(\Qff+\noisevariance\bfI)$. The second inequality uses that $\Kff \prec \Qff+ \|\Kff-\Qff\|_{\textup{op}}\bfI$ and the second part of \cref{prop:spsdmatrices}.

In problems where sparse GP regression is applied, computing the largest eigenvalue of $\Kff-\Qff$ is computationally prohibitive. However, we can use the upper bound $\|\Kff-\Qff\|_{\textup{op}} \leq \Tr(\Kff-\Qff)$, yielding
\[
\upperone \leq -\frac{1}{2}\log\det(\Qff+\noisevariance\bfI) -\frac{1}{2} y\transpose(\Qff+\Tr(\Kff-\Qff)\bfI+\noisevariance\bfI)^{-1}y-\frac{N}{2}\log 2 \pi = \uppertwo.
\]
The bound $\uppertwo$ can be computed in time $\mcO(NM^2)$ with memory $\mcO(NM)$ in much the same way as the ELBO is computed, as it only depends on the low-rank matrix $\Qff$ and the diagonal entries of $\Kff$.
\end{proof}

%%%%%%%%%%%%%%%%%%%%%%%%%%%%SECTION 4 APPENDIX%%%%%%%%%%%%%%%%%%%%%%%%%%%%%%%%%
\section{Proofs for Results Leading to Upper bounds on the KL-divergence}\label{app:proofs-section4}
In this appendix, we restate and provide proofs for the results in \cref{sec:regression-rates}.
\uppergeneraly*
\begin{proof}
 We apply the matrix identity $(A+B)^{-1} = A^{-1} - A^{-1}B(A+B)^{-1}$ to the expression
\begin{equation*}
    \mcU_1 - \ELBO = \frac{t}{2\noisevariance}+\frac{1}{2} y\transpose\left((\Qff+\noisevariance\bfI)^{-1}-(\Qff+\zeta\bfI+\noisevariance\bfI)^{-1}\right)y,
\end{equation*}
with $A = \Qff+\noisevariance\bfI$ and $B= \zeta\bfI$. This gives
\begin{align*}
    \mcU_1 - \ELBO &= \frac{t}{2\noisevariance}+\frac{\zeta}{2} y\transpose\left((\Qff+\noisevariance\bfI)^{-1}(\Qff+(\zeta+\noisevariance)\bfI)^{-1}\right)y \nonumber \\
     &= \frac{t}{2\noisevariance}+\frac{\zeta}{2} y\transpose(\Qff^2+(\zeta+\noisevariance)\Qff + \noisevariance(\zeta+\noisevariance)\bfI))^{-1}y. \nonumber 
\end{align*}
The matrix $\Qff^2+(\zeta+\noisevariance)\Qff$ is SPSD, as it is the product of SPSD matrices that commute. This implies that the eigenvalues of $\Qff^2+(\zeta+\noisevariance)\Qff + \noisevariance(\zeta+\noisevariance)\bfI$ are bounded below by  $\noisevariance(\zeta+\noisevariance)$. As the eigenvalues of the inverse of a SPSD matrix are the inverse of the eigenvalues of the original matrix, the largest eigenvalue of $(\Qff^2+(\zeta+\noisevariance)\Qff + \noisevariance(\zeta+\noisevariance)\bfI))^{-1}$ is bounded above by $(\noisevariance(\zeta+\noisevariance))^{-1}$. Therefore, 
\begin{align}
    \KL{Q}{P} \leq \mcU_1 - \ELBO 
     \leq\frac{t}{2\noisevariance}+\frac{\zeta\|y\|_2^2}{2\noisevariance(\zeta+\noisevariance)}.
\end{align}
This proves the second inequality in \cref{lem:agnostic-upper-bound}. The same argument using $\uppertwo$ in place of $\upperone$ yields,
\begin{equation*}
    \KL{Q}{P} 
     \leq \frac{t}{2\noisevariance}+\frac{t\|y\|_2^2}{2\noisevariance(t+\noisevariance)}. 
\end{equation*}
\end{proof}

\upperavgy*
We have already proven the lower bound in the main body. In order to prove the upper bound in \cref{lem:kl-averaged}, we use a H\"older-type inequality, \citet[Exercise 1.3.26]{tao2011topics}.
\begin{prop}\label{prop:holders-trace}
For any matrix, let $\|A\|_p:= (\sum_i |\sigma_i(A)|^p)^{1/p}$ if $p$ is finite and $\|A\|_{\infty}= \max_i |\sigma_i(A)|$, where $\sigma_i(A)$ are singular values of $A$. Then for $A, B \in \R^{n \times n}$ and any $1 \leq p,q \leq \infty$ such that $\frac{1}{p}+\frac{1}{q}=1$,
\[
\Tr(AB) \leq \|A\|_p \|B\|_q.
\]
\end{prop}
In particular, if $A$ and $B$ are SPSD (so that the singular values agree with the eigenvalues), taking $p=1,$ $q=\infty$,
\[
\Tr(AB\transpose) \leq \Tr(A)\|B\|_{op}.
\]
where $\|B\|_{\mathrm{op}}$ is the largest eigenvalue of $B$. 

\begin{proof}[Proof of \cref{lem:kl-averaged}]
 
For the upper bound, it remains to bound \[\mathrm{kl}(\Kff, \Qff) \coloneqq \KL{\mcN(0, \Kff+\noisevariance\bfI)}{\mcN(0, \Qff+\noisevariance\bfI)}.\]
\begin{align}\label{eqn:kl-avg-upper}
 \mathrm{kl}(\Kff, \Qff) &= \frac{1}{2} \Big(\log \det( \Qff + \noisevariance \bfI)-\log \det( \Kff + \noisevariance \bfI)  - N + \Tr((\Qff+\noisevariance \bfI)^{-1} (\Kff+\noisevariance \bfI))\Big) \nonumber \\
& \leq \frac{1}{2} \Big(- N + \Tr((\Qff+\noisevariance \bfI)^{-1} (\RKff+\noisevariance \bfI)\Big) \nonumber \\
& = \frac{1}{2} \Big(- N + \Tr((\Qff+\noisevariance\bfI)^{-1} ((
\Qff+\noisevariance I)+(\Kff-\Qff))\Big) \nonumber \\
& = \frac{1}{2}\Tr((\Qff+\noisevariance \bfI)^{-1}(\Kff-\Qff)). 
\end{align}
The inequality uses that $\Qff + \noisevariance \bfI \prec \Kff + \noisevariance\bfI$, so $\det(\Qff + \noisevariance\bfI) \leq \det(\Kff + \noisevariance\bfI)$ by \cref{prop:spsdmatrices}.
We can now apply \cref{prop:holders-trace} with $p=1, q=\infty$ to \cref{eqn:kl-avg-upper} giving,
\[
\frac{1}{2}\Tr((\Qff+\noisevariance \bfI)^{-1}(\Kff-\Qff)) \leq \frac{t}{2}\|(\Qff+\noisevariance \bfI)^{-1}\|_{\textup{op}}  \leq \frac{t}{2\noisevariance}.  
\]
Using this bound in \cref{eqn:kl-avg-upper} and combining with \cref{eqn:average-kl-exact} completes the proof of the upper bound.
\end{proof}

%%%%%%%%%%%%%%%%%%%%%%%%%%%%%%%%%SECTION 5 APPENDIX%%%%%%%%%%%%%%%%%%%%%%%%%
\section{Derivations of bounds for specific kernels and covariate distributions}\label{app:examples}
In this appendix, we restate and provide proofs for the results in \cref{sec:examples}.
\subsection{Bounds for Univariate Gaussian distributions and Squared Exponential Kernel}\label{app:sqexp-gaussian}
\seoned*
\begin{proof}[Proof of \cref{cor:gaussian-1D}]
Using \cref{eqn:se-1d-eigvals} and applying the geometric series formula, 
\[
\sum_{m=M+1}^\infty \lambda_m = \sqrt{\frac{2a}{A}}\frac{B^M}{1-B}.
\]
We can use this equation in \cref{thm:upper-bound-fixed-y} (a similar result could be obtained using \cref{thm:upper-bound-average-y}) yielding,
\[
	\Exp{}{\KL{Q}{P}} \leq \left(\sqrt{\frac{2a}{A}}\frac{(M+1)NB^M}{2\noisevariance(1-B)}+\frac{Nv\epsilon}{\noisevariance}\right)\left(1+ \frac{RN}{\noisevariance}\right). 
\]

Choose $\epsilon=\frac{\gamma\noisevariance}{2Nv(1+RN/\noisevariance)}=\Theta(\gamma/N^2)$. By \cref{lem:anari-approx-k-dpp}, an $M$-DPP can be sampled to this level of accuracy using not more than $\mcO(NM(\log \frac{N^2}{\gamma\delta}))$ iterations of MCMC, making the computational cost of selecting inducing inputs $\mcO(NM^3(\log \frac{N^2}{\gamma\delta}))$. We may assume that $M<N$, otherwise by choosing $Z=X$ the KL-divergence is zero and nothing more needs to be shown. Then,
\[
	\Exp{}{\KL{Q}{P}} \leq \sqrt{\frac{2a}{A}}\frac{N^2B^M}{2\noisevariance(1-B)}\left(1+ \frac{RN}{\noisevariance}\right)+\frac{\gamma}{2} 
\]
Take $M = \log_B \sqrt{\frac{A}{2a}}\frac{\gamma\delta \noisevariance(1-B)}{N^2(1+RN/\noisevariance)}=\mcO(\log (N^3/\gamma\delta))$, then
\[
	\Exp{}{\KL{Q}{P}} \leq \gamma.
\]
In the case of ridge leverage score initializations, from \cref{thm:ridge-leverage-agnostic-y} we have with probability $1-5\delta$,
\[
\KL{Q}{P} \leq \sqrt{\frac{2a}{A}}\frac{N^2B^S}{S(1-B)\delta^2\noisevariance}(1+R/\noisevariance)
\]
and $\bfM\leq S \log \frac{S}{\delta}$. Choose $S = \log_{B} \sqrt{\frac{A}{2a}}\frac{\gamma \noisevariance(1-B)\delta^2}{N^2(1+R/\noisevariance)}$. Then on the event where these bounds hold, $\KL{Q}{P} \leq \frac{\gamma}{S}\leq \gamma$ and $\bfM \leq S \log \frac{S}{\delta} = \mcO\left(\log \frac{N^2}{\delta^2\gamma}\log \frac{\log (N^2/\delta^2\gamma)}{\delta}\right)$. If we allow the implicit constant to depend on $\delta$ and $\gamma$ as well this becomes $\mcO(\log N \log \log N)$.
\end{proof}
\subsection{Bounds for Multivariate Gaussian distributions and Squared Exponential Kernel}
\taileigvalsgauss*
\begin{proof}[Proof of \cref{prop:tail-eigenvalues-gaussian}]
The proof of this proposition is nearly identical to an argument in \cite{seeger2008information}. Consider the upper bound, 
\[
\lambda_{M+D-1} \leq \left(\frac{2a}{A}\right)^{\frac{D}{2}}B^{M^{1/D}}. 
\]
Define $\tilde{M}=M-D+1$, then for $M>D-1$,
\begin{align*}
\sum_{m=M+1}^\infty \lambda_m &\leq \left(\frac{2a}{A}\right)^{\frac{D}{2}} \sum_{m=\widetilde{M}+1}^\infty B^{m^{1/D}}\leq \left(\frac{2a}{A}\right)^{\frac{D}{2}}\int_{s=\widetilde{M}}^\infty B^{s^{1/D}} \calcd s\\
& =  \left(\frac{2a}{A}\right)^{\frac{D}{2}}D\alpha^{-D}\int_{t=\alpha\widetilde{M}^{1/D}}^\infty \exp(-t) t^{D-1}\calcd t\\
& =\left(\frac{2a}{A}\right)^{\frac{D}{2}}D\alpha^{-D} \Gamma(D, \alpha(M-D+1)^{1/D}) 
\end{align*}
where in the second to last line we make the substitution $t = \alpha s^{1/D}$ and in the final line we recognized the integral as an incomplete $\Gamma$-function. 

From \citet[8.352]{gradshteyn2014table} for integer $D$ and $r>0$, 
\[
\Gamma(D,r) = (D-1)!e^{-r}\sum_{k=0}^{D-1}\frac{r^k}{k!}. 
\]
For fixed $D$ and $r$ large (which is satisfied by the condition $M\geq \frac{1}{\alpha}D^D+D-1$), we have that the final term in the above sum is the largest, so that 
\[
\Gamma(D,r) \leq D! e^{-r} \frac{r^{D-1}}{(D-1)!} = D e^{-r}y^{D-1}
\]
Using this bound, we arrive at 
\begin{align*}
\sum_{m=M+1}^\infty \lambda_m &\leq \left(\frac{2a}{A}\right)^{\frac{D}{2}}D^2\alpha^{-D} \exp(-\alpha(M-D+1)^{1/D})(\alpha(M-D+1)^{1/D})^{D-1}\\
& \leq \left(\frac{2a}{A}\right)^{\frac{D}{2}}\frac{D^2(M-D+1)}{\alpha}\exp(-\alpha(M-D)^{1/D})= \mcO(M\exp(-\alpha M^{1/D})). 
\end{align*}
\end{proof}
\multivariategaussiansecor*
\begin{proof}
\Cref{cor:gaussian-multi-D} is a consequence of \cref{thm:upper-bound-fixed-y} and \cref{prop:tail-eigenvalues-gaussian}. In the case of the $M$-DPP, we take $\epsilon=\frac{\gamma\noisevariance}{2Nv(1+RN/\noisevariance)}=\Theta(\gamma/N^2)$ as in the proof of \cref{cor:gaussian-1D}. It then remains to choose $M$ so that \[
\left(\frac{N^2}{2\noisevariance(1-B)}\left(1+ \frac{RN}{\noisevariance}\right)\sum_{m=M+1}^\infty\lambda_m\right) \leq \gamma/2.\]
From \cref{prop:tail-eigenvalues-gaussian}, there exists an $M = \mcO((\log \frac{N^3}{\gamma})^D)$ that satisfies this criteria. In the case of ridge leverage scores, it is sufficient to choose $S= \mcO\left(\left(\log \frac{N^2}{\delta \gamma}\right)^D\right)$, which means that with probability at least $1-5\delta$, $\bfM =  \mcO\left(\left(\log \frac{N^2}{\delta \gamma}\right)^D(\log\log \frac{N^2}{\delta \gamma} + \log (1/\delta)) \right)$.
\end{proof}
\subsection{Conditions for Widom's Theorem}
Widom's Theorem \citep{widom_asymptotic_1963}, states that for stationary kernels on compact subsets of Euclidean space, the eigenvalues of the operator $\mcK$ are closely linked to the decay of the spectral density of the kernel function. The theorem applies to any compactly supported covariate distribution with Lebesgue density and stationary kernel with spectral density satisfying the following three conditions:
\begin{enumerate}
    \item For all $i\in \{1,\cdots,D\}$, fixing all $\omega^{(j)}, j \neq i$, there exists an $\omega_0^{(i)}\in \R$ such that $s(\omega)$ is monotonically increasing as a function of $\omega^{(i)}$ for all $\omega^{(i)}<\omega_0^{(i)}$ and is monotonically decreasing as a function of $\omega^{(i)}$ for $\omega^{(i)} \geq \omega_0^{(i)}$.
    \item Let $\{\xi_i\}_{i=1}^\infty, \{\eta_i\}_{i=1}^\infty$, be sequences in $\R^D$ such that $\lim\limits_{i \to \infty} \frac{\|\eta_i-\xi_i\|}{\|\eta_i\|}=0$ and $\lim\limits_{i \to \infty} \|\xi_i\|=\infty,$ then $\lim\limits_{i \to \infty} \frac{|s(\xi_i)|}{|s(\eta_i)|}=1$. 
    \item Let $\{\xi_i\}_{i=1}^\infty, \{\eta_i\}_{i=1}^\infty$, be sequences in $\R^D$ such that $\lim\limits_{i \to \infty} \|\xi_i\|, \|\eta_i\|=\infty$ and $\lim\limits_{i \to \infty}\frac{\|\xi_i\|}{\|\eta_i\|}=0$, then $\lim\limits_{i \to \infty}\frac{|s(\xi_i)|}{|s(\eta_i)|}=0$.
\end{enumerate}

If the kernel and spectral density satisfy these conditions, the number of eigenvalues of $\mcK$ greater than $\epsilon$ is asymptotic (as $\epsilon \to 0$) to the volume of the collection of points in $\R^d \times \R^d$ such that $p(x)s(\omega)> \epsilon$. A precise statement of the result can be found in \citet{widom_asymptotic_1963}, and more discussion of the result is given in \citet{seeger2008information}. Because of the second condition, Widom's theorem cannot be applied to kernels with rapidly decaying spectral densities, such as the SE-kernel (though more stationary kernels are analyzed in \citet{widom1964asymptotic} for uniformly distributed covariates).

%%%%%%%%%%%%%%%%%%%%%%%%%%%SECTION 6 APPENDIX %%%%%%%%%%%%%%%%%%%%%%%%%%%%%%%%%%%%%%%%%%%%%%%%%%%%%%
\section{Lower bounds on the number of features}\label{app:lower-bounds}
In this appendix, we restate and prove the results stated in \cref{sec:lower-bounds}.
\subsection{General Lower Bound on KL-divergence}
\firstlowerbound*
\begin{proof}[Proof of \cref{lem:min-min-lower-bound}]
Define $\ELBO(\datay,Z)$ to be the evidence lower bound assuming $y$ are the observations and inducing points are placed at locations $Z$. Then for any $\datay \in \R^N$ and $Z \in \mcX^M$,
\begin{align}
    \log p(\datay) - \ELBO(\datay,Z) &\geq \log p(0) - \ELBO(0,Z) \nonumber\\ &= \frac{1}{2}\left(\frac{1}{\noisevariance}\Tr(\RKff-\RQff)-\log\frac{\det(\RKff+\noisevariance\bfI)}{\det(\RQff+\noisevariance\bfI)}\right). \label{eqn:lower-bound-tr-logdet}
\end{align}
The inequality uses that the only term in  $\log p(\datay) - \ELBO(\datay,Z)$ that depends on $y$ is the quadratic $\frac{1}{2}\datay\transpose \left((\RQff+\noisevariance\bfI)^{-1} - (\RKff+\noisevariance\bfI)^{-1}\right) \datay\geq 0$ since $(\RQff+\noisevariance\bfI)^{-1}\succ (\RKff+\noisevariance\bfI)^{-1}$.

\newcommand{\rqffeigenvalue}{\bm{\psi}_m}

We can rewrite \cref{eqn:lower-bound-tr-logdet} as a sum over the eigenvalues of $\RKff$ and $\RQff$, which we denote by $\rkffeigenvalue$ and $\rqffeigenvalue$ respectively. Also, since $\RKff \succ \RQff$, $\rkffeigenvalue \geq \rqffeigenvalue$ for all $1\leq m \leq N$ (\cref{prop:spsdmatrices}). This yields,
\begin{align*}
    \log p(0) &- \ELBO(0,Z) = \frac{1}{2}\sum_{m=1}^N \frac{\rkffeigenvalue - \rqffeigenvalue}{\noisevariance} - \log\left(1+ \frac{\rkffeigenvalue-\rqffeigenvalue}{\rqffeigenvalue+\noisevariance}\right) \\
    & \geq \frac{1}{2}\sum_{m=1}^N \frac{\rkffeigenvalue - \rqffeigenvalue}{\noisevariance} - \log \left(1+ \frac{\rkffeigenvalue-\rqffeigenvalue}{\noisevariance}\right) \tag{\textasteriskcentered} \label{eqn:cant-reverse} \\
    &=  \frac{1}{2}\left(\sum_{m=1}^M \frac{\rkffeigenvalue - \rqffeigenvalue}{\noisevariance} - \log \left(1+ \frac{\rkffeigenvalue- \rqffeigenvalue}{\noisevariance}\right) +\sum_{m=M+1}^N \frac{\rkffeigenvalue }{\noisevariance} - \log \left(1+ \frac{\rkffeigenvalue}{\noisevariance}\right)\right).
\end{align*}
In the final line, we use that $\RQff$ is at most rank $M$, so that $\rqffeigenvalue=0$ for all $m>M$. It follows from the inequality $\log (1+a)\leq a$ for $a\geq 0$ that each term in the first sum is non-negative. Hence, 
\begin{align}
    \log p(y) - \ELBO(y,Z) &\geq \frac{1}{2}\sum_{m=M+1}^N \frac{\rkffeigenvalue }{\noisevariance} - \log \left(1+ \frac{\rkffeigenvalue}{\noisevariance}\right). 
\end{align} 
\end{proof}

\subsection{Lower Bound on Eigenvalues of Multivariate Gaussian Inputs and Squared Exponential Kernel}
\lbgauss*
\begin{proof}
Recall from \cref{sec:upper-multivariate-se-gauss} that the eigenvalues of this operator are of the form,
\[
\lambda_r = \left(\frac{2a}{A}\right)^{D/2}B^{s}
\]
where the number of times each eigenvalue is repeated is equal to the number of ways to write $s$ as a sum of $D$ non-negative integers, where the order of the summands matters. This is equal to $\binom{s+D-1}{D-1}$. The number of eigenvalues greater than $\left(2a/A\right)^{D/2}B^{s}$ is therefore,
\[
\sum_{t =1}^s\binom{t+D-1}{D-1} = \binom{s+D}{D}. 
\]
The equality follows from observing that the right hand side is equal to the number of way to write $s$ as a sum of $D+1$ non-negative integers. For each of these representations, the first $D$ integers sum to some $t \leq s$, and once these are fixed there is a unique choice for the final integer. This is equivalent to the left hand side. We therefore conclude $\lambda_{\binom{s+D}{D}}=\left(\frac{2a}{A}\right)^{D/2}B^s$. Define 
\[
\tilde{r} = \min_{s\in \{0\} \cup\N} \left\{\binom{s+D}{D}: \binom{s+D}{D}>r\right\},
\]
and let $\tilde{s}$ denote the corresponding $s$. Then, \[\lambda_r = \lambda_{\tilde{r}}=\left(\frac{2a}{A}\right)^{D/2}B^{\tilde{s}} \text{ \quad and \quad} \binom{\tilde{s}}{D} \leq \binom{\tilde{s}-1+D}{D} \leq r.\]
Using the lower bound, $\left(\frac{\tilde{s}}{D}\right)^D \leq \binom{\tilde{s}}{D}$, we obtain $\tilde{s} \leq Dr^{1/D}$, completing the proof of the lower bound.
\end{proof}
\lbMgauss*
\begin{proof}
By \cref{lem:braun-boundedkernel} and \cref{prop:tail-eigenvalues-gaussian}, for $\delta \in (0,1)$, with probability $1-\delta$,
\begin{equation}
\frac{|\lambda_{m} - \frac{1}{N}\rkffeigenvalue|}{\lambda_m} = \mcO\left(r^2\lambda_m\lambda_{r}^{-1/2}N^{-1/2}\delta^{-1/2}+r\exp(-\alpha r^{1/D}) + \sqrt{\frac{r\exp(-\alpha r^{1/D})}{N\delta}}\right).
\end{equation}
with $\alpha =-\log B$ for any $1\leq r \leq N$. Using \cref{prop:lower-bound-eigenvals-gauss} we have,
\begin{align*}
\frac{|\lambda_{m} - \frac{1}{N}\rkffeigenvalue|}{\lambda_m}\!=\!\mcO\Bigg(r^2N^{-\frac{1}{2}}\exp(\alpha D r^{\frac{1}{D}}/2)\delta^{-\frac{1}{2}}+\frac{r}{\lambda_m}\exp(-\alpha r^{\frac{1}{D}}) +\frac{1}{\lambda_m} \sqrt{\frac{r\exp(-\alpha r^{1/D})}{N\delta}}\Bigg).
\end{align*}
For $\gamma \in (0,1/2)$, choose $r = \lceil (\frac{1}{\alpha D} \log N^{\gamma})^D\rceil$, then noting that for this choice of $r$, the third term in the sum is smaller than the second term,
\begin{equation*}
\frac{|\lambda_{m} - \frac{1}{N}\rkffeigenvalue|}{\lambda_m} = \mcO\left(\delta^{-1/2}\left(r^2N^{(\gamma-1)/2}+\lambda_m^{-1}rN^\frac{-\gamma}{D} \right)\right).
\end{equation*}
Applying \cref{prop:lower-bound-eigenvals-gauss}, with $M+1 = \lfloor (\frac{1}{D}\log_B N^{-\zeta/D})^D\rfloor$ with $\zeta \in (0,\gamma)$. We have
\[
\lambda_{M+1}^{-1}rN^\frac{-\gamma}{D} \leq (M+1) \left(\frac{A}{2a}\right)^{D/2}B^{-D(M+1)^{1/D}}N^{-\gamma/D} \leq (M+1) \left(\frac{A}{2a}\right)^{D/2}N^{(\zeta-\gamma)/D}. 
\]
Thus, for such a choice of $M$, we have with probability at least $1-\delta$, $N\lambda_{M+1} = \rkffeigenvalueM (1+o(1))$. It follows that with probability $1-\delta$, $\KL{Q}{P} \geq N\lambda_{M+1}(1+o(1))$ and $N\lambda_{M+1}=\Omega(N^{1-\zeta/D})$. Choosing $\gamma=1/4$ and $\zeta = \min\{\gamma/2,D\epsilon\}$, completes the proof.
\end{proof}

%%%%%%%%%%%%%%%%%%%SECTION 7 APPENDIX%%%%%%%%%%%%%%%%%%%%%%%%%%%%%%%%%%
\section{Effect of Jitter on Bounds}\label{app:experiments}
In this section, we restate and prove \cref{prop:jitter-bound}. Recall that for $\epsilon>0$, we define $\Qff(\epsilon)\coloneqq \Kuf\transpose(\Kuu + \epsilon \bfI)^{-1}\Kuf$.
\jitter*

\begin{proof}
Let $0 \leq \epsilon < \epsilon'$. For an arbitrary $v \in \R^n$, 
\[ 
v\transpose (\Qff(\epsilon) - \Qff(\epsilon'))v = (\Kuf v)\transpose \left((\Kuu + \epsilon \bfI)^{-1} - (\Kuu + \epsilon' \bfI)^{-1}\right)  (\Kuf v)\geq 0,
\] 
The final inequality follows from $\Kuu + \epsilon \bfI \prec \Kuu + \epsilon' \bfI$ and \cref{prop:spsdmatrices}. Therefore, $\Qff(\epsilon') \prec \Qjitter$. From \cref{prop:spsdmatrices}, we have 
\begin{equation}
    - \frac{1}{2}y\transpose(\Qjitter + \noisevariance \bfI)^{-1}y \geq - \frac{1}{2}y\transpose(\Qff(\epsilon') + \noisevariance \bfI)^{-1}y. \label{eqn:quadformbound}
\end{equation}
Let $A, B$ arbitrary $N\times N$ SPSD matrices with $A\succ B \succ \noisevariance \bfI$. Denote the eigenvalues of $A$ and $B$ respectively as $\lambda_1(A) \geq \dots \lambda_N(A)$ and $\lambda_1(B) \geq \dots \lambda_N(B)$. Then,
\begin{align}
    \log \det A &= \sum_{i=1}^N \log \lambda_i(A) \nonumber \\
    & = \sum_{i=1}^N \log \lambda_i(B) + \sum_{i=1}^n \log \frac{\lambda_i(A)}{\lambda_i(B)} \nonumber \\
    & = \log \det B + \sum_{i=1}^N \log\left(1 + \frac{\lambda_i(A) -\lambda_i(B)}{\lambda_i(B)}\right) \nonumber \\ 
    &\leq \log \det B + \sum_{i=1}^N \frac{\lambda_i(A) -\lambda_i(B)}{\lambda_i(B)} \nonumber \\
    &\leq \log \det B + \frac{1}{\noisevariance}\Tr(A-B) \label{eqn:logdetbound}.
\end{align}
The first inequality follows applying $\log(1+a) \leq a$ to each term in the sum. The second inequality used that $\lambda_i(B) \geq \sigma^2$ since $B \succ \noisevariance \bfI$. Then,
\begin{align}
    -\frac{1}{2}&\log \det(\Qjitter + \noisevariance \bfI ) - \frac{1}{2\noisevariance} \Tr(\Kff-\Qjitter) \nonumber \\
    &= -\frac{1}{2}\log \det(\Qjitter + \noisevariance \bfI ) -\frac{1}{2\noisevariance} \Tr(\Qff(\epsilon')-\Qjitter) - \frac{1}{2\noisevariance} \Tr(\Kff-\Qff(\epsilon')) \nonumber  \\
    & \geq -\frac{1}{2}{\log \det (\Qff(\epsilon') + \noisevariance \bfI})- \frac{1}{2\noisevariance} \Tr(\Kff-\Qff(\epsilon')) \label{eqn:logdetterm}.
\end{align}
where the final inequality follows from \cref{eqn:logdetbound} with $A = \Qff(\epsilon) + \noisevariance \bfI$ and $B = \Qff(\epsilon') + \noisevariance \bfI$. Combining \cref{eqn:quadformbound} with \cref{eqn:logdetterm} proves the monotonicity of the lower bound in $\epsilon$. The upper bound follows from \cref{prop:spsdmatrices} noting that in the quadratic form 
\begin{align*}
\Qff(\epsilon) + (\Tr(\Kff-\Qff(\epsilon)) + \noisevariance) \bfI - \Qff(\epsilon') &+ (\Tr(\Kff-\Qff(\epsilon')) + \noisevariance) \bfI \\
&= \Qff(\epsilon) - \Qff(\epsilon') + \Tr(\Qff(\epsilon) - \Qff(\epsilon'))\bfI\\
&\succ 0.
\end{align*}
\end{proof}

%%%%%%%%%%%%%LEVERAGE SCORE APPENDIX%%%%%%%%%%%%%%%%%%%%%%%%%
\section{An Alternative Ridge Leverage Sampling Initialization}\label{app:adaptive-leverage}
Many implementations of leverage score sampling allow for adaptively selecting the number of inducing points to achieve a desired level of accuracy. We briefly discuss the application of Algorithm 2 in \citet{Musco_ridge_leverage} to the problem of sparse variational inference in Gaussian processes.
\subsection{Effective Dimension}
The number of points sampled by ridge leverage score methods to achieve a desired level of accuracy is closely related to the \emph{effective dimension} of the kernel matrix, which can be thought of as measure of the complexity of the non-parameteric regression model. The effective dimension is defined as the sum of the ridge leverage scores,
\begin{equation}\label{eqn:effective-dimension}
    \deff^\omega \coloneqq \sum_{n=1}^N \ell^\omega(x_n)= \sum_{m=1}^N \frac{\kffeigenvalue}{\kffeigenvalue+\omega},
\end{equation}
and depends on the choice of kernel, the distribution of the covariates and the regularization parameter.

In order to compare such an adaptive method with the bounds discussed in \cref{sec:regression-rates}, we need to consider the typical size of the effective dimension, assuming a fixed kernel and a random set of covariates with identical marginal distributions (or marginal distributions satisfying the conditions in \cref{lem:average-eigenvalues}). 

For any fixed set of covariates, we can split the sum in \cref{eqn:effective-dimension} into two parts, yielding
\begin{equation}\label{eqn:effective-dimension-bound}
    \deff^\omega \leq S + \frac{1}{\omega}\sum_{m=S+1}^N \kffeigenvalue,
\end{equation}
where $S$ is an arbitrary positive integer. Upper bounds on the effective dimension can be obtained by choosing $S$ so that the two terms on the right hand side of \cref{eqn:effective-dimension-bound} are of the same order of magnitude.

\subsection{Adaptively Selecting the Number of Inducing Points with Leverage Scores}

We consider the application of \citet[Algorithm 2]{Musco_ridge_leverage} to the problem of selecting inducing inputs for sparse variational inference in GP models. This algorithm comes with the following bounds on the quality of the resulting Nystr\"om approximation.
\begin{lem}[\cite{Musco_ridge_leverage}, Theorem 7]\label{lem:ridge-leverage-adaptive}
Fix $\delta \in (0,\frac{1}{32})$. There exists an algorithm with run time $\mcO(NM^2)$ and memory complexity $\mcO(NM)$  that with probability $1-3\delta$ returns $M < 384 \deff^\omega\log(\deff^\omega/\delta)$ columns of $\Kff$ such that the resulting Nystr\"om approximation, $\Qff$, satisfies
\[
\|\Kff-\Qff\|_{\textup{op}} \leq \omega
\]
where $\deff^\omega$ denotes the effective dimension of the Gaussian process regressor with $\noisevariance = \omega$.\footnote{Note that $\Kff$ and $\Qff$ are both independent of the noise parameter, so there is no requirement that the `noise parameter' used for initializing inducing points matches the noise parameter used in performing regression.}
\end{lem}
We can now consider the implications of this bound on sparse variational GP regression using \cref{lem:agnostic-upper-bound,lem:kl-averaged}.

\subsection{Ridge Leverage Scores and Sparse Variational Inference}
We begin by considering the resulting error from employing \cref{lem:ridge-leverage-adaptive} in \cref{lem:kl-averaged}. 
Noting that $\Tr(\Kff-\Qff) \leq N\|\Kff-\Qff\|_{\mathrm{op}}$, \cref{lem:kl-averaged} gives us the bound
\begin{equation}\label{eqn:ridge-lev-avg-KL-adap}
\conditionalexp{\KL{Q}{P}}{\bfZ,\bfX} \leq N\frac{\|\RKff-\RQff\|_{op}}{\noisevariance}.
\end{equation}
Similarly, \cref{lem:agnostic-upper-bound} becomes
\[
    \KL{Q}{P} \leq  \frac{\|\RKff-\RQff\|_{\textup{op}}}{2\noisevariance}\left(N+\frac{\|\bfy\|_2^2}{\noisevariance}\right).
\]
Both of these bounds are small if $\|\RKff-\RQff\|_{\textup{op}} \ll 1/N$.

For simplicity, we consider the case when $\bfy$ is assumed to have a conditional distribution that agrees with the GP prior. Fix $\delta \in (0,1/32)$ and $\gamma>0$. Applying Markov's inequality to \cref{eqn:ridge-lev-avg-KL-adap}, with probability at least $1-\delta$,
\begin{equation}\label{eqn:ridge-lev-avg-KL-beta}
\KL{Q}{P} \leq N\frac{\|\RKff-\RQff\|_{\textup{op}}}{\delta\noisevariance}.
\end{equation}
We can apply the algorithm referred to in \Cref{lem:ridge-leverage-adaptive} with $\omega=\noisevariance\delta\gamma/N$, so that with probability at least $1-3\delta$ a set of inducing inputs is chosen such that,
\[
\|\RKff-\RQff\|_{\textup{op}} \leq \delta\noisevariance\gamma/N.
\]
We can then apply a union bound to conclude with probability at least $1-4\delta$,
\begin{equation}\label{eqn:KL-small-event}
   \KL{Q}{P} \leq \gamma. 
\end{equation}
By \cref{cor:avg-eigvals} and Markov's inequality, with probability at least $1-\delta$,
\[
\frac{1}{\omega}\sum_{m=S+1}^N \rkffeigenvalue \leq \frac{N}{\delta\omega}\sum_{m=S+1}^\infty \lambda_m
\]
for any $1\leq S \leq N$. On the event where this holds and recalling we chose the parameter $\omega = \noisevariance\delta\gamma/N$, \cref{eqn:effective-dimension-bound} implies that,
\begin{equation}\label{eqn:effective-dimension-bound-expectation}
    \deff^\omega \leq S + \frac{1}{\delta\omega}\sum_{m=S+1}^N \rkffeigenvalue \leq  S + \frac{N^2\gamma}{\noisevariance\delta^2}\sum_{m=S+1}^\infty \lambda_m.
\end{equation}

We can again apply the union bound to lower bound the probability that both the effective dimension is less that the bound in \cref{eqn:effective-dimension-bound-expectation} and that \cref{eqn:KL-small-event} holds. This yields the following probabilistic bounds on the quality of sparse VI in GP regression with inducing points placed according to approximate ridge leverage scores.
\begin{thm}\label{thm:ridge-leverage-average-y-adaptive}
Fix $\delta \in (0,\frac{1}{32}),$ $\gamma>0$. Under the same assumptions on the covariate distribution and the distribution of $\bfy$ as in \cref{thm:upper-bound-average-y} if inducing points are placed according to \citet[Algorithm 2]{Musco_ridge_leverage} with $\omega=\noisevariance\delta\gamma/N$, then with probability $1-5\delta$, $\bfM < 384 d \log(d/\delta)$ and
\[
\KL{Q}{P} \leq \gamma
\]
where $d= \min\limits_{S \in \N,S \leq N}\left(S + \frac{N^2}{\noisevariance\delta^2\gamma}\sum_{m=S+1}^\infty \lambda_m\right)$.
\end{thm}
A similar argument in the case when we do not assume $\bfy$ is distributed according to the prior model leads to the following result:
\begin{thm}\label{thm:ridge-leverage-agnostic-y-adaptive}
Fix $\delta \in (0,\frac{1}{32}),$ $\gamma>0$. Under the same assumptions on the covariate distribution and the distribution of $\bfy$ as in \cref{thm:upper-bound-fixed-y} if inducing points are placed according to \citet[Algorithm 2]{Musco_ridge_leverage} with $\omega=\frac{2\noisevariance\delta\gamma}{N(1+R/\noisevariance)}$ then with probability $1-5\delta$, $\bfM < 384 d' \log(d'/\delta)$ and
\[
\KL{Q}{P} \leq \gamma
\]
where $d'= \min\limits_{S \in \N,S \leq N}\left(S + \frac{N^2(1+R^2/\noisevariance)}{2\noisevariance\delta\gamma}\sum_{m=S+1}^\infty \lambda_m\right)$.
\end{thm}
Note that while the resulting bounds on $\bfM$ depend on the kernel and covariate distribution, the quality of the resulting approximation in both \cref{thm:ridge-leverage-average-y-adaptive,thm:ridge-leverage-agnostic-y-adaptive} does not.

\begin{table*}[ht]
\begin{center}
\begin{small}
\begin{sc}
\begin{tabular}{lcc}
\toprule
Kernel & Input Distribution & $M$  \\
\midrule
SE-Kernel   & Compact support &  $\mcO((\log N)^D\log\log(N))$\\
SE-Kernel &  Gaussian &  $\mcO((\log N)^D\log\log(N))$\\
Mat\'{e}rn $\nu$  & Compact support &  $\mcO(N^\frac{2D}{2\nu+D}\log N)$\\
\bottomrule
\end{tabular}
\end{sc}
\end{small}
\end{center}
\caption{Bounds on the number of inducing points used in \cref{thm:ridge-leverage-average-y-adaptive,thm:ridge-leverage-agnostic-y-adaptive}.}\label{table:numfeats_ridge}
\end{table*}
The bounds implied by these results for various kernels are given in \cref{table:numfeats_ridge}. Note that the asymptotic rates implied by both \cref{thm:ridge-leverage-average-y-adaptive,thm:ridge-leverage-agnostic-y-adaptive} are the same. This is because, unlike in the case of the $M$-DPP initialization in which the trace is bounded and this is used as an upper bound on the operator norm, the operator norm is bounded directly.

%%%%%%%%%%%%%%%%%%%%%%%%%%%%%%%%%%%%%%%%%%%%%%%%%%%%%%%%%%%%%%%%%%%%%%%%%%%%%%%%%%%%%%%%%%%%%%%%%%%%%%%%%%%%%%%%%%%%%%%%%%%%%%%%%%%%%%%%%%%%%%%%%%%%%%%%%%%%%%%%
\bibliography{bibliography.bib}
%%%%%%%%%%%%%%%%%%%%%%%%%%%%%%%%%%%%%%%%%%%%%%%%%%%%%%%%%%%%%%%%%%%%%%%%%%%%%%%%%%%%%%%%%
\end{document}